\pdfoutput=1
\documentclass[twoside,11pt]{article}

\usepackage[linesnumbered,ruled,vlined]{algorithm2e}
\usepackage{jmlr2e}
\usepackage{xparse}
\usepackage{ifmtarg}
\usepackage{amsmath}
\usepackage{multirow}
\usepackage{mathtools}
\usepackage[capitalize]{cleveref}

\usepackage{makecell}
\usepackage[makeroom]{cancel}
\usepackage{dsfont}

\usepackage{tabu}

\DeclareMathOperator*{\E}{\mathbb{E}}
\DeclareMathOperator*{\R}{\mathbb{R}}
\newcommand{\Simplex}{\Delta^{K-1}}
\DeclareMathOperator*{\argmax}{\arg\max} % YS:why do you use \, ?
\DeclareMathOperator*{\argmin}{\arg\min} % YS:why do you use \, ?
\NewDocumentCommand{\gap}{ m O{} O{} }{\Delta_{#1}}

\NewDocumentCommand \loss { m O{} O{} }{
{\ifx&#2&%
   \ell_{#3}% #1 is empty
\else
   \ell_{#3,#2}% #1 is nonempty
\fi}
}
\NewDocumentCommand \impLoss { m O{} O{} }{
{\ifx&#2&%
   \hat{\ell}_{#3}% #1 is empty
\else
   \hat{\ell}_{#3,#2}% #1 is nonempty
\fi}
}
\newcommand{\ip}[1]{\langle #1 \rangle}
\newcommand{\BestArm}{{i^*_T}}
\newcommand{\BestArmSto}{{i^*}}
\newcommand{\AnyArm}{i}
\NewDocumentCommand{\MyArm}{O{}}{I_{#1}}
\NewDocumentCommand{\MyLoss}{O{}}{\loss{}[\MyArm[#1]][#1]}
\NewDocumentCommand{\OptLoss}{O{}}{\loss{}[\BestArm][#1]}
\NewDocumentCommand{\OptLossSto}{O{}}{\loss{}[\BestArmSto][#1]}
\NewDocumentCommand{\ImpLoss}{O{}}{\impLoss{}[][#1]}
\NewDocumentCommand{\MyImpLoss}{O{}}{\impLoss{}[\MyArm[#1]][#1]}
\NewDocumentCommand{\OptImpLoss}{O{}}{\impLoss{}[\BestArm][#1]}
\NewDocumentCommand{\AnyImpLoss}{O{}}{\impLoss{}[\AnyArm][#1]}
\newcommand{\AnyGap}{\gap{\AnyArm}}
\newcommand{\OptGap}{\Deltamin}
\NewDocumentCommand{\AnyLoss}{O{}}{\loss{}[\AnyArm][#1]}
\NewDocumentCommand{\BestLoss}{O{}}{\loss{}[\BestArm][#1]}
\NewDocumentCommand{\BestLossSto}{O{}}{\loss{}[\BestArmSto][#1]}
\NewDocumentCommand{\OtherLearningRate}{O{}}{\eta_{#1,j}}
\NewDocumentCommand{\OptLearningRate}{O{}}{{\eta_{#1,\BestArm}}}
\NewDocumentCommand{\OptLearningRateSto}{O{}}{{\eta_{#1,\BestArmSto}}}
\NewDocumentCommand{\BestLearningRate}{O{}}{
{\ifx&#1&%
   \eta_{\BestArm}% #1 is empty
\else
   \eta_{#1,\BestArm}% #1 is nonempty
\fi}
}
\NewDocumentCommand{\MyLearningRate}{O{}}{\eta_{#1,\MyArm[#1]}}
\NewDocumentCommand{\LearningRate}{O{}}{{\eta_{#1}}}
\NewDocumentCommand{\CumLoss}{O{}}{\hat{L}_{#1}}
\NewDocumentCommand{\OptCumLossSto}{O{}}{\hat{L}_{#1,\BestArmSto}}
\NewDocumentCommand{\AnyCumLoss}{O{}}{\hat{L}_{#1,i}}
\NewDocumentCommand{\MyCumLoss}{O{}}{\hat{L}_{#1,\MyArm[#1]}}
\NewDocumentCommand{\OtherCumLoss}{O{}}{\hat{L}_{#1,j}}
\NewDocumentCommand{\OptProp}{O{}}{{w_{#1,\BestArm}}}
\NewDocumentCommand{\AnyProp}{O{}}{
{\ifx&#1&%
   w_{\AnyArm}% #1 is empty
\else
   w_{#1,\AnyArm}% #1 is nonempty
\fi}
}
\NewDocumentCommand{\OtherProp}{O{}}{{w_{#1,j}}}
\NewDocumentCommand{\MyProp}{O{}}{{w_{#1,I_t}}}
\NewDocumentCommand{\BestPropSto}{O{}}{{w_{#1,\BestArmSto}}}
\NewDocumentCommand{\Prop}{O{}}{{w_{#1}}}
\NewDocumentCommand{\BestCumLoss}{O{T}}{L_{\BestArm}}
\NewDocumentCommand{\AnyCumRealLoss}{O{T}}{L_{\AnyArm}}
\NewDocumentCommand{\BestCumLossSto}{O{T}}{L_{\BestArm}}
\NewDocumentCommand{\BestCumImpLoss}{O{T}}{\hat{L}_{#1,\BestArm}}
\NewDocumentCommand{\BestCumImpLossSto}{O{T}}{\hat{L}_{#1,\BestArm}}
\NewDocumentCommand{\ind}{O{T} O{}}{\mathds 1_{#1}(#2)} 

\newcommand{\Alg}[1]{\textsc{#1}}

\NewDocumentCommand{\OtherArm}{O{}}{J_{#1}}
\NewDocumentCommand{\OtherLoss}{O{}}{\loss{}[\OtherArm[#1]][#1]}

\newcommand{\TimeIndependentLearningRate}{
\AnyGap^{1-2\alpha}
}
\newcommand{\TimeDependentLearningRate}{
 \frac{16^\alpha}{4}\frac{1-\overline{t}^{-1+\alpha}}{(1-\alpha)t^\alpha}
}
\newcommand{\AnyExplicitLearningRate}{
\TimeIndependentLearningRate\TimeDependentLearningRate
}

\newcommand{\A}{\mathcal{A}}

\newcommand{\lr}[1]{\left (#1 \right)}
\newcommand{\lrc}[1]{\left \{ #1 \right \}}
\newcommand{\lrs}[1]{\left [ #1 \right ]}
\newcommand{\dprod}[1]{\left \langle#1 \right \rangle}

\newcommand{\Deltamin}{\Delta_{\min}}

\newcommand{\StoRegretBaseWithoutConst}{\lr{\sum_{i\neq i^*} \frac{\log(T) + 3}{\AnyGap}} + \frac{2}{\Deltamin}}
\newcommand{\StoRegretConst}{28K\log(T)+ \frac{3}{2}\sqrt{K} + 32}
\newcommand{\StoRegretBase}{\lr{\sum_{i\neq i^*} \frac{\log(T) + 3}{\AnyGap}} + 28K\log(T) + \frac{2}{\Deltamin} + \frac{3}{2}\sqrt{K} + 32}

\newcommand{\StoRegretBaseWithoutConstIW}{\lr{\sum_{i\neq i^*} \frac{4\log(T) + 12}{\AnyGap}} + \frac{2}{\Deltamin}}
\newcommand{\StoRegretConstIW}{4\log(T)+ \frac{3}{2}\sqrt{K} + 8}
\newcommand{\StoRegretBaseIW}{\lr{\sum_{i\neq i^*} \frac{4\log(T) + 12}{\AnyGap}} + 4\log(T) + \frac{2}{\Deltamin} + \frac{3}{2}\sqrt{K} + 8}

\newcommand{\Reg}{\overline{Reg}_T}

\usepackage{color}
\usepackage[normalem]{ulem}

\newcommand{\commentout}[1]{}

\newcommand{\rev}[1]{#1}

\ShortHeadings{An Optimal Algorithm for Stochastic and Adversarial Bandits}{Zimmert and Seldin}
\firstpageno{1}
\usepackage{lastpage}
\jmlrheading{22}{2021}{1-\pageref{LastPage}}{9/19; Revised
3/20}{2/21}{19-753}{Julian Zimmert and Yevgeny Seldin}
\ShortHeadings{An Optimal Algorithm for Stochastic and Adversarial Bandits}{Zimmert and Seldin}
\begin{document}

\title{Tsallis-INF: An Optimal Algorithm for Stochastic and Adversarial Bandits}
%\title{An Optimal Algorithm for Stochastic and Adversarial Bandits}
%Alternative: ``One Algorithm for Stochastic and Adversarial Bandits: Closing the Log Gap''

\author{\name Julian Zimmert \email zimmert@di.ku.dk \\
%       \addr University of Copenhagen\\
%       Copenhagen, Denmark
%       \AND
       \name Yevgeny Seldin \email seldin@di.ku.dk \\
       \addr University of Copenhagen, %\\
       Copenhagen, Denmark}

\editor{Peter Auer}

\maketitle

\begin{abstract}%   <- trailing '%' for backward compatibility of .sty file
We derive an algorithm that achieves the optimal (within constants) pseudo-regret in both adversarial and stochastic multi-armed bandits without prior knowledge of the regime and time horizon.\footnote{The paper expands and improves our earlier work \citep{ZS19}.} 
The algorithm is based on online mirror descent \rev{(OMD)} with Tsallis entropy regularization with power $\alpha=1/2$ and reduced-variance loss estimators. More generally, we define an adversarial regime with a self-bounding constraint, which includes stochastic regime, stochastically constrained adversarial regime \citep{wei2018more}, and stochastic regime with adversarial corruptions \citep{LML18} as special cases, and show that the algorithm achieves logarithmic regret guarantee in this regime and all of its special cases simultaneously with the optimal regret guarantee in the adversarial regime.  
The algorithm also achieves adversarial and stochastic optimality in the utility-based dueling bandit setting.
We provide empirical evaluation of the algorithm demonstrating that it significantly outperforms \Alg{Ucb1} and \Alg{Exp3} in stochastic environments. We also provide examples of adversarial environments, where \Alg{Ucb1} and \Alg{Thompson Sampling} exhibit almost linear regret, whereas our algorithm suffers only logarithmic regret. To the best of our knowledge, this is the first example demonstrating vulnerability of \Alg{Thompson Sampling} in adversarial environments. Last but not least, we present a general stochastic analysis and a general adversarial analysis of OMD algorithms with Tsallis entropy regularization for $\alpha\in[0,1]$ and explain the reason why $\alpha=1/2$ works best. 
\end{abstract}

\begin{keywords}
  Bandits, Online Learning, Best of Both Worlds, Online Mirror Descent, Tsallis Entropy, Multi-armed Bandits, Stochastic, Adversarial, I.I.D.
\end{keywords}

\section{Introduction}
Stochastic (i.i.d.) and adversarial multi-armed bandits are two fundamental sequential decision making problems in online learning \citep{thompson1933likelihood,robbins1952some,lai1985asymptotically,auer2002finite,auer2002nonstochastic}. 
When prior information about the nature of environment is available, it is possible to achieve $\mathcal{O}\lr{\sum_{\AnyArm:\AnyGap>0}\frac{\log(T)}{\AnyGap}}$ pseudo-regret in the stochastic case \citep{lai1985asymptotically,auer2002finite} and $\mathcal{O}(\sqrt{KT})$ pseudo-regret in the adversarial case \citep{audibert2009minimax,audibert2010regret}, where $T$ is the time horizon, $K$ is the number of actions (a.k.a.\ arms), and $\AnyGap$ are suboptimality gaps. Both results match the lower bounds within constants, see \cite{bubeck2012regret} for a survey.\footnote{To be precise, the $\mathcal{O}(\sum_{\AnyArm:\AnyGap>0}\frac{\log(T)}{\AnyGap})$ stochastic regret rate is optimal when the means of the rewards are close to $\frac{1}{2}$, see \citet{lai1985asymptotically}, \citet{CGM+13}, and \citet{KKM12} for refined lower and upper bounds otherwise. However, the refined analysis \rev{applies to stochastic bandits, whereas we consider a more general setting}, see Section~\ref{sec:problem} for details.% Therefore, in our case the rate $\mathcal{O}(\sum_{\AnyArm:\AnyGap>0}\frac{\log(T)}{\AnyGap})$ is optimal.
}
%\rev{In the purely stochastic setting, \citet{L19} presented an algorithm that simultaneously achieves the asymptotically optimal logarithmic bound and the worst-case finite horizon bound, which is also of order $\mathcal{O}(\sqrt{KT})$ in the stochastic regime.}
The challenge in recent years has been to achieve the optimal regret rates without prior knowledge about the nature of the problem.

%A growing number of papers in recent years proposed algorithms that achieve close to optimal regret guarantees in the stochastic setting with no or little compromise in the adversarial guarantees \citep{bubeck2012best, seldin2014one, auer2016algorithm, seldin2017improved, wei2018more}. However, neither managed to achieve the optimal rates for both simultaneously. 
%It has been of increasing interest whether a single algorithm, oblivious to the environment, can achieve both.
%Such an algorithm would allow to exploit ``nice'' (i.i.d.) environments without compromising on worst case guarantees.
One approach pursued by \citet{bubeck2012best} and later refined by \citet{auer2016algorithm} is to start playing under the assumption that the environment is i.i.d.\ and constantly monitor whether the assumption is satisfied. If a deviation from the i.i.d.\ assumption is detected, the algorithm performs an irreversible switch into an adversarial operation mode.
This approach recovers the optimal bound in the stochastic case, but suffers from a multiplicative logarithmic factor in the regret in the adversarial case.
Furthermore, the time horizon needs to be known in advance. 
The best known doubling schemes lead to extra multiplicative logarithmic factors in either the stochastic or the adversarial regime \citep{besson2018doubling}.

Another approach pioneered by \citet{seldin2014one} alters algorithms designed for adversarial bandits to achieve improved regret in the stochastic setting without losing the adversarial guarantees. \rev{They have introduced \Alg{EXP3++}, a modification of the \Alg{EXP3} algorithm for adversarial bandits, which was later improved by \citet{seldin2017improved} to achieve} an anytime regret of $\mathcal{O}\lr{\sum_{\AnyArm:\AnyGap>0}\frac{\log(T)^2}{\AnyGap}}$ in the stochastic case while preserving optimality in the adversarial case. 
%In this line of work, through modification of the \Alg{EXP3} algorithm \citet{seldin2017improved} have achieved an anytime regret of $\mathcal{O}\lr{\sum_{\AnyArm:\AnyGap>0}\frac{\log(T)^2}{\AnyGap}}$ in the stochastic case while preserving optimality in the adversarial case.
A related approach by \citet{wei2018more} uses log-barrier regularization instead of entropic regularization behind the \Alg{EXP3}. Their 
%analysis holds in a more general stochastically constrained adversarial setting and the 
stochastic regret bound scales with $\log(T)$, although 
\rev{the constants are not spelled out explicitly and by empirical evaluation seem to be very large.} 
%experimentally the algorithm lags behind the modification of \Alg{EXP3} due to suboptimal dependence on other parameters. 
%(The work of \citeauthor{wei2018more} assumes a fixed time horizon and, therefore, extra logarithmic factors appear in the anytime setting.)
Their adversarial regret guarantee scales with a square root of the cumulative loss of the best action in hindsight rather than a square root of the time horizon, but has an extra $\log T$ factor.

\citet{seldin2014one}, \citet{LML18}, and \citet{wei2018more} also define a number of intermediate regimes between stochastic and adversarial bandits and provide improved regret guarantees for them.

\rev{The question of whether it is at all possible to achieve simultaneous optimality in both worlds with no prior knowledge about the regime has remained open since the work of \citet{bubeck2012best}.} 
\citet{auer2016algorithm} have shown that \rev{no} algorithm obtaining the optimal stochastic pseudo-regret bound \rev{can} simultaneously achieve the optimal high-probability adversarial regret bound. Neither can an algorithm obtain the optimal stochastic pseudo-regret guarantee simultaneously with the optimal expected regret guarantee for adaptive adversaries.\footnote{This does not contradict our result because we bound the pseudo-regret, which is weaker than the expected regret.} 
In addition, 
%It might be impossible to avoid an additional logarithmic factor on one of the sides.
\citet{abbasi2018best} have shown that in the pure exploration setting %, i.e. in which the agent is given a fixed time budget $T$ and needs to identify the optimal arm with the highest possible probability, 
it is also impossible to obtain the optimal rates in both stochastic and adversarial regimes.

We show that for pseudo-regret it is possible to achieve optimality in both regimes with a surprisingly simple algorithm.
\rev{Moreover, we define a more general \emph{adversarial regime with a self-bounding constraint}, which includes the stochastic, stochastically constrained adversarial \citep{wei2018more}, and adversarially corrupted stochastic \citep{LML18} regimes as special cases.} We propose an algorithm that achieves logarithmic pseudo-regret guarantee in the adversarial regime with a self-bounding constraint simultaneously with the adversarial regret guarantee.
 The algorithm is based on online mirror descent with regularization by Tsallis entropy with power $\alpha$. We name it \Alg{$\alpha$-Tsallis-Inf}, or simply \Alg{Tsallis-Inf} \rev{for $\alpha=\frac{1}{2}$}, where $\Alg{INF}$ stands for Implicitly Normalized Forecaster \citep{audibert2009minimax}. 
The proposed algorithm is anytime: it requires neither the knowledge of the time horizon nor doubling schemes.

\rev{The main contributions of the paper are summarized in the following bullet points:
\begin{enumerate}
    \item We propose the \Alg{Tsallis-INF} algorithm, which is based on online mirror descent with regularization by Tsallis entropy with power $\alpha=\frac{1}{2}$. The algorithm achieves the optimal logarithmic pseudo-regret rate in the stochastic regime simultaneously with the optimal square-root adversarial regret guarantee with no prior knowledge of the regime. This resolves an open question of \citet{bubeck2012best}.
    \item When combined with reduced-variance loss estimators proposed by \citet{ZL19}, the leading constant of the stochastic regret bound for the \Alg{Tsallis-INF} algorithm matches the asymptotic lower bound of \citet{lai1985asymptotically} within a multiplicative factor of 2.
    \item The leading constant of the adversarial regret bound for the same combination matches the minimax lower bound of \citet[Theorem 6.1]{CBL06} within a multiplicative factor of less than 15. To the best of our knowledge, this is the best leading constant in an adversarial regret bound known today, matching the result of \citet{ZL19}.
    \item We introduce an adversarial regime with a self-bounding constraint, which includes stochastic, stochastically constrained adversarial, and adversarially corrupted stochastic regimes as special cases. We show that \Alg{Tsallis-INF} achieves logarithmic regret in the new regime.
    \item We improve the regret bound for adversarially corrupted stochastic regimes.
    \item We use \Alg{Tsallis-INF} in a \Alg{Sparring} framework \citep{ailon2014reducing} to obtain an algorithm that achieves stochastic and adversarial optimality in utility-based dueling bandits.
    \item We provide a general analysis of OMD with Tsallis-Entropy regularization with power $\alpha\in[0,1]$ and provide an intuition on why $\alpha=\frac{1}{2}$ works best.
    \item We provide an empirical comparison of \Alg{Tsallis-INF} with standard algorithms from the literature, \Alg{UCB1}, \Alg{Thompson Sampling}, \Alg{EXP3}, \Alg{EXP3++}, \Alg{Broad}. We show that in stochastic environments with expected losses close to $0.5$, \Alg{Tsallis-INF} is only slightly worse than \Alg{Thompson Sampling} and significantly outperforms all other competitors, whereas in stochastically constrained adversarial environments \Alg{Tsallis-INF} significantly outperforms all the competitors. 
    \item In one of the empirical comparisons, we design a stochastically constrained adversarial environment, where \Alg{Thompson Sampling} suffers almost linear regret. To the best of our knowledge, this is the first evidence that \Alg{Thompson Sampling} is not suitable for adversarial environments.
\end{enumerate}} 
\begin{table}[t]
\begin{centering}
\setlength\tabcolsep{2.5pt}
{\tabulinesep=.8mm
\begin{tabu}{c|c|c}
  &Regime& $\frac{Upper\, Bound}{Lower\, Bound}$\\%[2pt]
\hline
% \multirow{2}{*}{\makecell{\Alg{Broad} \citep{wei2018more} \\ \textit{Corresponds to Tsallis entropy with $\alpha= 0$.} \\ \textit{Doubling is used for gap estimation.}}}
\begin{tabular}{c}
\Alg{Broad} \citep{wei2018more} \\ \textit{Corresponds to Tsallis entropy regularization with $\alpha= 0$.} \\ \textit{Doubling is used for tuning the learning rate.}
\end{tabular}
&
\begin{tabular}{c}
Sto.\\[4pt]Adv.
\end{tabular}
&
\begin{tabular}{c}
$\mathcal{O}(K)$ \\
$\mathcal{O}\left(\sqrt{\log T}\right)$
\end{tabular}
\\%[4pt]
% &Sto. & $\mathbf{\mathcal{O}(K)}$ \\
% & & \\
% &Adv. & $\mathcal{O}\left(\sqrt{\log(T)}\right)$  \\[4pt]
\hline
%\multirow{2}{*}{\makecell{$\mathbf{\alpha=\frac{1}{2}}$ (This paper) \\ \textit{No need for doubling or mixing.}}} &Sto. \& Adv.& $\mathbf{\mathcal{O}(1)} $\\
%& & \\[4pt]
\begin{tabular}{c}
$\mathbf{\alpha=\frac{1}{2}}$ (This paper) \\ \textit{Anytime. No need for gap estimation, doubling, or mixing.}
\end{tabular}
&
~Sto. \& Adv.~
&
$\mathbf{\mathcal{O}(1)}$
\\%[4pt]
\hline
%\multirow{3}{*}{\makecell{\Alg{Exp3++}\citep{seldin2017improved} \\
%\textit{Corresponds to Tsallis entropy with $\alpha= 1$.} \\
%\textit{Mixed-in exploration is used for gap estimation.}}} 
% &Sto.& $\mathcal{O}(\log(T))$\\
%&Adv. &$\mathcal{O}\left(\sqrt{\log(K)}\right)$ 
\begin{tabular}{c}
\Alg{Exp3++} \citep{seldin2017improved} \\
\textit{Corresponds to Tsallis entropy regularization with $\alpha= 1$.} \\
\textit{Anytime. Mixed-in exploration is used for gap estimation.}
\end{tabular}
&
\begin{tabular}{c}
Sto.\\[4pt]Adv.
\end{tabular}
&
\begin{tabular}{c}
$\mathcal{O}(\log T)$\\
$\mathcal{O}\left(\sqrt{\log K}\right)$
\end{tabular}
\end{tabu}
}
\caption{\rev{Ratio of regret upper to lower bound for \Alg{Tsallis-INF} and the closest prior work, \Alg{BROAD} and \Alg{EXP3++}.}} 
\end{centering}
\label{tab:summary}
\end{table}

The paper is structured in the following way:
In Section~\ref{sec:problem}, we provide a formal definition of the problem setting, including the adversarial \rev{environment and the adversarial environment with a self-bounding constraint.} Stochastic environments are a special case of the latter.
In Section~\ref{sec:OMD}, we briefly review the framework of online mirror descent.
We follow the techniques of \citet{bubeck2010bandits} to \rev{derive an anytime version of} the family of algorithms based on regularization by $\alpha$-Tsallis Entropy \citep{tsallis1988possible,abernethy2015fighting}.
Section~\ref{sec:theorems} contains the main theorems.
We show that $\alpha=\frac{1}{2}$ provides an algorithm that is optimal in both adversarial \rev{regime and adversarial regime with a self-bounding constraint}. The latter implies optimality in the stochastic regime. 
Interestingly, it is the same regularization power $\alpha = \frac{1}{2}$ that has been used by \citet{audibert2009minimax,audibert2010regret} in \Alg{Poly-INF} algorithm to achieve the optimal regret rate in the adversarial regime. \rev{We analyse the algorithm with standard importance-weighted loss estimators and with reduced-variance loss estimators proposed by \cite{ZL19}. The latter further reduces the constants and gets within a multiplicative factor of less than 15 from the minimax lower bound in the adversarial case and a multiplicative factor of 2 from the asymptotic lower bound in the stochastic case.} 
Table~\ref{tab:summary} relates our results to the closest prior work \rev{on best-of-both-worlds algorithms}. 
\citet{wei2018more} use logarithmic regularization, which corresponds to Tsallis entropy with \rev{power} $\alpha = 0$, and apply doubling for tuning the learning rate. 
\citet{seldin2017improved} use entropic regularization, which corresponds to Tsallis entropy with \rev{power} $\alpha = 1$, and mix in additional exploration for estimation of the gaps. \rev{\Alg{Tsallis-INF} with $\alpha=\frac{1}{2}$ requires neither doubling nor mixing nor estimation of the gaps.}
At the end of Section~\ref{sec:theorems}, we also provide a general analysis of the regret of \Alg{$\alpha$-Tsallis-Inf} \rev{with $\alpha\in[0,1]$ in adversarial environments and a general analysis of the regret of \Alg{$\alpha$-Tsallis-Inf} with $\alpha\in[0,1]$ in stochastic environments. We show that for $\alpha\neq \frac{1}{2}$ the optimal form of regularization and learning rate for the adversarial regime and for the stochastic regime differ. Thus, for $\alpha\neq\frac{1}{2}$ the algorithm does not achieve simultaneous optimality in both. Furthermore, for $\alpha\neq\frac{1}{2}$ the optimal regularizer for the stochastic regime requires oracle access to the unknown gaps. Prior work \citep{seldin2014one,seldin2017improved,wei2018more} used additional techniques, such as mixed-in exploration or doubling, to control the regret, but as we show in Table~\ref{tab:summary} the results were suboptimal.}
%A summary of the results is provided in Table~\ref{tab:summary}.
%This is noteworthy because the popular \Alg{Exp3} algorithm is a member of this family.
In Section~\ref{sec:int} %, we prove the applicability of the algorithm for further problems.
%For two intermediate settings from the literature, we recover 
we show that \rev{the stochastic regime with adversarial corruptions \citep{LML18} is a special case of the adversarial regime with a self-bounding constraint and that} \Alg{Tsallis-Inf} achieves the optimal regret rate \rev{there as well}. 
In Section~\ref{sec:duel}, we apply \Alg{Tsallis-Inf} to dueling bandits.
Section~\ref{sec:proofs} contains proofs of our main theorems.
In Section~\ref{sec:experiments}, we provide an empirical comparison of \Alg{Tsallis-Inf} with baseline stochastic and adversarial bandit algorithms from the literature. We show that in stochastic environments \rev{with loss means close to $0.5$,} \Alg{Tsallis-Inf} \rev{with reduced-variance loss estimators significantly} outperforms \Alg{UCB1}, \Alg{EXP3}, \rev{\Alg{EXP3++}, and \Alg{Broad}, and follows closely} behind \Alg{Thompson Sampling}, whereas in certain adversarial environments it significantly outperforms \Alg{UCB1} and \Alg{Thompson Sampling}, which suffer almost linear regret, and also \rev{significantly} outperforms \Alg{EXP3}, \rev{\Alg{EXP3++}, and \Alg{Broad}}.
%In particular, we provide empirical evidence that Thompson Sampling \cite{Thompson + Kaufmann} can suffer almost linear regret in certain adversarial regimes. 
To the best of our knowledge, this is also the first evidence that \Alg{Thompson Sampling} is vulnerable \rev{in adversarial} environments.
We conclude with a summary in Section~\ref{sec:dis}.

\section{Problem Setting}
\label{sec:problem}
At time $t=1,2,\dots$, the agent chooses an arm $\MyArm[t]\in\{1,\ldots,K\}$ out of a set of $K$ arms.
The environment picks a loss vector $\ell_t \in [0,1]^K$ and the agent observes and suffers \emph{only} the loss of the arm played, $\MyLoss[t]$. The performance of an algorithm is measured in terms of pseudo-regret:
\begin{align*}
\overline{Reg}_T = \mathbb{E}\left[\sum_{t=1}^T\MyLoss[t]\right]-\min_{\AnyArm}\E\left[\sum_{t=1}^T\AnyLoss[t]\right]=\mathbb{E}\left[\sum_{t=1}^T\left(\MyLoss[t]-\BestLoss[t]\right) \right],
\end{align*}
where $\BestArm \in \arg\min_i \mathbb{E}\left[\sum_{t=1}^T\AnyLoss[t]\right]$ is defined as a best arm in expectation in hindsight and the expectation is taken over internal randomization of the algorithm and the environment.

In the (\emph{adaptive}) \emph{adversarial setting}, the adversary selects the losses arbitrarily, potentially based on the history of the agent's actions $(\MyArm[1],\dots,\MyArm[t-1])$ and the adversary's own internal randomization. For deterministic oblivious adversaries, the definition of pseudo-regret coincides with the expected regret defined as $\mathbb{E}[Reg_T] = \mathbb{E}\left[\min_{\AnyArm}\sum_{t=1}^T\left(\MyLoss[t]-\AnyLoss[t]\right)\right]$. 
%\commentout{\ys{I am not sure if we need this sentence.}\jz{I leave it for the first submission.} }

\rev{We further define an \emph{adversarial regime with a $(\Delta, C, T)$ self-bounding constraint}, where $\Delta\in[0,1]^K$  and $C\geq 0$. In this regime, the adversary selects losses such that at time $T$ the regret of any algorithm satisfies
\begin{equation}
\label{eq:self-bound}
\Reg \geq \sum_{t=1}^T \sum_i \Delta_i \mathbb{P}(I_t = i) - C.
\end{equation}
The above condition should be satisfied at time $T$, but there is no requirement that it should be satisfied for all $t < T$.

A simple instance of an adversarial regime with a self-bounding constraint is the \emph{stochastic} regime. In the stochastic regime, the losses $\ell_{t,i}$ are drawn from distributions with fixed means, $\E[\ell_{t,i}] = \mu_i$ independently of $t$, and the pseudo-regret can be written as
\begin{equation}
    \Reg = \sum_{t=1}^T\sum_i \Delta_i \mathbb{P}(I_t=i),
\label{eq:stochastic-regret}
\end{equation}
where $\Delta_i = \E[\ell_{t,i}] - \min_i \E[\ell_{t,i}]$ is the suboptimality gap of action $i$. Thus, \eqref{eq:self-bound} is satisfied with $\Delta$ being the vector of suboptimality gaps and $C=0$. In the stochastic regime, the best arm $i^* = \argmin_i \mu_i$ is the same for all the rounds, $\BestArm = i^*$ for all $T$  (if there is more than one best arm we can pick one arbitrarily).

Another instance of an adversarial regime with a self-bounding constraint is the 
\emph{stochastically constrained adversarial setting} \citep{wei2018more}. In this setting, the losses $\ell_{t,i}$ are drawn from distributions with fixed gaps, $\E[\ell_{t,i} - \ell_{t,j}] = \tilde \Delta_{i,j}$ independently of $t$, but the means, as well as other parameters of the distributions of all arms, are allowed to change with time and may depend on the agent's past actions $\MyArm[1],\dots,\MyArm[t-1]$. Obviously, the stochastic regime is a special case of a stochastically constrained adversary. By using $i^* = \argmin_i \tilde \Delta_{i,1}$ to denote an optimal arm (if there is more than one, we can pick one arbitrarily) we define a vector of suboptimality gaps $\Delta$ by taking $\Delta_i = \tilde \Delta_{i,i^*}$, and then the pseudo-regret satisfies the identity in \eqref{eq:stochastic-regret} and the condition in equation \eqref{eq:self-bound} is satisfied with the vector $\Delta$ and $C=0$. In the stochastically constrained adversarial setting, the best arm is also the same for all rounds, $\BestArm = i^*$ for all $T$.} 

\rev{In Section~\ref{sec:int}, we show that \emph{stochastic bandits with adversarial corruptions} \citep{LML18} are also a special case of an adversarial regime with a self-bounding constraint.}

\rev{The motivation behind the definition of the adversarial regime with a self-bounding constraint will become clear when we explain the analysis. For simplified intuition, the reader can think about its special case, the stochastic regime, where the constraint \eqref{eq:self-bound} is satisfied by the identity in \eqref{eq:stochastic-regret}.}

\section{Online Mirror Descent}
\label{sec:OMD}
We recall a number of basic definitions and facts from convex analysis.
The convex conjugate (a.k.a.\ Fenchel conjugate) of a function $f:\mathbb{R}^K\rightarrow\mathbb{R}$ is defined by
$$f^*(y) = \rev{\sup}_{x\in\mathbb{R}^K}\left\{\left\langle x, y\right\rangle - f(x)\right\}.$$
We use
\[
\mathcal{I}_{\A}(x) := \begin{cases} 0, &\mbox{if }x\in \A\\\infty, &\mbox{otherwise}\end{cases}\]
to denote the characteristic function of a closed and convex set $\A\subset\mathbb{R}^K$. Hence, $(f+\mathcal{I}_\A)^*(y) = \max_{x\in \A}\left\{\left\langle x, y\right\rangle - f(x)\right\}$.
By standard results from convex analysis \citep{rockafellar2015convex}, for differentiable and convex $f$ with invertible gradient $(\nabla f)^{-1}$, it holds that
$$\nabla (f+\mathcal{I}_\A)^*(y) = \argmax_{x\in \A}\left\{\left\langle x, y\right\rangle - f(x)\right\}\in \A.$$

\subsection{General Framework}
The traditional online mirror descent (OMD) framework uses a fixed regularizer $\Psi$ with certain regularity constraints \citep{shalev2012online}.
The update rule is \rev{
\begin{align*}
w_1 = \min_{w\in \mathcal{A}}\Psi(w)\,,\qquad w_{t+1} = \min_{w\in \mathcal{A}}a_t\langle w,\ell_t\rangle+D_\Psi(w,w_t)\,,
\end{align*}
where $\ell_t$ is the observed loss at time $t$, $\mathcal{A}$ is the convex body of the action set, $a_t$ is a weight parameter, and $D_\Psi$ is the Bregman divergence $D_\Psi(x,y) = \Psi(x)-\Psi(y)-\langle x-y,\nabla\Psi(y)\rangle$.
If the norm of the gradient of the regularizer $||\nabla\Psi(x)||$ is unbounded at the boundary of $\mathcal{A}$, then the update rule is equivalent to} $w_{t+1}=\nabla (\Psi+\rev{\mathcal{I}_{\mathcal{A}}})^*(-\sum_{s=1}^t a_s\ell_s)$, where $\sum_{s=1}^t a_s\ell_s$ is a weighted sum of past losses.
This setting has been generalized to time-varying regularizers $\Psi_t$ \citep{orabona2015generalized}, where the updates are given by $w_{t+1} = \nabla (\Psi_t+\rev{\mathcal{I}_{\mathcal{A}}})^*(-\sum_{s=1}^t \ell_s)$. Note that this formulation uses no weighting $a_s$ of the losses.

In the bandit setting, we do not observe the complete loss vector $\loss{}[][t]$. 
Instead, an unbiased estimator $\impLoss{}[][t]$ satisfying $ \mathbb{E}_{\MyArm[t]\sim w_t}\left[\impLoss{}[][t]\right] = \loss{}[][t]$ is used for updating the cumulative losses.
The \rev{common way of constructing unbiased loss estimators is through importance-weighted sampling: }
\begin{equation}
\label{eq:IW}
\tag{IW}
\hat\ell_{t,i} = \frac{\ind[t][i]\ell_{t,i}}{w_{t,i}}\,,\mbox{ where }\ind[t][i] = \ind[][I_t=i] \mbox{ is the indicator function.}
\end{equation}
\rev{We use \eqref{eq:IW} to denote these estimators. \citet{ZL19} proposed reduced-variance importance-weighted loss estimators, which we call for brevity reduced-variance estimators or \eqref{eq:IIW}-estimators, and they are defined by
\begin{equation}
\label{eq:IIW}
\tag{RV}
\hat\ell_{t,i} = \frac{\ind[t][i](\ell_{t,i}-\mathbb{B}_t(i))}{w_{t,i}}+\mathbb{B}_t(i)\,,\mbox{ where }\mathbb{B}_t(i) := \frac{1}{2}\ind[][w_{t,i} \geq \eta_t^2] \,.
\end{equation}
For any $\smash{\mathbb{B}_t(i)\in[0,1]}$ the loss estimators remain unbiased, but their second moment $\E[\hat \ell_{t,i}^2]$ and variance are reduced.
The value $\smash{\mathbb{B}_t(i)=\frac{1}{2}}$ minimizes the worst-case variance of $\hat\ell_{t,i}$.
However, the reduced-variance estimators can take negative values, $\smash{\hat \ell_{t,i} \geq -\frac{1}{2}\lr{\frac{1}{w_{t,i}} - 1}}$, while the analysis relies on non-negativity of the loss estimators.
\citet{ZL19} show that negative loss estimators can be dealt with, as long as they satisfy \smash{$\hat\ell_{t,i}\geq - \frac{1}{2}\eta_t^{-2}$}. We achieve this by only reducing variance of the estimators with $w_{t,i} \geq \eta_t^2$.
}

The algorithm is provided in Algorithm~\ref{alg:OMD}. At every step, we choose a probability distribution over arms $w_t$. We add $\mathcal{I}_{\Simplex}$ to the regularizers $\Psi_t$, thereby ensuring that $w_t\in \Simplex$, where $\Simplex$ is the probability simplex. Note that the framework is equivalent to what \citet{abernethy2014online} call \Alg{Gradient-Based Prediction} (\Alg{Gbp}), where they replace $\nabla(\Psi_t+\mathcal{I}_{\Simplex})^*$ with suitable functions $\nabla\Phi_t :\mathbb{R}^K\rightarrow \Simplex$.
We adopt the notation $\Phi_t := (\Psi_t+\mathcal{I}_{\Simplex})^*$.
\begin{algorithm}
\caption{Online Mirror Descent for bandits}
\label{alg:OMD}
\DontPrintSemicolon
\LinesNumbered
\KwIn{$(\Psi_t)_{t=1,2,\dots}$}
{\bf Initialize:} $\hat{L}_0 = \mathbf{0}_K$ (where $\mathbf{0}_K$ is a vector of $K$ zeros)\;
\For{$t= 1,\ldots$}{
choose $w_t = \nabla(\Psi_t+\mathcal{I}_{\Simplex})^*(-\hat L_{t-1})$ ~~~\rev{\textit{\% see Alg.\ \ref{alg:newton} for an explicit calculation}}\;\label{ln:wt}
sample $\MyArm[t] \sim w_t$\;
observe $\MyLoss[t]$\;
\rev{use \eqref{eq:IW} or \eqref{eq:IIW} to} construct $\ImpLoss[t]$ \;
update $\hat L_t = \hat L_{t-1}+\ImpLoss[t]$\;
}
\end{algorithm}
\subsection{OMD with Tsallis Entropy Regularization}
We now consider a family of algorithms, which are regularized by the (negative) $\alpha$-Tsallis entropy $H_\alpha(x) := \frac{1}{1-\alpha}\left(1-\sum_i x_i^\alpha\right)$ \citep{tsallis1988possible}.
We change the scaling and add linear terms, resulting in the following regularizer \rev{with learning rate $\eta_t$}:
\begin{align*}
&\Psi(w) := -\sum_{\AnyArm}\frac{w_i^\alpha-\alpha w_i }{\alpha(1-\alpha)\xi_{i}},\\
&\Psi_{t}(w) := \frac{1}{\eta_t}\Psi(w).
\end{align*}
Unless stated otherwise, we assume that $\xi_i = 1$ for all $i$, which leads to \emph{symmetric regularization}. 
In the stochastic analysis of \Alg{$\alpha$-Tsallis-Inf} with $\alpha\neq \frac{1}{2}$, we take $\xi_i = \Delta_i^{1-2\alpha}$ which leads to \emph{asymmetric regularization}. 
Since the gaps are unknown, the latter is mainly interesting from a theoretical point of view.
%If $\xi_i = 1$ for all $i$ the regularization is \emph{symmetric}The factors $\xi_i$ allow for an {\it asymmetric regularization}.
%If not stated otherwise, we assume $\xi_i = 1$.
%Generally, we would either need prior information or oracle information to choose the appropriate factors $\xi_i$ when they are required.

The resulting family of algorithms is a subset of \Alg{Inf} \citep{audibert2009minimax}, which we call \Alg{$\alpha$-Tsallis-Inf}.
\Alg{$\alpha$-Tsallis-Inf} with symmetric regularization is related to the \Alg{Poly-INF} algorithm of \citet{audibert2009minimax,audibert2010regret} and equivalent to the \Alg{Gbp} algorithm proposed by \citet{abernethy2015fighting}.

As has been observed earlier \citep{abernethy2015fighting,agarwal2016corralling}, \Alg{$\alpha$-Tsallis-Inf} includes \Alg{Exp3}, which is based on the \rev{negative Shannon entropy $\sum_{i=1}^Kw_i\log(w_i)$ \citep{CT06}}, and algorithms based on the log-barrier potential \rev{$\sum_{i=1}^K-\log(w_i)$ \citep{FLST16}} as special cases.\rev{\footnote{\rev{We use $\log$ to denote the natural logarithm throughout the paper.}}}
This can be seen by adding a constant term to the regularizer, so that $\Psi(w) = -\sum_i \frac{w_i^\alpha - \alpha w_i - (1-\alpha)}{\alpha(1-\alpha)\xi_i}$, and taking the respective limits $\alpha\rightarrow 0$ and $\alpha\rightarrow 1$. \rev{It gives:}
\begin{align*}
&\lim_{\alpha\rightarrow 0}-\frac{w_i^\alpha-\alpha w_i -(1-\alpha)}{\alpha(1-\alpha)\xi_{i}} 
= \lim_{\alpha\rightarrow 0}-\frac{\log(w_i)w_i^\alpha-w_i +1}{(1-2\alpha)\xi_{i}}
= -\xi_{i}^{-1}(\log(w_i)-w_i+1),\\
&\lim_{\alpha\rightarrow 1}-\frac{w_i^\alpha-\alpha w_i -(1-\alpha)}{\alpha(1-\alpha)\xi_{i}} 
= \lim_{\alpha\rightarrow 1}-\frac{\log(w_i)w_i^\alpha-w_i +1}{(1-2\alpha)\xi_{i}} 
= \xi_{i}^{-1}(\log(w_i)w_i-w_i+1),\\
\end{align*}
which are within linear and constant terms identical to the log-barrier potential and the negative Shannon entropy, respectively.  Note that \rev{for symmetric regularization, neither the constant nor the linear terms influence the algorithm's choice of $w$, since it is normalized}.

\subsection{Implementation Details}
The weights $w_{t,i}$ in \Alg{Tsallis-Inf} are given implicitly through a solution of a constrained optimization problem:
\[
w_t = \argmax_{w \in \Simplex} \ip{w, -\hat L_t} + \frac{4}{\eta_t}\sum_i \sqrt{w_{i}}.
\]
The solution takes the form
\[
w_{t,i} = 4\lr{\eta_t\lr{\hat L_{t,i} - x}}^{-2},
\]
where the normalization factor $x$ is defined implicitly through the constraint 
\[\sum_i 4\lr{\eta_t \lr{\hat L_{t,i} - x}}^{-2} = 1.\]
The normalization factor can be efficiently approximated by Newton's Method, reaching a sufficient precision in very few iterations. Details of the computation are provided in Algorithm~\ref{alg:newton}.
\begin{algorithm}
\caption{Newton's Method approximation of $w_{t}$ in \Alg{Tsallis-Inf} \rev{($\alpha=\frac{1}{2}$)}}
\label{alg:newton}
\DontPrintSemicolon
\LinesNumbered
\KwIn{$x, \CumLoss[t], \LearningRate[t]$ \small{\it \%we use $x$ from the previous iteration as a warmstart}}
\Repeat{convergence}{
$\forall \AnyArm:\, \AnyProp[t] \leftarrow 4(\LearningRate[t](\AnyCumLoss[t]-x))^{-2}$ \; 
$x \leftarrow x- (\sum_\AnyArm w_{t,i}-1)/(\LearningRate[t]\sum_\AnyArm w_{t,i}^\frac{3}{2})$\;
}
\end{algorithm}

\section{Main Results}
\label{sec:theorems}

In this section we present \rev{our main result, the \Alg{Tsallis-INF} algorithm with $\alpha = \frac{1}{2}$ that achieves the optimal regret bounds in both adversarial and stochastic bandits. We show that it also achieves a logarithmic regret guarantee in the more general adversarial regime with a self-bounding constraint. In fact, the stochastic regret bound follows as a special case of the more general analysis. We then present a general analysis of \Alg{$\alpha$-Tsallis-INF} with $\alpha\in[0,1]$ and explain the intuition of why $\alpha=\frac{1}{2}$ works best.}

\rev{\subsection{Analysis of \Alg{Tsallis-INF} with $\alpha=1/2$}}
\label{sec:sto}

We show that %with a careful tuning 
\Alg{Tsallis-Inf} with $\alpha = \frac{1}{2}$ and symmetric regularizer achieves \rev{the optimal $\sqrt{T}$ regret scaling in the adversarial regime and simultaneously} $\log(T)$ regret \rev{scaling} in \rev{the adversarial regime with a self-bounding constraint}. 
\rev{The latter} ensures the same regret scaling in stochastic \rev{and stochastically constrained adversarial} environments as special cases. 
% At the moment, for any $\alpha \neq \frac{1}{2}$ the algorithm requires oracle access to the gaps for tuning the %learning rates. 
% parameters $\xi_i$ and the accompanying adversarial regret guarantee is suboptimal. 
%We leave it to future research to explore 
%Exploration of the possibility of replacing the true gaps with gap estimates, as \rev{done by} \citet{seldin2017improved}, is left for future work. 
%For $\alpha = \frac{1}{2}$ we provide tuning of the learning rate that %requires no knowledge of the gaps and 
%achieves the optimal problem-dependent constant of $\sum_{\AnyArm\neq\BestArmSto}\frac{1}{\AnyGap}$ in the regret bound. 
%We start with a refined analysis of $\alpha = \frac{1}{2}$ case in Theorem~\ref{th:easy} and then provide a general analysis for all $\alpha$ in Theorem~\ref{th:all}.
\rev{We analyse the algorithm with \eqref{eq:IW} and \eqref{eq:IIW} loss estimators. Both estimators achieve the optimal regret scaling in both regimes, but the \eqref{eq:IIW} estimator yields better constants. The results for the two estimators are presented alongside each other using cases brackets and marked by \eqref{eq:IW} and \eqref{eq:IIW}, respectively.}

\begin{theorem}
\label{th:easy}
The pseudo-regret of \Alg{Tsallis-Inf} with $\alpha = \frac{1}{2}$, symmetric regularization ($\xi_i=1$), and learning rate 
\[\eta_t = \begin{cases}
2\sqrt{\frac{1}{t}},& \text{for \eqref{eq:IW} estimators,}\\
4\sqrt{\frac{1}{t}},& \text{for \eqref{eq:IIW} estimators,}
\end{cases}
\]
in any adversarial bandit problem satisfies:
\begin{equation}\overline{Reg}_T \leq \begin{cases}
4\sqrt{KT}+1,& \text{with \eqref{eq:IW} estimators,}\\
2\sqrt{KT}+10K\log(T)+16,& \text{with \eqref{eq:IIW} estimators.}
\end{cases}
\label{eq:adversarial-bound}
\end{equation}
If there exists a vector $\Delta \in [0,1]^K$ with a unique zero entry $i^*$ (i.e., $\Delta_{i^*} = 0$ and $\Delta_i > 0$ for all $i\neq i^*$) and a constant $C$, such that the pseudo-regret at time $T$ satisfies
\begin{align}
\label{eq:mu-bound}
\overline{Reg}_T \geq \mathbb{E}\left[\sum_{t=1}^T \sum_{\AnyArm\neq\BestArmSto}\AnyProp[t]\AnyGap \right] - C\,,
\end{align}
then the pseudo-regret further satisfies
\begin{align*}
\overline{Reg}_T \leq &\begin{cases}
\StoRegretBaseIW + C,& \text{with \eqref{eq:IW},}\\[0.2cm]
\StoRegretBase + C,& \text{with \eqref{eq:IIW},}
\end{cases}
\end{align*}
where $\Deltamin := \min_{\Delta_i > 0} \Delta_i$. If $C$ satisfies
\[
\begin{array}{ll}
    C>\StoRegretBaseWithoutConstIW, & \text{for \eqref{eq:IW},}\\
    C>\StoRegretBaseWithoutConst, & \text{for \eqref{eq:IIW},}
\end{array}
\]
then the regret additionally satisfies
\begin{align*}
    \overline{Reg}_t \leq \begin{cases}
    2\sqrt{\lr{\StoRegretBaseWithoutConstIW}C} + \StoRegretConstIW, & \text{with \eqref{eq:IW},}\\[0.3cm]
    2\sqrt{\lr{\StoRegretBaseWithoutConst}C} + \StoRegretConst, & \text{with \eqref{eq:IIW}.}
    \end{cases}
\end{align*}
\end{theorem}

The proof is postponed to Section~\ref{sec:proofs}.
%\begin{remark}
\rev{We call the condition in equation \eqref{eq:mu-bound} a $(\Delta,C,T)$ \emph{self-bounding property} of the regret. As we have mentioned in Section~\ref{sec:problem},} in the stochastically constrained adversarial environments and stochastic bandits as their special case $\overline{Reg}_T = \mathbb{E}\left[\sum_{t=1}^T\sum_{\AnyArm\neq\BestArmSto} \AnyProp[t]\AnyGap \right]$, \rev{where $\Delta$ is the vector of suboptimality gaps, and under the assumption that the best arm is unique}, the condition in equation \eqref{eq:mu-bound} is satisfied with $C=0$. Thus, in the above regimes, the regret of \Alg{Tsallis-INF} with $\alpha=\frac{1}{2}$ and RV loss estimators is
\[
\overline{Reg}_T \leq \StoRegretBase.
\]
%\end{remark}
The worst case lower bound for \rev{stochastic multiarmed bandits (MAB)} with Bernoulli losses is achieved when the expectations of the losses are close to $\frac{1}{2}$. Let $\Delta$ denote the vector of gaps and let $\E[\AnyLoss[t]]=\frac{1}{2}+\AnyGap$. \rev{By adapting the well known divergence-dependent lower bound of \citet{lai1985asymptotically}, we can show that} for any consistent algorithm
\begin{align*}
    \lim_{||\Delta|| \rightarrow 0} \left(\left(\sum_{\AnyArm:\AnyGap>0}\frac{1}{\AnyGap}\right)^{-1}\liminf_{t\rightarrow\infty}\frac{\E\left[\overline{Reg}_t\right]}{\log(t)}\right) \geq \frac{1}{2}\,. 
\end{align*}
See Appendix~\ref{app:asym} for details. Therefore, the asymptotic regret upper bound of \Alg{Tsallis-Inf} \rev{with RV-estimators in the stochastic regime} is optimal within a multiplicative factor of \rev{$2$}, which is arguably a small price for a significant gain in robustness \rev{against adversaries}.
We leave it to future work to close the gap or prove that it is impossible to do so without \rev{compromising on} the adversarial guarantees. 

\rev{To the best of our knowledge, the leading constant 2 in the adversarial regret bound of \Alg{Tsallis-INF} with RV estimators (the bound in equation \eqref{eq:adversarial-bound}) provides the tightest adversarial regret guarantee known today. It matches the minimax adversarial lower bound in \citet[Theorem 6.1]{CBL06} within a multiplicative factor of less than 15. Under the assumption of known time horizon, \citet{ZL19} provide an adversarial regret bound with a leading constant of $\sqrt 2$. The $\sqrt 2$ multiplicative difference between their result and ours is the standard conversion rate between fixed-horizon and anytime regret bounds.}

\begin{remark} \rev{The assumption that $\Delta$ has a unique zero entry and the corresponding assumption on uniqueness of the best arm in the stochastically constrained adversarial setting} is a technical assumption we had to use in our proofs, but our experiments \rev{suggest} that this is an artifact of the analysis. 
\rev{We conjecture that it can be removed, but explain the challenges in achieving the goal in Section~\ref{sec:proofs}.}
\end{remark}

\subsection{A General Analysis of \Alg{$\alpha$-Tsallis-INF} with $\alpha\in[0,1]$}

Now we provide a general analysis of \Alg{$\alpha$-Tsallis-INF} with $\alpha\in[0,1]$ and then explain the intuition of why $\alpha=\frac{1}{2}$ works best. Since $\alpha\neq \frac{1}{2}$ anyway leads to suboptimal regret rates and in order to keep things simple, we restrict the general analysis to IW estimators. We note that in Theorem~\ref{th:easy}, the RV estimators helped improve the constants, but they did not change the rates. Therefore, we save the effort of optimising the constants in a priori suboptimal bounds. To keep things even simpler, we derive logarithmic bounds for stochastically constrained adversarial environments rather than the more general adversarial regime with a self-bounding constraint (technically speaking, we work with $C=0$).

Note that the adversarial analysis in Theorem~\ref{th:adv} and stochastic analysis in Theorem~\ref{th:all} consider different versions of \Alg{$\alpha$-Tsallis-INF}. The adversarial analysis uses \emph{symmetric} regularization, whereas stochastic analysis uses \emph{asymmetric} regularization. We get back to this point after we present the results. 

\subsubsection{Adversarial Regime}

\Alg{$\alpha$-Tsallis-Inf} \rev{with symmetric regularization} has been previously analyzed \rev{in the adversarial setting} by \citet{abernethy2015fighting} and \citet{agarwal2016corralling}. \citeauthor{abernethy2015fighting} provide a finite-time analysis for $\alpha\in (0,1]$, while \citeauthor{agarwal2016corralling} analyze the case of $\alpha=0$. The main contribution of the following theorem is that it provides a unified and anytime treatment for all $\alpha \in [0,1]$. The bound recovers the constants from \citeauthor{abernethy2015fighting} without the need of tuning the learning rate by the time horizon $T$.

\begin{theorem}
\label{th:adv}
For any $\alpha\in[0,1]$ and any adversarial bandit problem, the pseudo-regret of \Alg{$\alpha$-Tsallis-Inf} with symmetric regularizer ($\xi_i=1$), learning rate $\LearningRate[t] = \sqrt{\frac{K^{1-2\alpha}-K^{-\alpha}}{1-\alpha}\frac{1-t^{-\alpha}}{\alpha t}}$, \rev{and IW loss estimators} at any time $T$ satisfies
$$\overline{Reg}_T\leq 2\sqrt{\min\left\{\frac{1}{\alpha-\alpha^2},\frac{\log(K)}{\alpha},\frac{\log(T)}{1-\alpha}\right\}KT}+1.$$
(At the boundaries $\alpha = 0$ and $\alpha = 1$, the learning rates are defined by $\lim_{\alpha\rightarrow 0}\eta_t = \sqrt{\frac{(K-1)\log(t)}{t}}$ and $\lim_{\alpha\rightarrow 1}\eta_t = \sqrt{\frac{\log(K)(1-t^{-1})}{t}}$, respectively.)
\end{theorem}
The proof is postponed to Section~\ref{sec:proofs}.

\rev{
\subsubsection{Stochastically Constrained Adversarial Regime}
}

Now we present \rev{an analysis} of \Alg{$\alpha$-Tsallis-Inf} with $\alpha\in[0,1]$ \rev{and \emph{asymmetric} regularization} in the stochastically constrained adversarial setting. We let $\overline{t}=\max\{e,t\}$. 
For learning rates $\LearningRate[t]=\TimeDependentLearningRate$ and asymmetric regularizer \rev{with $\xi_i=\TimeIndependentLearningRate$ for $i\neq i^*$ and $\xi_{i^*} = \OptGap^{1-2\alpha}$,} where $\OptGap = \min_{\AnyArm\neq\BestArmSto}\AnyGap$, we prove the following theorem:

\begin{theorem}
\label{th:all}
For any $\alpha\in[0,1]$ and any stochastically constrained adversarial regime with a unique best arm \rev{(i.e., $\Delta_i > 0$ for all $i$ except a unique index $i^*$ for which $\Delta_{i^*}=0$)}, the pseudo-regret of \Alg{$\alpha$-Tsallis-INF} with learning rate $\LearningRate[t]=\TimeDependentLearningRate$ and asymmetric regularizer with parameters $\xi_i=\TimeIndependentLearningRate$ \rev{for $i\neq i^*$ and $\xi_{i^*} = \OptGap^{1-2\alpha}$} at any time $T$ satisfies
\begin{align*}
\overline{Reg}_T \leq \sum_{\AnyArm\neq\BestArmSto}\Bigg(\frac{(8 \min\{\frac{1}{1-\alpha},\log(T)\}+64)\log(T)}{\AnyGap}\Bigg) +\frac{ 16\log^4(\frac{16}{\OptGap^2}\log^2(\frac{16}{\OptGap^2}))}{\OptGap} +4.
\end{align*}
\end{theorem}
The proof is provided in Appendix~\ref{app:thm-all}.
\begin{remark}
We emphasize that for $\alpha \neq \frac{1}{2}$, the result in Theorem~\ref{th:all} requires knowledge of the gaps $\AnyGap$ for tuning the regularization parameters $\xi_i$. For $\alpha = \frac{1}{2}$, this knowledge is not required. Therefore, Theorem~\ref{th:all} is primarily interesting from the theoretical perspective of characterization of behavior of \Alg{$\alpha$-Tsallis-Inf} in stochastically constrained adversarial environments, whereas $\alpha = \frac{1}{2}$ is the only practically interesting value with the refined analysis in Theorem~\ref{th:easy}.
\end{remark}
\begin{remark}
\rev{For $\alpha \neq \frac{1}{2}$, the version \Alg{$\alpha$-Tsallis-INF} in Theorem~\ref{th:all} uses \emph{asymmetric} regularization, whereas \Alg{$\alpha$-Tsallis-INF} in Theorem~\ref{th:adv} uses \emph{symmetric} regularization. The corresponding learning rates also differ. Therefore, for $\alpha\neq\frac{1}{2}$, neither of the two versions of \Alg{$\alpha$-Tsallis-INF} achieves simultaneous optimality in the stochastic and adversarial setting.} 
In fact, the time dependence of the adversarial regret guarantee \rev{for \Alg{$\alpha$-Tsallis-INF} in Theorem~\ref{th:all}} is in the order of $T^\alpha + T^{1-\alpha}$.
\end{remark}
\rev{
\begin{remark}
We note that while Tsallis entropy with $\alpha=0$ corresponds to log-barrier potential used in \Alg{Broad}, and Tsallis entropy with $\alpha=1$ corresponds to entropic regularization used in \Alg{EXP3++}, the two algorithms (\Alg{Broad} and \Alg{EXP3++}) use \emph{symmetric} regularization, whereas \Alg{$\alpha$-Tsallis-INF} in Theorem~\ref{th:all} uses \emph{asymmetric} regularization. Therefore, there is no direct relation between the result of Theorem~\ref{th:all} and these two algorithms. In particular, \Alg{Broad} and \Alg{EXP3++} use other techniques to achieve slightly suboptimal, but simultaneous stochastic and adversarial regret guarantees (as described in Table~\ref{tab:summary}), which is not the case for \Alg{$\alpha$-Tsallis-INF} with asymmetric regularization in Theorem~\ref{th:all}.
\end{remark}
}

\subsection{Intuition Behind the Success of \Alg{Tsallis-Inf} with $\alpha=\frac{1}{2}$}
\label{sec:intuition}

It has been previously shown that regularization by Tsallis entropy with power $\alpha = 1/2$ leads to the minimax optimal regret rate in the adversarial regime \citep{audibert2009minimax}.
Here we provide some basic intuition on why the same value of $\alpha$ works well in the stochastic case. 
We also highlight the key breakthroughs that allow us to overcome challenges faced in prior work.

We start with a simple ``back of the envelope'' approximation of the form of the weights $w_t$ played by \Alg{Tsallis-Inf}. By definition of Algorithm~\ref{alg:OMD}, at round $t$ we have
\[
w_t = \argmax_{w \in \Simplex} \lrc{\dprod{w, -\hat L_{t-1}} + \frac{1}{\eta_t} \sum_i \frac{w_i^\alpha - \alpha w_i}{\alpha(1-\alpha)\xi_i}}.
\]
Taking a derivative of the Langrangian of the above expression with respect to $w_i$ and equating it to zero, we obtain
\[
-\hat L_{t-1,i} + \frac{1}{\eta_t(1-\alpha)\xi_i} (w_{t,i}^{\alpha-1} - 1) - \nu = 0,
\]
where $\nu$ is a Lagrange multiplier corresponding to the constraint that $w$ is a probability distribution. We can express $\nu$ as
\[
\nu = \frac{1}{\eta_t(1-\alpha) \xi_{i^*}} (w_{t,i^*}^{\alpha-1} - 1) - \hat L_{t-1,i^*}.
\]
For $i \neq i^*$ this gives
\begin{align}
w_{t,i} &= \lr{\eta_t(1-\alpha)\xi_i \lr{\hat L_{t-1,i} + \nu} + 1}^{\frac{1}{\alpha-1}}\notag\\
&= \lr{\eta_t(1-\alpha)\xi_i \lr{\hat L_{t-1,i} - \hat L_{t-1,i^*} + \frac{1}{\eta_t(1-\alpha) \xi_{i^*}} (w_{t,i^*}^{\alpha-1}  - 1)} + 1}^{\frac{1}{\alpha-1}}\notag\\
&= \lr{\eta_t(1-\alpha)\xi_i \lr{\hat L_{t-1,i} - \hat L_{t-1,i^*}} + \frac{\xi_i}{\xi_{i^*}}(w_{t,i^*}^{\alpha-1} - 1) + 1}^{\frac{1}{\alpha-1}}\notag\\
&\approx \lr{\eta_t(1-\alpha)\xi_i \lr{\hat L_{t-1,i} - \hat L_{t-1,i^*}}}^{\frac{1}{\alpha-1}},\notag%\label{eq:bote}
\end{align}
where the approximation holds because asymptotically the first term dominates the sum. A bit more explicitly, in order for the algorithm to deliver non-trivial regret guarantee, $w_{t,i^*}$ should be close to 1. Thus, the last two terms in the brackets are roughly a constant. At the same time, as we discuss below, the whole expression in the brackets must grow roughly as $(\Delta_i^2 t)^{1-\alpha}$. Thus, the first term must dominate. In the stochastic regime $\mathbb{E}\left[\hat L_{t,i} - \hat L_{t,i^*}\right]=\AnyGap t$. 
If we use this in our back-of-the-envelope calculation, we obtain that for $i\neq i^*$ in the stochastic regime $\E[w_{t,i}] \approx \E\left[\lr{\eta_t (1-\alpha) \xi_i (\hat L_{t-1,i} - \hat L_{t-1,i^*})}^\frac{1}{\alpha - 1}\right] \propto \lr{\eta_t \xi_i \AnyGap t}^{\frac{1}{\alpha - 1}}$. 
(Strictly speaking, when we take the expectation inside the power we obtain an inequality, but we ignore this detail in the high-level discussion. We also ignore the $(1-\alpha)$ factor, which can be seen as a constant for $\alpha < 1$.)

In order to achieve a regret rate of $\Theta(\sum_{\AnyArm\neq\BestArmSto} \frac{\log t}{\AnyGap})$ in the stochastic regime, the suboptimal arms should be explored at a rate of $\Theta(\frac{1}{\AnyGap^2t})$ per round (if $\E[w_{t,i}] = \Theta(\frac{1}{\AnyGap^2 t})$, then $\Delta_i \E\lrs{\sum_{s=1}^t w_{s,i}} = \Theta(\frac{\log t}{\AnyGap})$, as desired).
Exploring more than that leads to excessive regret from the exploration alone. 
Exploring less is also prohibitive, because it leads to an overly high probability of misidentifying the best arm. 
By looking at the approximation of $\E[w_{t,i}]$ from the previous paragraph, we obtain that we should have $\lr{\eta_t \xi_i \AnyGap t}^{\frac{1}{\alpha - 1}}\propto\frac{1}{\Delta_i^2t}$ or, equivalently, $\eta_t \xi_i \propto t^{-\alpha} \AnyGap^{1-2\alpha}$. 
The learning rate takes care of the time-dependent quantities, i.e., $\eta_t \propto t^{-\alpha}$, and $\xi_i$ should take care of the arm-dependent quantities, i.e., we should have $\xi_i \propto \AnyGap^{1-2\alpha}$.
Note that $\alpha = \frac{1}{2}$ leads to a symmetric regularizer $\Psi$ (i.e., $\xi_i = 1$), whereas for $\alpha \neq \frac{1}{2}$ the regularizer must be tuned using unknown gaps $\AnyGap$. 
The necessity to tune the regularizer based on unknown gaps has hindered progress in the work of \citet{wei2018more}, who used the log-barrier regularizer corresponding to $\alpha=0$.

Another crucial novelty behind the success of our analysis is basing it on the self-bounding property of the regret \rev{in equation \eqref{eq:mu-bound}. The new proof technique uses the same mechanism for controlling the regret in stochastic and adversarial regimes and we explain the intuition behind it in Section~\ref{sec:self-bound-intuition}.}  
The earlier approach by \citet{seldin2014one} and \citet{seldin2017improved} has controlled the regret in stochastic and adversarial regimes through separate mechanisms. 
The stochastic analysis was based on using empirical estimates of the gaps and high-probability control of the weights $w_{t,i}$. 
However, gap estimation is challenging, because the variance of $\AnyCumLoss[t]$ is of the order of $\sum_{s=1}^t \frac{1}{w_{s,i}}$. 
If the arms are played according to the target probabilities of $w_{t,i} \approx \frac{1}{t\AnyGap^2}$, then the variance of $(\AnyCumLoss[t] - \OptCumLossSto[t])$ is of the order of $\Theta(\AnyGap^2 t^2)$. 
This is prohibitively large, because the square root of the variance is of the same order as the expected cumulative gap and standard tools, such as Bernstein's inequality, cannot guarantee concentration of $(\AnyCumLoss[t] - \OptCumLossSto[t])$ around $\AnyGap t$. 
\citet{seldin2014one} have coped with this by mixing in additional exploration, but this has led to a regret growth rate of the order of $(\log T)^3$ in the stochastic regime. 
\citet{seldin2017improved} have mixed in less exploration and used unweighted losses for the gap estimates, which has decreased the regret growth rate down to $(\log T)^2$. 
It is currently unknown whether direct gap estimation can be further improved to support the desired $\log T$ stochastic regret rates. 
Additionally, existing oracle analysis in \citet[Theorem 2]{seldin2014one} and Theorem~\ref{th:all} here only support $(\log T)^2$ regret rate for \Alg{EXP3}-based algorithms (corresponding to $\alpha = 1$) in the stochastic regime. 
It is also unknown whether this rate can be improved. 
To summarize, the main breakthrough compared to this line of work is moving from $\alpha = 1$ to $\alpha = \frac{1}{2}$ and shifting from an analysis based on gap estimation to an analysis based on self-boundedness of the regret. 
The proposed algorithm does not mix in any additional exploration. 

\section{Additional Intermediate Regimes Between Stochastic and Adversarial}
\label{sec:int}

In this section, we \rev{show that stochastic bandits with adversarial corruptions proposed by \citet{LML18} are also a special case of an adversarial environment with a self-bounding constraint.} We further propose an extension of their regime by combining it with a stochastically constrained adversary. \rev{We show that the combination is also a special case of an adversarial environment with a $(\Delta, 2C, T)$ self-bounding constraint, where \Alg{Tsallis-INF} achieves logarithmic regret.} We finish the section with an open question on whether \Alg{Tsallis-Inf} can achieve logarithmic regret guarantees in the intermediate regimes defined by \citet{seldin2014one}. % The main message of this section is that whenever there is a gap in performance between the optimal and suboptimal actions, \Alg{$\frac{1}{2}$-Tsallis-INF} is able to exploit it and achieve ``logarithmic'' regret.

\subsection{Stochastic Bandits with Adversarial Corruptions}

\citet{LML18} have proposed a regime in which an adversary is allowed to make corruptions to an otherwise stochastic environment. 
Let $\overline{\mathcal{L}}_T = (\bar \ell_1, \dots, \bar \ell_T)$ and $\mathcal{L}_T = (\ell_1, \dots, \ell_T)$ be two sequences of losses, then the amount of corruption is measured by $\sum_{t=1}^T \|\bar \ell_t - \ell_t\|_\infty$. 

Let $\overline{\mathcal{L}}_T$ be a sequence of losses generated by a stochastically constrained adversary with best arm $i^*$ and gaps $\AnyGap$, and let $\mathcal{L}_T$ be its adaptively corrupted version with corruption amount bounded by $C$. The regret of an algorithm executed on $\mathcal{L}_T$ satisfies
\begin{align}
\overline{Reg}_T =& \max_i\E\left[\sum_{t=1}^T\MyLoss[t]-\AnyLoss[t]\right]\geq \E\left[\sum_{t=1}^T\MyLoss[t]-\BestLossSto[t]\right]\notag\\
=&\E\left[\sum_{t=1}^T\overline{\ell}_{t,I_t}-\overline{\ell}_{t,i^*}\right] + \E\left[\sum_{t=1}^T\MyLoss[t]-\overline{\ell}_{t,I_t}\right]+\E\left[\sum_{t=1}^T\overline{\ell}_{t,i^*}-\BestLossSto[t]\right]\notag\\
\geq & \sum_{t=1}^T\sum_{\AnyArm\neq \BestArmSto} \AnyGap\E[\AnyProp[t]]-2C.\label{eq:corruptions}
\end{align}
\rev{Thus, a stochastically constrained adversary with adversarial corruptions is an adversarial regime with a $(\Delta, 2C, T)$ self-bounding constraint.} This leads to a direct corollary of Theorem~\ref{th:easy}, which improves \rev{upon the pseudo-regret bounds} of \citet{LML18} and \citet{GKT19}, \rev{the latter providing an $\mathcal{O}\left(\sum_{i\neq i^*}\frac{\log(T)}{\Delta_i}+KC\right)$ guarantee}.
\rev{We note that \citet{LML18} and \citet{GKT19} do not assume uniqueness of the best arm and also provide high-probability regret guarantees, but they only consider the more restricted stochastic setting with adversarial corruptions rather than stochastically constrained adversarial setting with adversarial corruptions.}

\begin{corollary}
The regret of \Alg{Tsallis-Inf} in a stochastically constrained adversarial environment with a unique best arm \rev{$i^*$, adaptively corrupted with} corruption amount bounded by $C$ satisfies
\rev{\[
\overline{Reg}_T = \mathcal{O}\left(\sum_{i\neq i^*}\frac{\log(T)}{\Delta_i}+\sqrt{\sum_{i\neq i^*}\frac{\log(T)}{\Delta_i}C}\right)\,.
\]}
%\[
%\overline{Reg}_T \leq \StoRegretBase +  2C\,.
%\]
%If $2C>\StoRegretBaseWithoutConst$, then the regret is bounded by
%\begin{align*}
%    \overline{Reg}_t \leq &\StoRegretConst +\sqrt{8\lr{\StoRegretBaseWithoutConst}C}\,.
%\end{align*}
\label{cor:corrupted}
\end{corollary}
\begin{remark}
We emphasize that the assumption of best arm uniqueness is on the stochastically constrained adversary \emph{before} corruption. After the adaptive corruption, it is allowed to have multiple best arms and the identity of the best arm is allowed to change.
\end{remark}
\begin{proof} 
    By equation \eqref{eq:corruptions}, the self-bounding condition \eqref{eq:mu-bound} of Theorem~\ref{th:easy} is satisfied with $\Delta$ being the vector of gaps of the underlying stochastically constrained adversary and the constant being $2C$. Thus, with RV loss estimators for $2C\leq\StoRegretBaseWithoutConst$ \Alg{Tsallis-INF} achieves
\begin{align*}
\overline{Reg}_T \leq &\StoRegretBase + 2C
\end{align*}
and otherwise
\begin{align*}
    \overline{Reg}_t \leq &2\sqrt{\lr{\StoRegretBaseWithoutConst}2C} + \StoRegretConst\,.
\end{align*}    
\end{proof}

\subsection{Open Problem: The Performance in \texorpdfstring{\citeauthor{seldin2014one}}{Seldin and Slivkins}' Environments}
\label{sec:open}

\citet{seldin2014one} define \emph{moderately contaminated stochastic regime} and \emph{an adversarial regime with a gap}. In the moderately contaminated stochastic regime, the adversary is allowed to change up to $\frac{t\Delta_i}{4}$ arbitrarily selected observations for a suboptimal arm $i$ and up to $\frac{t\Deltamin}{4}$ observations for the optimal arm $i^*$ (where $\Deltamin = \min_{\Delta_i > 0} \Delta_i$). The \rev{logic behind the definition} is that in expectation, the adversary can reduce the gap $\Delta_i$ by a factor of 2, but cannot eliminate it completely. The adversarial regime with a gap is an adversarial regime, where starting from a certain time $\tau$ (unknown to the algorithm) the cumulative loss of an optimal arm maintains a certain gap $\Delta_\tau$ to all other arms until the end of the game. \citeauthor{seldin2014one} show that their \Alg{EXP3++} algorithm achieves ``logarithmic'' regret in both regimes. Note that in the moderately contaminated stochastic regime, the amount of contamination is allowed to grow linearly with time. While the regime could be seen as a special case of stochastic bandits with adversarial corruptions discussed earlier, the regret bound in Corollary~\ref{cor:corrupted} only supports ``logarithmic'' regret for ``logarithmic'' amount of corruption $C$. So far we have been unable to obtain ``logarithmic'' regret guarantees for \Alg{Tsallis-Inf} in the intermediate regimes of \citeauthor{seldin2014one} (the analysis proposed in \citet{ZS19} is incorrect). The challenge is that the gaps are defined through cumulative rather than instantaneous quantities. Deriving ``logarithmic'' regret guarantees for \Alg{Tsallis-Inf} in these regimes is an interesting open problem.

\section{Dueling Bandits}
\label{sec:duel}

In the sparring approach to stochastic utility-based dueling bandits, \citep{ailon2014reducing} each side in the sparring can be modeled as a stochastically constrained adversarial environment. This makes it a perfect application domain for \Alg{Tsallis-Inf}.
%An application of the stochastically constrained adversary is utility-based dueling bandits with linear link function \citep{ailon2014reducing}.
The problem is defined by $K$ arms with utilities $u_i\in [0,1]$.
At each round, an agent has to select two arms, $\MyArm[t]$ and $\OtherArm[t]$, to ``duel''.
The feedback is the winner $W_t$ of the ``duel'', which is chosen according to $\mathbb{P}[W_t=\MyArm[t]] = \frac{1+u_{\MyArm[t]}-u_{\OtherArm[t]}}{2}$.
The regret is defined by the distance to the optimal utility:
\begin{align*}
    \overline{Reg}_T = \sum_{t=1}^T 2u_{\BestArm}-\mathbb{E}\left[\sum_{t=1}^T (u_{\MyArm[t]}+u_{\OtherArm[t]})\right].
\end{align*}
In the adversarial version of the problem, the utilities $u_i$ are not constant but time dependent, $u_{t,i}$, and selected by an adversary.
The regret in this case is the difference to the optimal utility in hindsight:
\begin{align*}
    \overline{Reg}_T = \max_\AnyArm\mathbb{E}\lrs{\sum_{t=1}^T 2u_{t,\AnyArm}}-\mathbb{E}\left[\sum_{t=1}^T (u_{t,\MyArm[t]}+u_{t,\OtherArm[t]})\right].
\end{align*}
\citet{ailon2014reducing} have proposed the \Alg{Sparring} algorithm, in which two black-box MAB algorithms spar with each other.
The first algorithm selects $\MyArm[t]$ and receives the loss $\MyLoss[t]=\ind[][W_t \neq \MyArm[t]]$.
The second algorithm selects $\OtherArm[t]$ and receives the loss $\OtherLoss[t]=\ind[][W_t \neq \OtherArm[t]]$.
They have shown that the regret is the sum of individual regret values for both MABs, thereby recovering $\mathcal{O}(\sqrt{KT})$ regret in the adversarial case if MABs with $\mathcal{O}(\sqrt{KT})$ adversarial regret bound are used.
In the stochastic case, each black-box MAB \rev{plays in} a stochastically constrained \rev{adversarial environment because the relative winning probability of the arms stays fixed, but depending on the arm choice of the sparring partner, the baseline shifts up and down}.
Since no algorithm has been known to achieve $\log(T)$ regret \rev{in stochastically constrained adversarial setting}, \citet{ailon2014reducing} provide no analysis of \Alg{Sparring} in the stochastic case. \rev{Indeed, as we demonstrate in our experiments, standard algorithms for stochastic multi-armed bandits, such as \Alg{UCB} or \Alg{Thompson Sampling}, may exhibit almost linear regret in stochastically constrained adversarial setting and, therefore, are not suitable for sparring.}

By applying Theorem~\ref{th:easy}, we directly obtain the following corollary.
\begin{corollary}
\rev{In a utility-based dueling bandit problem} \Alg{Sparring} with two independent versions of \Alg{Tsallis-Inf} suffers a regret of
\begin{align*}
    \overline{Reg}_T \leq \mathcal{O}\left(\sum_{\AnyArm:\AnyGap>0}\frac{\log(T)}{\AnyGap}\right) 
\end{align*}
in the stochastic case \rev{with a unique best arm} and 
\begin{align*}
    \overline{Reg}_T \leq \mathcal{O}\left(\sqrt{KT}\right) 
\end{align*}
in the adversarial case.
\end{corollary}

\section{Proofs}
\label{sec:proofs}
In this section, we first revise the general proof framework of OMD and provide a compact summary of how to modify it to obtain stochastic guarantees. 
Afterward, we provide proofs of Theorems~\ref{th:easy} and \ref{th:adv}.
A proof of Theorem~\ref{th:all} along with proofs of all the lemmas in this section are provided in the appendix.

\subsection{High-Level Overview of OMD Modification for Stochastic Analysis}
\label{sec:self-bound-intuition}

We follow the standard OMD analysis \citep[Chapter 28]{LS19bandit-book} and introduce the potential function $\Phi_t(-L) = \max_{w\in\Simplex}\{\left\langle w,-L\right\rangle - \Psi_t(w)\}$ to decompose the regret into {\it stability} and {\it penalty} terms.
\begin{equation}
\begin{aligned}
\overline{Reg}_T &= \E\left[\sum_{t=1}^T \left(\MyLoss[t]-\OptLoss[t]\right)\right] \\
&=\underbrace{\E\left[\sum_{t=1}^T\MyLoss[t]+ \Phi_t(- \CumLoss[t]) - \Phi_t(-\CumLoss[t-1]) \right]}_{stability}+\underbrace{\E\left[\sum_{t=1}^T - \Phi_t(- \CumLoss[t]) + \Phi_t(-\CumLoss[t-1]) -\OptLoss[t]\right]}_{penalty}.
\end{aligned}
\label{eq:regsplit}
\end{equation}
The OMD analysis bounds the {\it stability} and {\it penalty} terms separately.
For Tsallis-entropy regularizers, \citet{abernethy2015fighting} have proven the following bounds:
\begin{align*}
&stability \leq \sum_{t=1}^T \LearningRate[t]\sum_{\AnyArm=1}^K f(\E[\AnyProp[t]]),\\
&penalty \leq \sum_{t=1}^T (\LearningRate[t+1]^{-1}-\LearningRate[t]^{-1})\sum_{\AnyArm=1}^K g(\E[\AnyProp[t]]),
\end{align*}
where $f(x)$ and $g(x)$ are proportional to $x^{1-\alpha}$ and $x^{\alpha}$, respectively. 
Adversarial bounds that scale with $\sqrt{T}$ are obtained by applying $\sum_{i=1}^K f(\E[\AnyProp[t]]) \leq \max_{w\in\Simplex}\sum_{i=1}^K f(w)$, $\sum_{i=1}^K g(\E[\AnyProp[t]]) \leq \max_{w\in\Simplex}\sum_{i=1}^K g(w)$, and choosing an appropriate learning rate. \rev{In particular, for $\alpha=1/2$ we have $f(x) \propto \sqrt{x}$ and $g(x)  \propto \sqrt{x}$ and we use $\eta_t  \propto 1/\sqrt{t}$, for which $\eta_{t+1}^{-1} - \eta_t^{-1} = \Theta(1/\sqrt{t})$. This gives
\[
\overline{Reg}_T \leq \sum_{t=1}^T c\frac{1}{\sqrt{t}} \sum_{i=1}^K \sqrt{\E[\AnyProp[t]]} \leq \sum_{t=1}^T c\frac{1}{\sqrt{t}} \max_{z \in \Delta^{K-1}} \sum_{i=1}^K \sqrt{z_i} \leq  \sum_{t=1}^T c\frac{1}{\sqrt{t}} \sqrt{K} \leq 2c \sqrt{KT},
\]
where $c$ is a small constant and we replace $\E[\AnyProp[t]]$ with $z_i$ in the maximization.
}

The main insight of the paper is that the same framework can be used to obtain logarithmic bounds in the stochastic case.
\rev{ 
The key novelty is that if we constrain the maximization of $\E[\AnyProp[t]]$ by the self-bounding property of the regret \eqref{eq:mu-bound}, the space of solutions excludes the worst-case scenario, where the regret grows with the square root of the time horizon. For simplicity, we first explain the approach with $C=0$. By the self-bounding property \eqref{eq:mu-bound}, we then have $\Reg\geq \sum_{t=1}^T \sum_{i\neq i^*} \Delta_i \E[\AnyProp[t]] = \sum_{t=1}^T \sum_{i} \Delta_i \E[\AnyProp[t]]$ (since $\Delta_{i^*}=0$ by definition), which we can use to write
\begin{equation}
\Reg \leq 2\Reg - \sum_{t=1}^T \sum_{i} \Delta_i \E[\AnyProp[t]].
\label{eq:bound-intuition}
\end{equation}
The negative contributions $-\Delta_i \E[\AnyProp[t]]$ are used to achieve better control of the growth of $\E[\AnyProp[t]]$, but they are only helpful for $i$ with $\Delta_i > 0$, i.e., only for $i\neq i^*$. Therefore, we derive refined} bounds for the stability and penalty terms:
\begin{align*}
&stability \leq \sum_{t=1}^T \LearningRate[t]\sum_{\AnyArm\neq\BestArmSto} \tilde{f}(\E[\AnyProp[t]]),\\
&penalty \leq \sum_{t=1}^T (\LearningRate[t+1]^{-1}-\LearningRate[t]^{-1})\sum_{\AnyArm\neq \BestArmSto} g(\E[\AnyProp[t]]),
\end{align*}
\rev{where the summation excludes the best arm $i^*$, which has no negative contribution in \eqref{eq:bound-intuition}. The cost of excluding the best arm is an addition of a linear term to $f: \tilde{f}(x) = f(x)+c'x \leq (1+c')f(x)$, where $c'$ is a small constant. In particular, for $\alpha = \frac{1}{2}$ and learning rate $\eta_t  \propto 1/\sqrt{t}$ we have
\begin{align*}
\Reg &\leq 2\Reg - \sum_{t=1}^T \sum_{i\neq i^*} \Delta_i\E[\AnyProp[t]]\\
&\leq \sum_{t=1}^T \sum_{i\neq i^*} \lr{2c_1\frac{1}{\sqrt t}\sqrt{\E[\AnyProp[t]]} - \Delta_i \E[\AnyProp[t]]}\\
&\leq \sum_{t=1}^T \sum_{i\neq i^*} \max_z \lr{2c_1\frac{1}{\sqrt t}\sqrt{z}- \Delta_i z}\\
&\leq \sum_{t=1}^T \sum_{i\neq i^*} \frac{c_2}{\Delta_i t}\\ 
&= O\lr{\sum_{i\neq i^*} \frac{\log T}{\Delta_i}},
\end{align*}
where $c_1$ and $c_2$ are small constants and in the second line we used the refined stability and penalty bounds to bound $2\Reg$. The negative contribution is exploited in the maximization in the third line, which is now done coordinate-wise and the constraint that $w_t$ is a probability distribution is dropped.

We assume uniqueness of the zero-entry in $\Delta$, because currently we are only able to exclude one arm from the summation in the refined bound on stability. Had there been multiple arms with $\Delta_i=0$, they would have no negative contributions to control $\E[\AnyProp[t]]$. The challenge in excluding more than one arm from the summation is explained in Lemma~\ref{lem:stability}, where we derive the refined bound.

In the more general analysis with $C>0$, we introduce a parameter $\lambda$ and write $\Reg\leq(1+\lambda)\Reg - \lambda\lr{\sum_{t=1}^T\sum_{i\neq i^*} \Delta_i\E[\AnyProp[t]] - C}$. We use $\lambda$ for optimizing the dependence on $C$. The parameter $\lambda$ can also be seen as a Lagrange multiplier in a constrained optimization problem of maximizing the regret bound (stability bound + penalty bound) under the self-bounding constraint that (stability bound + penalty bound) $\geq \sum_{t=1}^T \sum_{i\neq i^*} \Delta_i \E[\AnyProp[t]]-C$.
}

\subsection{Key Lemmas}

The proofs of Theorems~\ref{th:easy}, \ref{th:adv}, and \ref{th:all} are based on the following two lemmas that bound the {\it stability} and {\it penalty} terms. 
The proofs of the lemmas are provided in Appendix~\ref{App:lemma-proofs}.
\rev{
\begin{lemma}
\label{lem:stability}
For a positive learning rate, the instantaneous \emph{stability} of \Alg{$\alpha$-Tsallis-Inf} satisfies at any time $t$ 
\begin{align*}
    \E\left[\MyLoss[t] +\Phi_{t}(-\CumLoss[t])-\Phi_{t}(-\CumLoss[t-1])\right] \leq \begin{cases}
    \min\left\{\sum_{\AnyArm=1}^K\frac{\LearningRate[t]\xi_i}{2}\E\left[\AnyProp[t]\right]^{1-\alpha},1\right\},&
    \mbox{if 1.}\\
    \frac{\eta_t^2}{2}+\sum_{i=1}^K\frac{\eta_t}{2}\E[\AnyProp[t]]^\frac{1}{2}(1-\E[\AnyProp[t]]),&
    \mbox{if 2.}\\
    \frac{7\eta_t^2}{8}K+\sum_{i=1}^K\frac{\eta_t}{8}\E[\AnyProp[t]]^\frac{1}{2}(1-\E[\AnyProp[t]]),&
    \mbox{if 3.}\\
    \sum_{\AnyArm\neq j}\lr{\frac{\LearningRate[t]\xi_i}{2}\E\left[\AnyProp[t]\right]^{1-\alpha} + \frac{\LearningRate[t](\xi_i+2\xi_j)}{2}\E\left[\AnyProp[t]\right]},&
    \mbox{if 4.}\,,
    \end{cases}
\end{align*}
where 
\begin{enumerate}
    \item $\hat L_t$ is based on IW estimators. The inequality holds for any $\eta_t>0$ and $\alpha\in[0,1]$.
    \item $\hat L_t$ is based on IW estimators,  $1 \geq \eta_t > 0$, and $\smash{\alpha=\frac{1}{2}}$.
    \item $\hat L_t$ is based on RV estimators,  $1 \geq \eta_t > 0$, and $\smash{\alpha=\frac{1}{2}}$.
    \item $\hat L_t$ is based on IW estimators and $\smash{\eta_t\xi_i\leq\frac{1}{4}}$ for all $i$. The inequality holds for any $j$ and $\alpha\in[0,1]$.
\end{enumerate}
\end{lemma}}
The first part of the Lemma is due to \citet{abernethy2015fighting}.
The remaining parts are non-trivial refinements that are crucial for our analysis, as outlined in the previous section. \rev{The first inequality is used in the proof of Theorem~\ref{th:adv}, the second and third inequalities are used for the two results in Theorem~\ref{th:easy}, and the last inequality is used in the proof of Theorem~\ref{th:all}. In the proof of Theorem~\ref{th:easy} we use $\E[\AnyProp[t]]^\frac{1}{2}(1-\E[\AnyProp[t]]) \leq \E[\AnyProp[t]]^\frac{1}{2}$ for $i\neq i^*$ and  $\E[\BestPropSto[t]]^\frac{1}{2}(1-\E[\BestPropSto[t]) \leq (1-\E[\BestPropSto[t]]) = \sum_{i\neq i^*} \E[\AnyProp[t]]$ for $i^*$. This eliminates $\E[w_{t,i^*}]$ from the regret bound and allows to exploit the self-bounding property. The approach only allows to eliminate one arm from the regret bound, which is the reason we rely on the assumption of uniqueness of the best arm.}

\begin{lemma}
\label{lem:penalty}
For any $\alpha \in [0,1]$ and any unbiased loss estimators the \emph{penalty} term of \Alg{$\alpha$-Tsallis-Inf} satisfies:
\begin{enumerate}
\item For the symmetric regularizer and a non-increasing sequence of positive learning rates $\eta_1,\eta_2,\dots$
\begin{align*}
\E\left[\sum_{t=1}^T \lr{\Phi_t(-\CumLoss[t-1])-\Phi_t(-\CumLoss[t]) -\OptLoss[t]}\right] \leq \frac{(K^{1-\alpha}-1)(1-T^{-\alpha})}{(1-\alpha)\alpha \LearningRate[T]} + 1.
\end{align*}
\item For an arbitrary regularizer, a non-increasing sequence of positive learning rates $\eta_1,\eta_2,\dots$, and any $x\in [1,\infty]$  
\end{enumerate}
\vspace{-0.3cm}
\begin{align*}
\hspace{0.4cm}&\E\left[\sum_{t=1}^T \lr{\Phi_t(-\CumLoss[t-1])-\Phi_t(-\CumLoss[t]) -\OptLoss[t]}\right] \\
&\leq \frac{1-T^{-\alpha x}}{\alpha}\sum_{\AnyArm\neq\BestArm}\lr{\frac{\E[\AnyProp[1]]^\alpha-\alpha\E[\AnyProp[1]]}{\LearningRate[1]\xi_i(1-\alpha)}+\sum_{t=2}^T\lr{\frac{1}{\LearningRate[t]}-\frac{1}{\LearningRate[t-1]}}\frac{\E[\AnyProp[t]]^\alpha-\alpha\E[\AnyProp[t]]}{\xi_i(1-\alpha)}} + T^{1-x}.
\end{align*}
\end{lemma}
The first part of the Lemma is a straightforward improvement of the penalty bound in \citet{abernethy2015fighting} with the techniques from \citet{agarwal2016corralling}.
The second part is again a crucial refinement. 
\rev{It is obtained by exploiting the negative contribution of $\Psi_T(\mathbf{e}_{i^*})$ in an intermediate step of the proof, 
which \citet{abernethy2015fighting} trivially bounded by $0$. 
}

\subsection{Proofs of Theorems~\ref{th:easy} and \ref{th:adv}}

Now we are ready to present proofs of the main theorems.

\begin{proof}{\bf  of Theorem~\ref{th:easy}}
We provide a proof of regret bounds for \Alg{Tsallis-INF} with RV estimators.
The analysis of \Alg{Tsallis-INF} with IW estimators in the adversarial case is analogous to the proof of Theorem~\ref{th:adv} and under the self-bounding constraint \eqref{eq:mu-bound}, it is analogous to the analysis of RV estimators with the bound in Part~3 of Lemma~\ref{lem:stability} replaced by the bound in Part~2. Therefore, the proofs of both results for the IW estimators are omitted.

To analyze the regret, we start by bounding the {\it stability} term.
We use Lemma~\ref{lem:stability}. For $t < 16$ we have $\eta_t>1$ and the RV estimators are equivalent to IW estimators.
Thus, we can apply the first part of the lemma to bound the instantaneous stability by $1$.
For $t\geq 16$, we use the third part of the lemma.
\begin{align}
stability &= \E\left[\sum_{t=1}^T\MyLoss[t]+ \Phi_t(- \CumLoss[t]) - \Phi_t(-\CumLoss[t-1]) \right]\notag\\
&\leq 15+\sum_{t=16}^T\left(\frac{7\eta_t^2}{8}K+\sum_{i=1}^K\frac{\eta_t}{8}\sqrt{\E[\AnyProp[t]]}(1-\E[\AnyProp[t]])\right)\notag\\
&\leq 15+14K\log(T)+\sum_{t=16}^T\sum_{i=1}^K\frac{\sqrt{\E[\AnyProp[t]]}(1-\E[\AnyProp[t]])}{2\sqrt{t}}\,.\label{eq:stability departure}
\end{align}

\paragraph{Adversarial bound.}
We  bound $\sum_{i=1}^K\sqrt{\E[\AnyProp[t]]}(1-\E[\AnyProp[t]])\leq \sum_{i=1}^K\sqrt{\E[\AnyProp[t]]} \leq \sqrt{K}$, where the last step holds by simple maximization. Then we have
\begin{align*}
stability &\leq 15+14K\log(T)+\sum_{t=16}^T\frac{\sqrt{K}}{2\sqrt{t}}\\
&\leq 15+14K\log(T)+\sqrt{KT}\,.
\end{align*}
For the \emph{penalty} term, we use the first part of Lemma~\ref{lem:penalty} to obtain
\begin{align*}
    penalty \leq \sqrt{KT} +1\,.
\end{align*}
Combining \emph{stability} and \emph{penalty} completes the proof.

\paragraph{Bound under the self-bounding constraint \eqref{eq:mu-bound}.}
We continue bounding the \emph{stability} up from equation~\eqref{eq:stability departure}.
For $i\neq i^*$, we use $\sqrt{\E[\AnyProp[t]]}(1-\E[\AnyProp[t]])\leq \sqrt{\E[\AnyProp[t]]}$. For $i^*$, we use $\sqrt{\E[\BestPropSto[t]]}(1-\E[\BestPropSto[t]]) \leq (1-\E[\BestPropSto[t]]) = \sum_{i\neq i^*}\E[\AnyProp[t]]$. For a constant $0 < \lambda \leq 1$ that will be specified at a later stage of the proof and $t \leq T_0 = \left\lceil(\frac{1}{\lambda\OptGap})^2\right\rceil$, we further bound the last expression as $\sum_{i\neq i^*}\E[\AnyProp[t]] \leq 1$. Altogether, this gives
\begin{align*}
    \sum_{t=16}^T\sum_{i=1}^K\frac{\sqrt{\E[\AnyProp[t]]}(1-\E[\AnyProp[t]])}{2\sqrt{t}}
    &\leq \sum_{t=16}^{T_0}\frac{1}{2\sqrt{t}} + \sum_{i\neq i^*}\left(\sum_{t=T_0+1}^{T}\frac{\E[\AnyProp[t]]}{2\sqrt{t}}+\sum_{t=16}^{T}\frac{\sqrt{\E[\AnyProp[t]]}}{2\sqrt{t}}\right)\\
    &\leq \sqrt{T_0} + \sum_{i\neq i^*}\left(\sum_{t=1}^{T}\frac{\sqrt{\E[\AnyProp[t]]}}{2\sqrt{t}}+\sum_{t=T_0+1}^{T}\frac{\E[\AnyProp[t]]}{2\sqrt{t}}\right)\,
\end{align*}
and
\begin{align*}
    stability \leq 15+14K\log(T)+\sqrt{T_0} + \sum_{i\neq i^*}\left(\sum_{t=1}^{T}\frac{\sqrt{\E[\AnyProp[t]]}}{2\sqrt{t}}+\sum_{t=T_0+1}^{T}\frac{\E[\AnyProp[t]]}{2\sqrt{t}}\right)\,.
\end{align*}
In order to bound the \emph{penalty} term, we use the second part of Lemma~\ref{lem:penalty} with $x=\infty$. At the end of the derivation we use Lemma~\ref{lem:sum}, by which $\sum_{t=2}^\infty\lr{\sqrt{t}- \sqrt{t-1} - \frac{1}{2\sqrt{t}}} \leq \frac{1}{4}$.
%Note that $2\sqrt{T} = 2+\int_{t=1}^T\frac{1}{\sqrt{t}}\,dt\leq 2+\sum_{t=1}^T\frac{1}{\sqrt{t}}$. 
\begin{align}
penalty &= \E\left[\sum_{t=1}^T - \Phi_t(- \CumLoss[t]) + \Phi_t(-\CumLoss[t-1]) -\OptLossSto[t] \right]\notag\\
&\leq 4\sum_{\AnyArm\neq\BestArmSto}\lr{\frac{\sqrt{\E[\AnyProp[1]]}-\frac{1}{2}\E[\AnyProp[1]]}{\LearningRate[1]}+\sum_{t=2}^T\lr{\frac{1}{\LearningRate[t]}-\frac{1}{\LearningRate[t-1]}}\left(\sqrt{\E[\AnyProp[t]]}-\frac{1}{2}\E[\AnyProp[t]]\right)}\notag\\
&= \sum_{\AnyArm\neq\BestArmSto}\lr{\lr{\sqrt{\E[\AnyProp[1]]}-\frac{1}{2}\E[\AnyProp[1]]}+\sum_{t=2}^T\lr{\sqrt{t}-\sqrt{t-1}}\lr{\sqrt{\E[\AnyProp[t]]}-\frac{1}{2}\E[\AnyProp[t]]}}\notag\\
&=\sum_{t=1}^T\lr{\sum_{\AnyArm\neq\BestArmSto}\frac{\sqrt{\E[\AnyProp[t]]}-\frac{1}{2}\E[\AnyProp[t]]}{2\sqrt{t}}}+\sum_{\AnyArm\neq\BestArmSto}\Bigg(\frac{\sqrt{\E[\AnyProp[1]]}-\frac{1}{2}\E[\AnyProp[1]]}{2}\notag\\ &\hspace{4cm}+\sum_{t=2}^T\lr{\sqrt{t}-\sqrt{t-1}-\frac{1}{2\sqrt{t}}}\lr{\sqrt{\E[\AnyProp[t]]}-\frac{1}{2}\E[\AnyProp[t]]}\Bigg)\notag\\
&\leq \sum_{t=1}^T\lr{\sum_{\AnyArm\neq\BestArmSto}\frac{\sqrt{\E[\AnyProp[t]]}-\frac{1}{2}\E[\AnyProp[t]]}{2\sqrt{t}}} +\lr{\frac{1}{2} +\sum_{t=2}^T\lr{\sqrt{t}-\sqrt{t-1}-\frac{1}{2\sqrt{t}}}}\sqrt{K}\notag\\
&\leq \sum_{t=1}^T\lr{\sum_{\AnyArm\neq\BestArmSto}\frac{\sqrt{\E[\AnyProp[t]]}-\frac{1}{2}\E[\AnyProp[t]]}{2\sqrt{t}}}+\frac{3}{4}\sqrt{K}. \notag
\end{align}
Combining \emph{penalty} and \emph{stability} gives the bound
\begin{align*}
    \overline{Reg}_T\leq \sum_{i\neq i^*}\left(\sum_{t=1}^T \frac{\sqrt{\E[\AnyProp[t]]}}{\sqrt{t}} + \sum_{t=T_0+1}^T\frac{\E[\AnyProp[t]]}{4\sqrt{t}}\right) +\sqrt{T_0}+ \underbrace{\frac{3}{4}\sqrt{K}+15+14K\log(T)}_{=:M}\,.
\end{align*}
%By inserting the bounds for {\it stability} and {\it penalty} into equation~\eqref{eq:regsplit} and using the self-bounding property \eqref{eq:mu-bound} in the second line in the derivation below we obtain
By using the self-bounding property \eqref{eq:mu-bound} and $(1+\lambda)\leq 2$ we obtain
\begin{align*}
& \overline{Reg}_T \leq \overline{Reg}_T +\lambda\lr{\overline{Reg}_T -\sum_{t=1}^T\sum_{\AnyArm\neq\BestArmSto}\AnyGap\E[\AnyProp[t]] + C}\\
&\quad  \leq \sum_{\AnyArm\neq\BestArmSto}\left(\sum_{t=1}^T \frac{2\sqrt{\E[\AnyProp[t]]}}{\sqrt{t}}+\sum_{t=T_0+1}^T\frac{\E[\AnyProp[t]]}{2\sqrt{t}}\right)  +2\sqrt{T_0}+ 2 M -\lambda\sum_{t=1}^T\sum_{\AnyArm\neq\BestArmSto}\AnyGap\E[\AnyProp[t]] + \lambda C\\
&\quad = \sum_{\AnyArm\neq\BestArmSto} \lr{\sum_{t=1}^{T_0} \lr{\frac{2\sqrt{\E[\AnyProp[t]]}}{\sqrt{t}} - \lambda\AnyGap\E[\AnyProp[t]]} + \sum_{t=T_0+1}^T \lr{\frac{2\sqrt{\E[\AnyProp[t]]}+ \frac{1}{2}\E[\AnyProp[t]]}{\sqrt{t}} - \lambda\AnyGap\E[\AnyProp[t]]}} \\
&\hspace{10cm}+2\sqrt{T_0}+ 2 M + \lambda C\\
&\quad \leq \sum_{\AnyArm\neq\BestArmSto}\lr{\sum_{t=1}^{T_0} \max_{z\geq 0} \lr{\frac{2\sqrt{z}}{\sqrt{t}} - \lambda\AnyGap z} + \sum_{t=T_0+1}^T \max_{z\geq 0}\lr{\frac{2\sqrt{z}+ \frac{1}{2}z}{\sqrt{t}} - \lambda\AnyGap z}} \\
&\hspace{10cm}
+2\sqrt{T_0}+ 2 M + \lambda C.
\end{align*}
Simple optimization shows that $\max_{z>0} 2\gamma\sqrt{z}-\beta z = \frac{\gamma^2}{\beta}$.
Thus, we have
\begin{align*}
\max_{z\geq 0}\frac{2\sqrt{z}}{\sqrt{t}}-\lambda\AnyGap z = \frac{1}{\lambda\Delta_i t}
\end{align*}
and
\begin{align*}
\max_{z\geq 0}\frac{2\sqrt{z}+\frac{1}{2}z}{\sqrt{t}}-\lambda\AnyGap z &= \frac{1}{(\lambda\AnyGap -\frac{1}{2\sqrt{t}})t}\\
& = \frac{1}{\lambda\Delta_i t} + \frac{1}{(\lambda\AnyGap -\frac{1}{2\sqrt{t}})t} - \frac{1}{\lambda\Delta_i t}\\
& = \frac{1}{\lambda\Delta_i t} + \frac{1}{2\lambda^2\AnyGap^2 t^\frac{3}{2}-\lambda\AnyGap t} \,.
\end{align*}
In order to bound the summation of the above terms, we use the following bound from Lemma~\ref{lem:t-sum} in the appendix:
\begin{align*}
\sum_{t=T_0+1}^T \frac{1}{bt^\frac{3}{2}-ct} \leq \frac{2}{b\sqrt{T_0}-c}\,.
\end{align*}
By definition of $T_0$ we have $\frac{1}{\lambda\Deltamin} \leq \sqrt{T_0} \leq \frac{1}{\lambda\Deltamin} + 1$ and 
\[
\frac{1}{\lambda^2\Delta_i^2 \sqrt{T_0} - \frac{1}{2}\lambda\Delta_i} 
= \frac{2}{\lambda\Delta_i(2\lambda\Delta_i \sqrt{T_0} - 1)}
\leq \frac{2}{\lambda\Delta_i\lr{2 \frac{\Delta_i}{\Deltamin} - 1}}
\leq \frac{2}{\lambda\Delta_i}.
\]
By plugging the calculations into the regret bound above we obtain:
\begin{align*}
\overline{Reg}_T &\leq \sum_{\AnyArm\neq\BestArmSto}\lr{\sum_{t=1}^T\frac{1}{\lambda\AnyGap t} +\sum_{t=T_0+1}^T\frac{1}{2\lambda^2\AnyGap^2 t^{\frac{3}{2}}-\lambda\AnyGap t}} +2\sqrt{T_0}+ 2M+\lambda C\\
&\leq \sum_{\AnyArm\neq\BestArmSto}\lr{\lr{\sum_{t=1}^T\frac{1}{\lambda\AnyGap t}} +\frac{1}{\lambda^2\Delta_i^2 \sqrt{T_0} - \frac{1}{2} \lambda\Delta_i}} +2\sqrt{T_0} + 2M + \lambda C\\
&\leq \sum_{i\neq i^*}\frac{\log(T)+3}{\lambda\Delta_i}+\frac{2}{\lambda\Deltamin} +2(M+1)+ \lambda C.
\end{align*}
Finally, choosing $\lambda = \min\left\{1, \sqrt{\lr{\sum_{i\neq i^*}\frac{\log(T)+3}{\Delta_i}+\frac{2}{\Deltamin}}\Big/C}\right\}$ completes the proof.

\end{proof}

\begin{proof}{\bf of Theorem~\ref{th:adv}}
We start from equation~\eqref{eq:regsplit}.
Since the regularization is symmetric, we have $\xi_i = 1$ for all $i$.
Using Lemma~\ref{lem:stability}, we bound the {\it stability} term as
\begin{align*}
stability &= \E\left[\sum_{t=1}^T\MyLoss[t]+ \Phi_t(- \CumLoss[t]) - \Phi_t(-\CumLoss[t-1]) \right]\leq \sum_{t=1}^T \sum_{\AnyArm=1}^K\frac{\LearningRate[t]}{2}\E\left[\AnyProp[t]\right]^{1-\alpha}\\
&\leq \left(\sum_{t=1}^T\frac{\LearningRate[t]}{2}\right)\max_{z\in\Simplex}\sum_{i=1}^K z_i^{1-\alpha} =  \left(\sum_{t=1}^T\sqrt{\frac{K^{1-2\alpha}-K^{-\alpha}}{1-\alpha}\frac{1-t^{-\alpha}}{\alpha t}}\right)\frac{K^{\alpha}}{2}\\
&\leq  \left(\sum_{t=1}^T\sqrt{\frac{1-K^{\alpha-1}}{1-\alpha}\frac{1-T^{-\alpha}}{\alpha t}}\right)\frac{\sqrt{K}}{2} \leq \sqrt{\frac{1-K^{\alpha-1}}{1-\alpha}\frac{1-T^{-\alpha}}{\alpha}KT}.
\end{align*}

The {\it penalty} is bounded according to Lemma~\ref{lem:penalty}
\begin{align*}
\end{align*}
\begin{align*}
penalty & = \E\left[\sum_{t=1}^T - \Phi_t(- \CumLoss[t]) + \Phi_t(-\CumLoss[t-1]) -\OptLoss[t] \right]\\
&\leq \frac{(K^{1-\alpha}-1)(1-T^{-\alpha})}{(1-\alpha)\alpha\LearningRate[T]} + 1 = \sqrt{\frac{1-K^{\alpha-1}}{1-\alpha}\frac{1-T^{-\alpha}}{\alpha}KT} +1.
\end{align*}

The proof is completed by noting that
the first factor is bounded by $\sqrt{\frac{1}{1-\alpha}}$ and monotonically increasing in $\alpha$ with the limit $\lim_{\alpha\rightarrow 1}\sqrt{\frac{1-K^{\alpha-1}}{1-\alpha}}=\sqrt{\log(K)}$ (details in Lemma~\ref{lem:loglimit} in the appendix).
By the same argument, the second factor is  bounded by $\sqrt{\frac{1}{\alpha}}$ and monotonically decreasing in $\alpha$ with the limit
$\lim_{\alpha\rightarrow 0}\sqrt{\frac{1-T^{-\alpha}}{\alpha}}= \sqrt{\log(T)}$.

\end{proof}

%\ys{I am not sure whether we need this comment here.} 
%In the proof of the final theorem, Theorem~\ref{th:all}, provided in the appendix we do not explicitly use a maximization $\max_w \frac{w^{\alpha}}{t^{1-\alpha}}+\frac{w^{1-\alpha}}{t^\alpha} - \Delta w$.
%Instead, we use the concavity of the functions $f(w)=w^{\alpha}$ and $f(w)= w^{1-\alpha}$ to bound these terms by the first order Taylor approximation $f(w) \leq f(w^*)+f'(w^*)(w-w*)$.
%Implicitly this is equivalent to taking a maximum with respect to some Lagrange multiplier $\lambda$.
%We do that to simplify the analysis.

\section{Experiments}
\label{sec:experiments}
We provide an empirical comparison of \Alg{Tsallis-Inf} \rev{with IW and with RV loss estimators} with the classical algorithms for stochastic bandits, \Alg{Ucb1} \citep[with parameter $\alpha=1.5$]{auer2002finite} and \Alg{Thompson Sampling} \citep{thompson1933likelihood}\footnote{Another leading stochastic algorithm, \Alg{KL-UCB} \citep{CGM+13}, has performed comparably to \Alg{Thompson Sampling} in our experiments and, therefore, is not reported in the figures.}, and the classical algorithm for adversarial bandits, \Alg{Exp3}, implemented for the losses \citep{bubeck2012regret}.
We also compare with the state-of-the-art algorithms for stochastic and adversarial bandits, \Alg{EXP3++} with parametrization proposed by \citet{seldin2017improved} and \Alg{Broad} \citep{wei2018more}.
The pseudo-regret is estimated by 100 repetitions of the corresponding experiments and two standard deviations of the \rev{empirical pseudo-regret, $\sum_{t=1}^T\Delta_{I_t}$, over the 100 repetitions} are depicted by the shaded areas on the plots. 
We always show the first $10000$ time steps on a linear plot and then the time steps from $10^4$ to $10^7$ on a separate log-log plot.

%The first two experiments illustrate the superiority of \Alg{Tsallis-Inf}. 
%In the first two experiments we use the same number of arms $K=8$, a single best arm, and the same gap $\Delta=0.125$ for all suboptimal arms. 
The first experiment, shown in Figures~\ref{fig:both} and \ref{fig:both-log}, is a standard stochastic MAB, where the mean rewards are $(1+\Delta)/2$ for the single optimal arm and $(1-\Delta)/2$ for all the suboptimal arms. The number of arms $K$ and the gaps $\Delta$ are varied as described in the figures.
Unsurprisingly, \Alg{Thompson Sampling} exhibits the lowest regret, but \Alg{Tsallis-Inf} \rev{with RV estimators follows closely behind and outperforms all other competitors by a large margin. \Alg{Tsallis-INF} with IW estimators} takes a confident \rev{third} place, while \Alg{UCB1}, \Alg{EXP3}, and \Alg{EXP3++} fall roughly in the same league. \Alg{Broad} suffers from extremely large constant factors and is out of question for practical applications.

\begin{figure}
    \centering
    \fontsize{0.1pt}{0.12pt}
    \def\svgwidth{\columnwidth}
    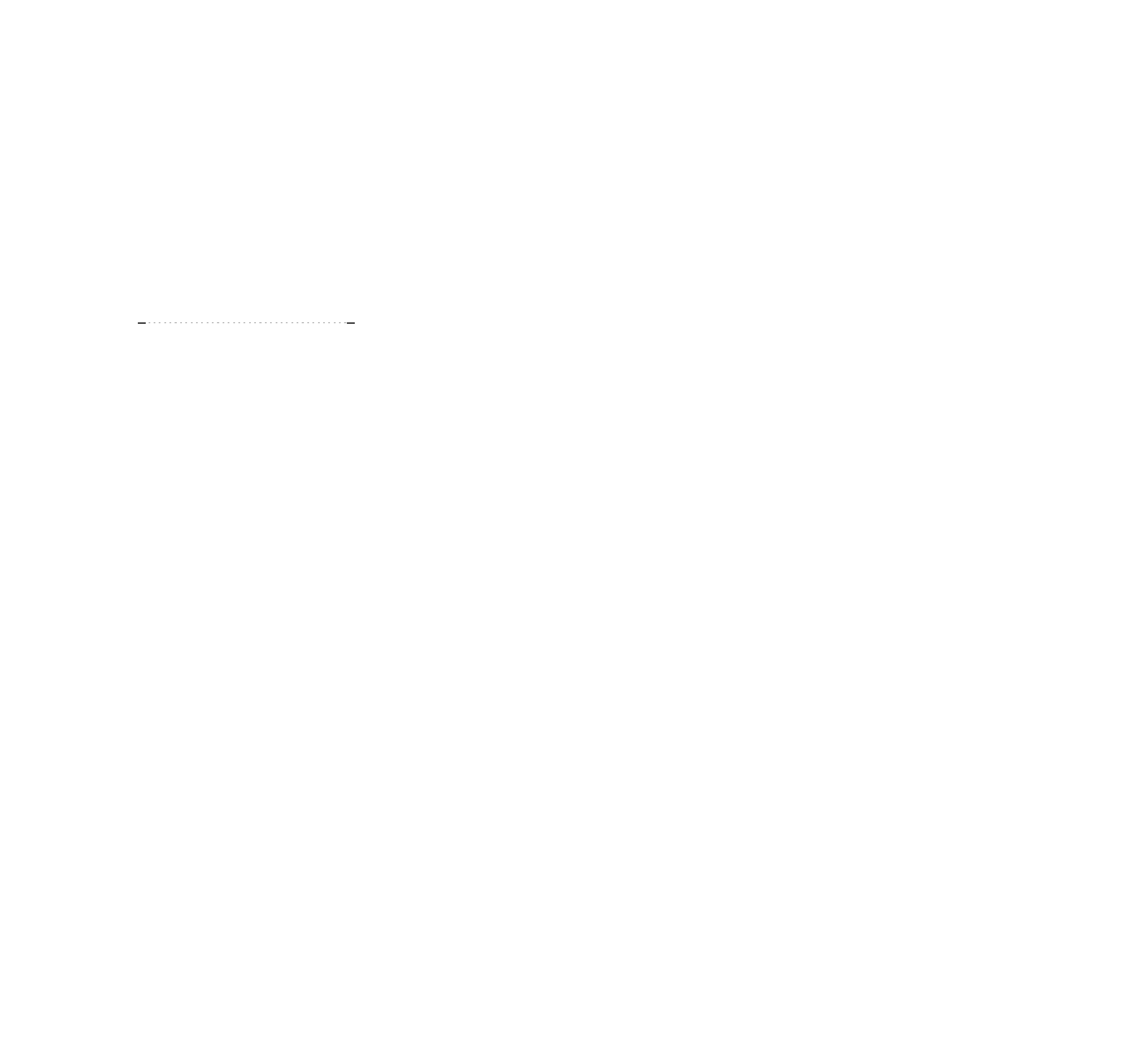
    \caption{Comparison of \Alg{Tsallis-Inf} \rev{with IW and with RV estimators} with \Alg{Thompson Sampling}, \Alg{Ucb1}, \Alg{Exp3}, \Alg{Exp3++}, and \Alg{Broad} in a stochastic environment with fixed mean losses of $\frac{1-\Delta}{2}$ for the optimal arm and $\frac{1+\Delta}{2}$ for all sub-optimal arms. The experiment is repeated for different number of arms $K$ and different gaps $\Delta$. The figure shows the first $10000$ time steps on a linear plot. The pseudo-regret is estimated by 100 repetitions and we depict 2 standard deviations of the empirical pseudo-regret by the shaded areas.}
    \label{fig:both}
\end{figure}

\begin{figure}
    \centering
    \fontsize{0.1pt}{0.12pt}
    \def\svgwidth{\columnwidth}
    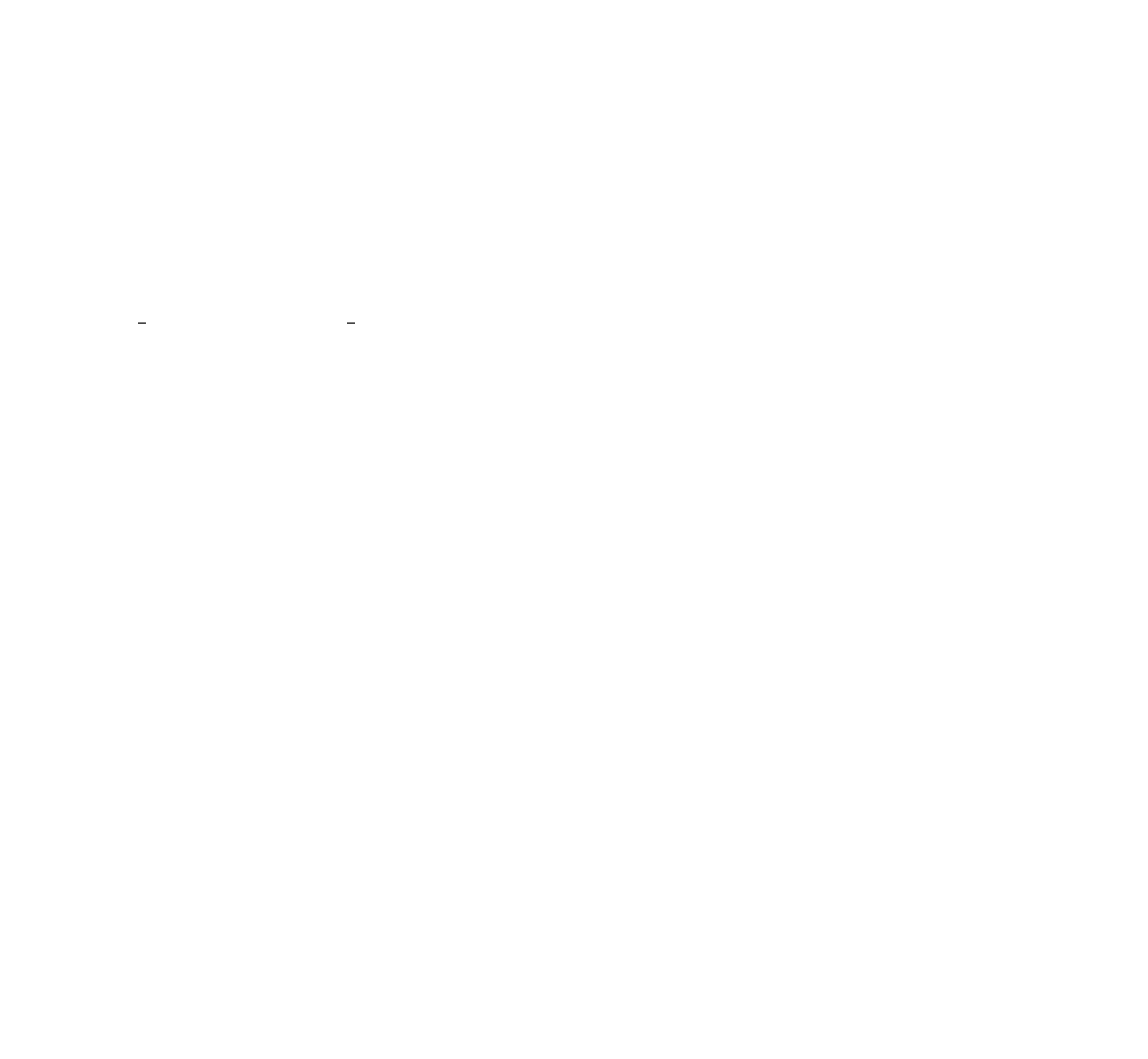
    \caption{Comparison of \Alg{Tsallis-Inf} \rev{with IW and with RV estimators} with \Alg{Thompson Sampling}, \Alg{Ucb1}, \Alg{Exp3}, \Alg{Exp3++}, and \Alg{Broad} in a stochastic environment with fixed mean losses of $\frac{1-\Delta}{2}$ for the optimal arm and $\frac{1+\Delta}{2}$ for all sub-optimal arms. The experiment is repeated for different number of arms $K$ and different gaps $\Delta$. The figure shows the time steps from $10^4$ to $10^7$ on a log-log plot. The pseudo-regret is estimated by 100  repetitions and we depict 2 standard deviations of the empirical pseudo-regret by the shaded areas.}
    \label{fig:both-log}
\end{figure}

\begin{figure}
    \centering
    \fontsize{0.1pt}{0.12pt}
    \def\svgwidth{\columnwidth}
    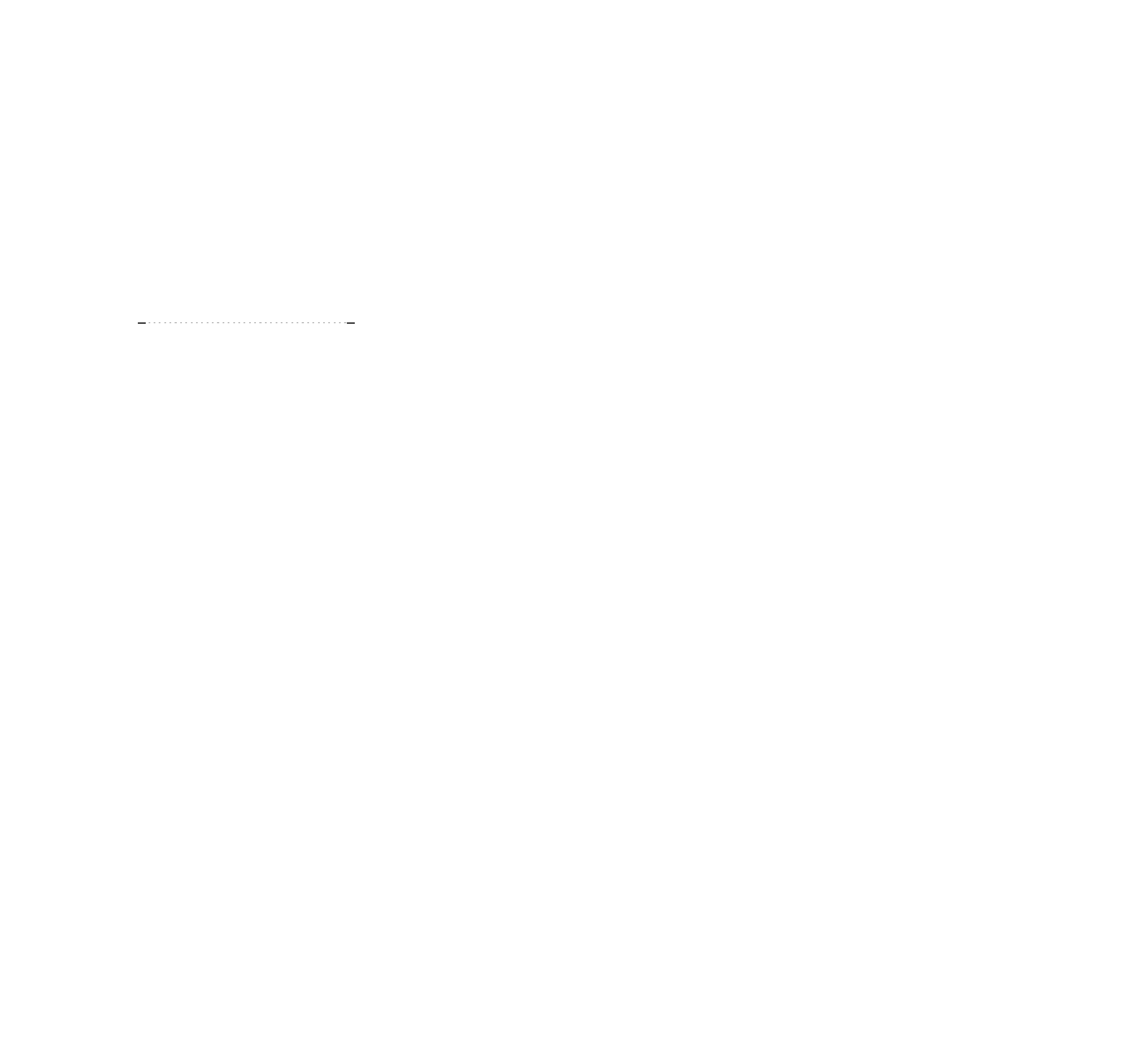
    \caption{Comparison of \Alg{Tsallis-Inf} \rev{with IW and with RV estimators} with \Alg{Thompson Sampling}, \Alg{Ucb1}, \Alg{Exp3}, \Alg{Exp3++}, and \Alg{Broad} in a stochastically constrained adversarial environment. The environment (unknown to the agent) alternates between two stochastic settings. In the first setting the expected loss of the optimal arm is 0 and $\Delta$ for sub-optimal arms. In the second the expected losses are $1-\Delta$ and $1$, respectively. The time between alternations increases exponentially (with factor $1.6$) after each switch. The experiment is repeated for different number of arms $K$ and different gaps $\Delta$. The figure shows the first $10000$ time steps on a linear plot. The pseudo-regret is estimated by 100 repetitions and we depict 2 standard deviations of the empirical pseudo-regret by the shaded areas.}
    \label{fig:both-adv}
\end{figure}

\begin{figure}
    \centering
    \fontsize{0.1pt}{0.12pt}
    \def\svgwidth{\columnwidth}
    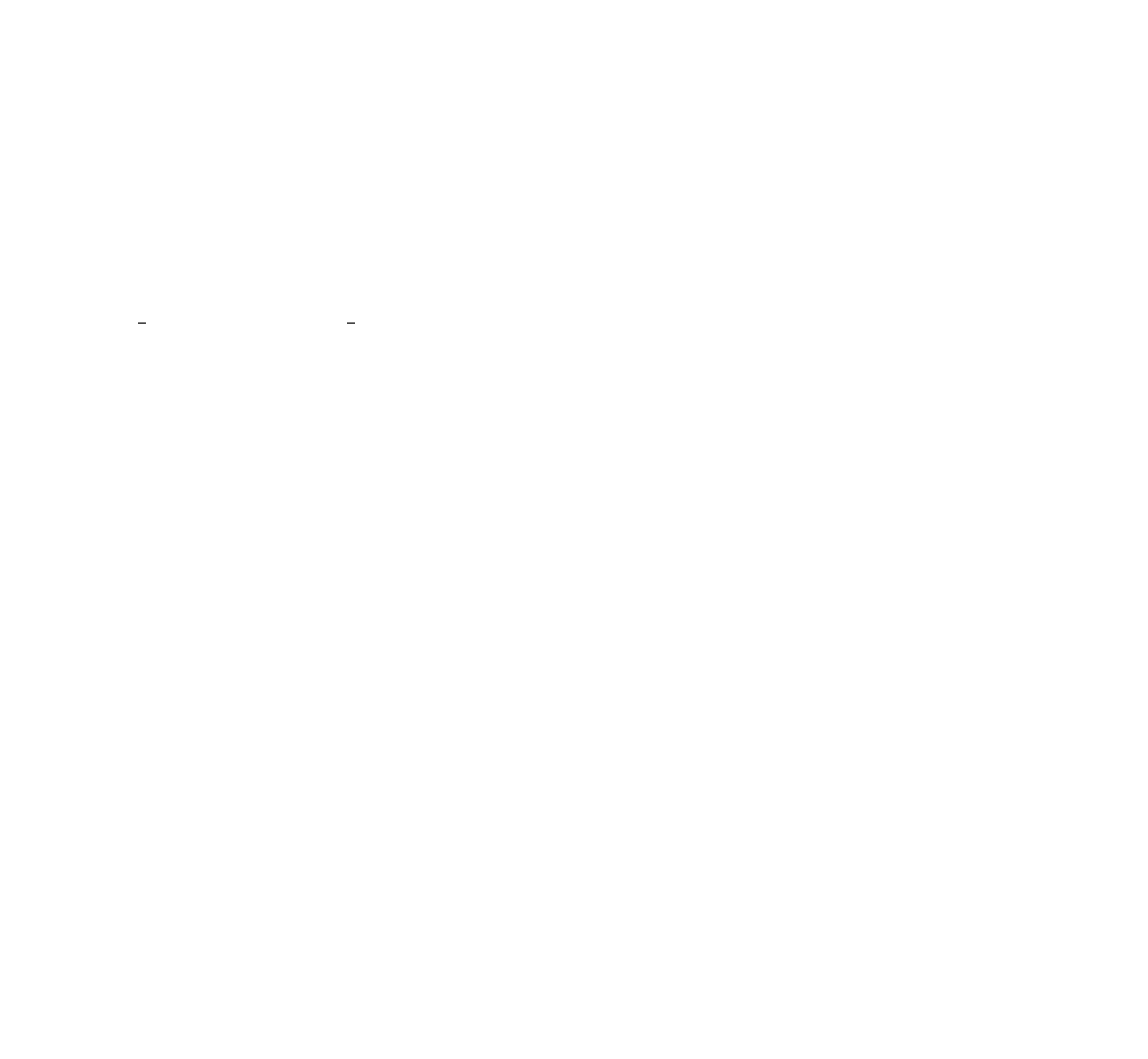
    \caption{Comparison of \Alg{Tsallis-Inf} \rev{with IW and with RV estimators} with \Alg{Thompson Sampling}, \Alg{Ucb1}, \Alg{Exp3}, \Alg{Exp3++}, and \Alg{Broad} in a stochastically constrained adversarial environment. The environment (unknown to the agent) alternates between two stochastic settings. In the first setting the expected loss of the optimal arm is 0 and $\Delta$ for sub-optimal arms. In the second the expected losses are $1-\Delta$ and $1$ respectively. The time between alternations increases exponentially (with factor $1.6$) after each switch. The experiment is repeated for different number of arms $K$ and different gaps $\Delta$. The figure shows the time steps from $10^4$ to $10^7$ on a log-log plot. The pseudo-regret is estimated by 100 repetitions and we depict 2 standard deviations of the empirical pseudo-regret by the shaded areas.}
    \label{fig:both-adv-log}
\end{figure}

The second experiment, shown in Figures~\ref{fig:both-adv} and \ref{fig:both-adv-log}, simulates stochastically constrained adversaries.
The mean loss of (optimal arm, all sub-optimal arms) switches between $(1-\Delta,1)$ and $(0,\Delta)$, while staying unchanged for phases that are increasing exponentially in length.
Both \Alg{Ucb1} and \Alg{Thompson-Sampling} suffer almost linear regret.
To the best of our knowledge, this is the first empirical evidence clearly demonstrating that \Alg{Thompson Sampling} is unsuitable for adversarial regimes.
All other algorithms are almost unaffected by the shifting of the means.

\rev{Both experiments confirm that \Alg{Tsallis-INF} with IW estimators achieves logarithmic regret in stochastic and stochastically constrained adversarial environments and that RV estimators significantly improve the constants.}

\subsection{Multiple Optimal Arms}
Since our theoretical results for the stochastic setting do not include multiple optimal arms, we explore this setting empirically. 
We use a single suboptimal arm with a mean loss of $9/16$.
All other arms are optimal with a mean loss of $7/16$.
We run the experiment with 1000 repetitions and increase the number of arms.
Figure~\ref{fig:copies} clearly shows that the regret does not suffer if the optimal arm is not unique.
\rev{On the opposite}, we observe that the regret decreases with the growth of the number of suboptimal arms. 
Therefore, we conjecture that the requirement of uniqueness is merely an artifact of the analysis.
\begin{figure}
    \centering
    \fontsize{10pt}{12pt}
    \def\svgwidth{\columnwidth}
    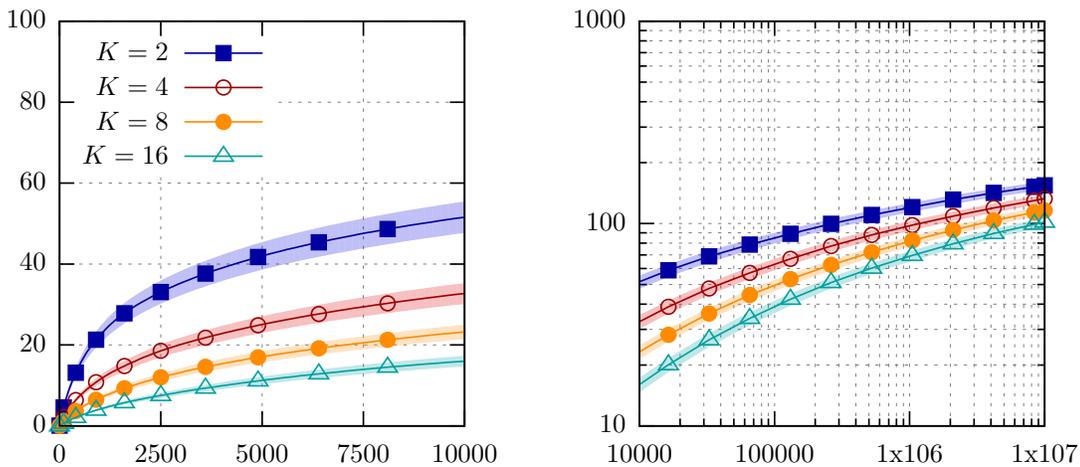
    \caption{Increasing number of copies of the best arm}
    \label{fig:copies}
\end{figure}

\section{Discussion}
\label{sec:dis}
We have presented a \rev{general analysis} of online mirror descent algorithms regularized by Tsallis entropy with $\alpha\in[0,1]$.
As the main contribution, we have shown that the special case of $\alpha=\frac{1}{2}$ achieves optimality in both adversarial and stochastic regimes, while being oblivious to the environment at hand.
Thereby, we have closed logarithmic gaps to lower bounds, which were present in existing best-of-both-worlds algorithms.
We introduced a novel proof technique based on the self-bounding property of the regret, circumventing the need of controlling the variance of loss estimates. 
We have provided an empirical evidence that our algorithm \rev{outperforms} \Alg{UCB1} in stochastic environments and is significantly more robust than \Alg{UCB1} and \Alg{Thompson Sampling} in non-i.i.d.\ settings.
\rev{We have introduced an adversarial regime with a self-bounding constraint, which includes} stochastically constrained adversaries and adversarially corrupted stochastic bandits \rev{as special cases and improved regret bounds for the latter two regimes}. We have also shown that \Alg{Tsallis-Inf} can be applied to achieve stochastic and adversarial optimality in utility-based dueling bandits.

A weak point of the current \rev{analysis} is the \rev{assumption} on uniqueness of \rev{the zero entry in a vector of suboptimality gaps in the adversarial regime with a self-bounding constraint. In stochastic and stochastically constrained adversarial settings, it corresponds to assumption of uniqueness of }the best arm.
Our experiments suggest that this is most likely an artifact of the analysis and we aim to address this shortcoming in future work.

\rev{Another open question is whether it is possible to close the remaining factor $2$ gap between the upper and lower gap-dependent asymptotic regret bounds, either by improving the upper bound in the stochastic regime or deriving a tighter lower bound for the adversarial regime with a self-bounding constraint.}

One more open question is whether logarithmic regret is achievable by \Alg{Tsallis-Inf} in the intermediate regimes defined by \citet{seldin2014one}. We have \rev{discussed} this question in more detail in Section~\ref{sec:open}.

An additional direction for future research is the application of \Alg{Tsallis-Inf} to other problems. 
The fact that the algorithm relies solely on importance weighted losses makes it a suitable candidate for partial monitoring games. 
One step in this direction has already been taken by \citet{ZLW19}. 

\acks{We would like to thank Chloé Rouyer for pointing out several bugs in the previous version of the work \citep{ZS19} and Haipeng Luo for the idea on how to improve our regret bounds for stochastic bandits with adversarial corruptions in the large $C$ case. We are also grateful to the anonymous reviewers for their comments. We acknowledge partial support by the Independent Research Fund Denmark, grant number 9040-00361B.}
%\newpage

\appendix

\section{Asymptotic Lower Bound}
\label{app:asym}
If the optimal arm has mean reward $\frac{1}{2}$ and suboptimal arms have the gaps $\AnyGap$ then the following lower bound for any consistent algorithm follows from \citet[Theorem 2]{lai1985asymptotically}
\begin{align*}
\lim_{t\rightarrow \infty}\frac{\overline{Reg}_t}{\log(t)} \geq \sum_{\AnyArm:\AnyGap>0}\frac{\AnyGap}{\operatorname{kl}(\frac{1}{2}+\AnyGap,\frac{1}{2})}.
\end{align*}
For any $\AnyGap\in[0,0.5]$ the $\operatorname{kl}$ term can be upper bounded as
\begin{align*}
\operatorname{kl}(\frac{1}{2}+\AnyGap,\frac{1}{2})\leq 2\AnyGap^2+3\AnyGap^3,
\end{align*}
which can be verified by taking Taylor's expansion at $\AnyGap=0$.
Therefore, 
\begin{align*}
\sum_{\AnyArm:\AnyGap>0}\frac{\AnyGap}{\operatorname{kl}(\frac{1}{2}+\AnyGap,\frac{1}{2})}
&\geq \sum_{\AnyArm:\AnyGap>0}\frac{1}{2\AnyGap+3\AnyGap^2}\\
&= \frac{1}{2}\sum_{\AnyArm:\AnyGap>0}\frac{1}{\AnyGap}-\sum_{\AnyArm:\AnyGap>0}\frac{\frac{3}{2}\AnyGap}{2\AnyGap+3\AnyGap^2}\geq \frac{1}{2}\lr{\sum_{\AnyArm:\AnyGap>0}\frac{1}{\AnyGap}-\frac{3}{2}K}.
\end{align*}
\rev{Thus, for any consistent algorithm we obtain
\begin{align*}
    \lim_{||\Delta|| \rightarrow 0} \left(\left(\sum_{\AnyArm:\AnyGap>0}\frac{1}{\AnyGap}\right)^{-1}\liminf_{t\rightarrow\infty}\frac{\E\left[\overline{Reg}_t\right]}{\log(t)}\right) &\geq \lim_{||\Delta|| \rightarrow 0} \left(\frac{1}{2}-\frac{3}{4}\left(\sum_{\AnyArm:\AnyGap>0}\frac{1}{\AnyGap}\right)^{-1}K\right)=\frac{1}{2}\,,
\end{align*}
since $\lim_{||\Delta||\rightarrow 0} \left(\sum_{\AnyArm:\AnyGap>0}\AnyGap^{-1}\right)^{-1}K = 0$.}

\section{Technical Lemmas}

\begin{lemma}
\label{lem:sum}
\[
\sum_{t=2}^\infty \lr{\sqrt{t} - \sqrt{t-1} - \frac{1}{2\sqrt{t}}} \leq \frac{1}{4}.
\]
\end{lemma}

\begin{proof}
We have $\sqrt{t} - \sqrt{t-1} = \frac{(\sqrt{t} - \sqrt{t-1})(\sqrt{t}+\sqrt{t-1})}{\sqrt{t}+\sqrt{t-1}} = \frac{1}{\sqrt{t}+\sqrt{t-1}}$
and $\frac{1}{\sqrt{t} + \sqrt{t-1}} - \frac{1}{2\sqrt{t}} = \frac{1}{2\sqrt{t}(\sqrt{t}+\sqrt{t-1})^2} \leq \frac{1}{8(t-1)^{3/2}}$. Summation of the latter is related to the Riemann zeta function $\zeta\lr{\frac{3}{2}} = \sum_{t=1}^\infty \frac{1}{t^{3/2}} \leq 2.62$. To get a slightly tighter bound, we count the first two terms explicitly and bound the rest using Riemann zeta function:
\begin{align*}
\sum_{t=2}^\infty \lr{\sqrt{t} - \sqrt{t-1} - \frac{1}{2\sqrt{t}}} &= \sqrt{2} - 1 - \frac{1}{2\sqrt{2}} + \sqrt{3} - \sqrt{2} - \frac{1}{2\sqrt{3}} + \sum_{t=4}^\infty \lr{\sqrt{t} - \sqrt{t-1} - \frac{1}{2\sqrt{t}}}\\
&\leq \sqrt{3}-1-\frac{1}{2\sqrt{2}} - \frac{1}{2\sqrt{3}} + \frac{1}{8}\sum_{t=3}^\infty \frac{1}{t^{3/2}}\\
&= \sqrt{3}-1-\frac{1}{2\sqrt{2}} - \frac{1}{2\sqrt{3}} + \frac{1}{8}\lr{\zeta\lr{\frac{3}{2}} - 1 - \frac{1}{2\sqrt{2}}}\\
&\leq \frac{1}{4}.
\end{align*}
\end{proof}

\begin{lemma}
\label{lem:loglimit}
For any $y > 0$ and $x>0$ the function $\frac{1-y^{-x}}{x}$ is non-increasing in $x$ and has the limit
\begin{align*}
\lim_{x\rightarrow 0} \frac{1-y^{-x}}{x} = \log(y),
\end{align*}
therefore, $\frac{1-y^{-x}}{x} \leq \min\{x^{-1},\log(y)\}$.
\end{lemma}
\begin{proof}
Taking the derivative and using the inequality $z \leq e^z-1$:
\begin{align*}
\frac{\partial}{\partial x}\left(\frac{1-y^{-x}}{x}\right) &= \frac{\log(y)y^{-x}x-(1-y^{-x})}{x^2}\leq \frac{(e^{\log(y)x}-1)y^{-x}-(1-y^{-x})}{x^2} = 0.
\end{align*}
The limit by L'H\^{o}pital's rule is
\begin{align*}
\lim_{x\rightarrow 0} \frac{1-y^{-x}}{x} = \lim_{x\rightarrow 0} \frac{\log(y)y^{-x}}{1} = \log(y).
\end{align*}
\end{proof}

\begin{lemma}
\label{lem:t-sum}
For any $b > 0$ and $c >0$ and $T_0,T \in \mathbb{N}$, such that $T_0 < T$ and $b\sqrt{T_0}>c$, it holds that
\begin{align*}
\sum_{t=T_0+1}^T\frac{1}{bt^\frac{3}{2}-ct} \leq \frac{2}{b\sqrt{T_0}-c}\,.
\end{align*}
\end{lemma}
\begin{proof}
In the domain $(c^2/b^2,\infty)$, the function $f(t) = \frac{1}{bt^\frac{3}{2}-ct}$ is positive, monotonically decreasing and has the antiderivative
\begin{align*}
F(t) = \frac{1}{c}\lr{2\log(b\sqrt{t}-c)-\log(t)}\,,
\end{align*}
which can be verified by taking the respective derivatives.
Therefore, we can bound
\begin{align*}
\sum_{t=T_0+1}^T f(t) &\leq \int_{T_0}^\infty f(t)\,dt 
= \lim_{t\rightarrow\infty}F(t) - F(T_0)
= \frac{1}{c}\lr{2\log(b)-2\log(b\sqrt{T_0}-c)+\log(T_0)} \\
&= \frac{2}{c}\log\lr{\frac{b\sqrt{T_0}}{b\sqrt{T_0}-c}}
\leq \frac{2}{c}\lr{\frac{b\sqrt{T_0}}{b\sqrt{T_0}-c}-1}
=\frac{2}{b\sqrt{T_0}-c}\,.
\end{align*}
\end{proof}

\begin{lemma}
\label{lem:ugly-ineq}
For any $\alpha\in[0,1]$ and $z\geq 1$, it holds that
\begin{align*}
\frac{1-z^{-1+\alpha}}{1-\alpha}\leq (\log z)^\alpha\,,
\end{align*}
\rev{where for $\alpha=1$ we consider it as the limit case $\lim_{\alpha\rightarrow 1}\frac{1-z^{-1+\alpha}}{1-\alpha}\leq \log(z)$ and for $z=1$ and $\alpha=0$ we use the convention $0^0=1$.}
\end{lemma}
\begin{proof}
\rev{For $\alpha = 0$ we have
\[
\frac{1-z^{-1+0}}{1-0}=1-z^{-1} < (\log z)^0 = 1,
\]
where for $z=1$ we use the convention $0^0=1$.
For $\alpha=1$ we consider the limit case
\begin{align*}
\lim_{\alpha\rightarrow 1}\frac{1-z^{-1+\alpha}}{1-\alpha}=\log(z)=(\log z)^1\,.
\end{align*}
It is left} to verify the statement for $\alpha\in(0,1)$.
Consider the function
\begin{align*}
f(z) = (\log z)^\alpha - \frac{1-z^{-1+\alpha}}{1-\alpha}\,.
\end{align*}
The function is continuous for $z \geq 1$, takes the value $0$ at $z=1$ and goes to infinity for $z\rightarrow \infty$.
If there is a point where the function is negative, there must also exist an extreme point.
Setting the derivative to $0$ shows that all extreme points $z^*$ satisfy
\begin{align*}
z^* = \log(z^*)\alpha^\frac{1}{\alpha-1}\,.
\end{align*}
The function values at the extreme points are therefore lower bounded by
\begin{align*}
f(z^*) \geq \min_{z\geq 1} \lr{(\log z)^\alpha -\frac{1-(\log(z)\alpha^\frac{1}{\alpha-1})^{-1+\alpha}}{1-\alpha}}=\min_{\tilde z\geq 0}\lr{\tilde z^\alpha -\frac{1-\tilde z^{-1+\alpha}\alpha}{1-\alpha}}\,,
\end{align*}
where we apply the substitution $\tilde z = \log(z)$. The RHS goes to infinity for $\tilde z\rightarrow 0$ and $\tilde z\rightarrow \infty$, which means that the only extreme point (can be verified by taking the derivative) at $\tilde z=1$ is the minimum.
Since $1^\alpha-\frac{1-1^{-1+\alpha}\alpha}{1-\alpha}=0$, the function $f$ is always positive, which concludes the proof.
\end{proof}

\section{Support Lemmas for Section~\ref{sec:proofs}}
\label{App:lemma-proofs}

\rev{We use $v =  (v_i)_{i=1,\dots,K}$ to denote a column vector $v\in \R^K$ with elements $v_1,\dots,v_K$. We use $\operatorname{diag}(v)$ to denote a $K\times K$ matrix with $v_1,\dots,v_K$ on the diagonal and 0 elsewhere. For a positive semidefinite matrix $M$ we use $||\cdot||_M=\sqrt{\langle\cdot, M\cdot\rangle}$ to denote the canonical norm with respect to $M$.} We also use the following properties of the potential function. 
\begin{align}
&\Psi_t(\Prop) = -\sum_{\AnyArm}\frac{\AnyProp^\alpha-\alpha \AnyProp }{\alpha(1-\alpha)\LearningRate[t]\xi_i},\notag\\
&\nabla\Psi_t(\Prop) = -\left(\frac{\AnyProp^{\alpha-1}-1}{(1-\alpha)\LearningRate[t]\xi_i}\right)_{i=1,\dots,K},\notag\\
&\nabla^2\Psi_t(\Prop) = \operatorname{diag}\left(\left(\frac{\AnyProp^{\alpha-2}}{\LearningRate[t]\xi_i}\right)_{i=1,\dots,K}\right).\notag\\
&\mbox{For }Y\leq 0:\notag\\
&\Psi_t^*(Y) = \max_w \dprod{w,Y} + \frac{1}{\eta_t} \sum_i \frac{w_i^\alpha - \alpha w_i}{\alpha(1-\alpha) \xi_i},\notag\\
&\nabla\Psi_t^*(Y) = \argmax_w \dprod{w,Y} + \frac{1}{\eta_t} \sum_i \frac{w_i^\alpha - \alpha w_i}{\alpha(1-\alpha) \xi_i} = \lr{\lr{-\eta_t(1-\alpha)\xi_i Y_i+1}^{\frac{1}{\alpha-1}}}_{i=1,\dots,K}.\label{eq:nabla-Psi}
\end{align}
%\ys{The above function is \emph{decreasing} in $L$, not \emph{increasing}, as it is stated later.} 
As we have shown in Section~\ref{sec:intuition}, there exists a Lagrange multiplier $\nu$, such that the algorithm picks the probabilities
\begin{align}
&w_t = \nabla \Phi_t(-\hat L_{t-1}) = \nabla \Psi_t^*(-\hat L_{t-1}+\nu\mathbf{1}_K).\label{eq:lagr mult}
\end{align}
$\Psi_t$ is a Legendre function, which implies that its gradient is invertible and $\nabla\Psi_t^{-1} = \nabla\Psi_t^*$ \citep{rockafellar2015convex}. 
Furthermore, by the Inverse Function theorem,  
\begin{align}
\nabla^2\Psi^*_t(\nabla\Psi_t(\Prop))=\nabla^2\Psi_t(\Prop)^{-1} = \operatorname{diag}\left(\LearningRate[t]\xi_i\AnyProp^{2-\alpha}\right)_{i=1,\dots,K}.
\label{eq:Dual-Hessian}
\end{align}
The Bregman divergence associated with a Legendre function $f$ is defined by
$$D_f(x,y) = f(x)-f(y)-\left\langle \nabla f(y), x-y\right\rangle.$$
By Taylor's theorem, it satisfies for some $z \in \operatorname{conv}(x,y)$
\begin{align}
D_f(x,y) \leq \frac{1}{2}||x-y||^2_{\nabla^2f(z)}. \label{eq:taylor}
\end{align}

\subsection{Controlling the {\it stability} Term}
Equation~\eqref{eq:taylor} gives a way of bounding the {\it stability} term.
The following lemma allows to control the eigenvalues of the Hessian $\nabla^2\Psi^*_t$.

\begin{lemma}
\label{lem:wrange}
Let $w \in \Simplex$ and $\tilde w = \nabla\Psi_{t}^*(\nabla\Psi_t(w)-\ell)$. 
\rev{If $\LearningRate[t]\xi_i\leq \frac{1}{4}$ for all $i$, then for all $\ell\in\R^K$ with $\ell_i \geq -1$ for all $i$, it holds that  
$\tilde{w}_i^{2-\alpha} \leq 2w_i^{2-\alpha}$ for all $i$.}
\end{lemma}
\begin{proof}
Since $\nabla \Psi_t$ is the inverse of $\nabla\Psi_t^*$, we have
\begin{align*}
&\nabla\Psi_t(w)_i -\nabla\Psi_t(\tilde w)_i =  \ell_i \geq -1,\\
&\frac{w_i^{\alpha-1}-1}{(1-\alpha)\LearningRate[t]\xi_i } - \frac{\tilde{w}_\AnyArm^{\alpha-1}-1}{(1-\alpha)\LearningRate[t]\xi_i } \leq 1,\\
&\tilde{w}_i^{1-\alpha} \leq \frac{w_i^{1-\alpha}}{1-\LearningRate[t]\xi_i(1-\alpha)w_i^{1-\alpha}}\leq \frac{w_i^{1-\alpha}}{1-\LearningRate[t]\xi_i(1-\alpha)},\\
&\tilde{w}_i^{2-\alpha} \leq \frac{w_i^{2-\alpha}}{\lr{1-\LearningRate[t]\xi_i(1-\alpha)}^{\frac{2-\alpha}{1-\alpha}}}.
\end{align*}
It remains to bound $\lr{1-\LearningRate[t]\xi_i(1-\alpha)}^{-\frac{2-\alpha}{1-\alpha}}$.
Note that this function is monotonically decreasing in $\alpha$, which can be verified by confirming that the derivative is negative in $[0,1]$.
Using the fact that $\LearningRate[t]\xi_i \leq \frac{1}{4}$, we have
\begin{align*}
\lr{1-\LearningRate[t]\xi_i(1-\alpha)}^{-\frac{2-\alpha}{1-\alpha}} \leq \lr{1-\LearningRate[t]\xi_i}^{-2} \leq \frac{4^2}{3^2}\leq 2.
\end{align*}
\end{proof}

For the RV estimators and $\alpha=1/2$, we provide a tighter bound for the stability term by using the following two lemmas. 
For $\alpha = 1/2$ and $\xi_i = 1$ we have
\begin{align}
&\Psi_t(w) = -4\eta_t^{-1}\sum_{i=1}^K(w_i^\frac{1}{2}-\frac{1}{2}w_i)\,,\notag\\
&\nabla\Psi_t(w) = \lr{-2\eta_t^{-1}(w_i^{-\frac{1}{2}}-1)}_{i=1,\dots,K}\,.\label{eq:nabla-Psi-half}
\end{align}
\begin{lemma}
\label{lem:conv conjug}
The convex conjugate of $\Psi_t(w) = -4\eta_t^{-1}\sum_{i=1}^K(w_i^\frac{1}{2}-\frac{1}{2}w_i)$ is
\begin{align*}
&\Psi_t^*(Y) = \begin{cases} \sum_{i=1}^K \frac{2\eta_t^{-1}}{1-\eta_tY_i/2}, &\text{if $Y_i< 2\eta_t^{-1}$ for all $i$,}\\
\infty, & \text{otherwise.}
\end{cases}
\end{align*}
\end{lemma}
\begin{proof}
\begin{align*}
    \Psi_t^*(Y)&=\sup_{w\in\R^K} \langle w, Y\rangle - \Psi_t(w)\\
    &=\sum_{i=1}^K \sup_{w\in\R} w(Y_i-2\eta_t^{-1})+4\eta_t^{-1}w^\frac{1}{2}\,.
\end{align*}
For $Y_i\geq 2\eta^{-1}$, the term goes to infinity as $w\rightarrow\infty$.
Otherwise, the maximum is obtained by $w = \frac{1}{(1-\frac{1}{2}\eta_tY_i)^2}$, which concludes the proof.
\end{proof}
Using the explicit form of the convex conjugate, we can show a general bound on the stability.
\begin{lemma}
\label{lem:reduced variance stability}
Let $\alpha = 1/2$ and $\xi_i =1$. Then for any $x$, such that $\min_i \eta_t(\AnyImpLoss[t]-x)w_{t,i}^\frac{1}{2} \geq -1$, the instantaneous stability satisfies
\begin{align*}
    \langle \Prop[t],\ImpLoss[t] \rangle +\Phi_{t}(-\CumLoss[t])-\Phi_{t}(-\CumLoss[t-1]) \leq \sum_{i=1}^K \frac{\eta_t}{2}\AnyProp[t]^\frac{3}{2}(\AnyImpLoss[t]-x)^2 + \frac{\eta_t^2}{2} \AnyProp[t]^2\left|x-\AnyImpLoss[t]\right|_+^3\,,
\end{align*}
where $|z|_+ = \max\{z,0\}$.
\end{lemma}
\begin{proof}
By equation~\eqref{eq:lagr mult} and since $\nabla \Psi_t^{-1} = \nabla \Psi^*_t$ and $w_t = \nabla\Phi_t(-\hat L_{t-1})=\nabla(\Psi_t+\mathcal{I}_{\Simplex})^*(-\hat L_{t-1})$, there exists a Lagrange multiplier $\nu$ such that
\begin{align*}
    -\CumLoss[t-1] = \nabla\Psi_t(w_t) - \nu\mathbf{1}_K\,.
\end{align*}
Furthermore, $\Phi_t(-L - \nu\mathbf{1}_K) = \Phi_t(-L) -\nu$, since the maximization over $w$ is restricted to the probability simplex.
Using these two properties, we have for any $x\in\mathbb{R}$
\begin{align*}
    \langle \Prop[t],\ImpLoss[t] \rangle +\Phi_{t}(-\CumLoss[t])-\Phi_{t}(-&\CumLoss[t-1]) \\
    &= \langle \Prop[t],\ImpLoss[t] \rangle +\Phi_{t}(-\ImpLoss[t]+\nabla\Psi_t(\Prop[t])-\nu\mathbf{1}_K)-\Phi_{t}(\nabla\Psi_t(\Prop[t])-\nu\mathbf{1}_K)\\
    &=\langle \Prop[t],\ImpLoss[t] \rangle +\Phi_{t}(-\ImpLoss[t]+\nabla\Psi_t(\Prop[t]))-\Phi_{t}(\nabla\Psi_t(\Prop[t]))\\
    &=\langle \Prop[t],\ImpLoss[t]-x\mathbf{1}_K \rangle +\Phi_{t}(x\mathbf{1}_K-\ImpLoss[t]+\nabla\Psi_t(\Prop[t]))-\Phi_{t}(\nabla\Psi_t(\Prop[t]))\\
    &\leq \langle \Prop[t],\ImpLoss[t]-x\mathbf{1}_K \rangle +\Psi_{t}^*(x\mathbf{1}_K-\ImpLoss[t]+\nabla\Psi_t(\Prop[t]))-\Psi^*_{t}(\nabla\Psi_t(\Prop[t]))\,,
\end{align*}
where the last line uses $\Phi_t(\nabla\Psi_t(w))=\Psi^*_t(\nabla\Psi_t(w))$ for any $w\in\Delta^{K-1}$, which holds because the argmax in both terms is $w$, and the inequality $\Phi_t(L) \leq \Psi_t^*(L)$, which holds because $\Phi_t$ is a constrained version of $\Psi_t^*$.

Using the explicit expression for the convex conjugate in Lemma~\ref{lem:conv conjug} and the explicit expression for $\nabla \Psi_t(w)$ in equation~\eqref{eq:nabla-Psi-half} and assuming that $x$ is in the range defined in the statement of Lemma~\ref{lem:reduced variance stability}, which ensures that the convex conjugate is bounded, we have
\begin{align}
    &\langle \Prop[t],\ImpLoss[t]-x\mathbf{1}_K \rangle +  \Psi_{t}^*(x\mathbf{1}_K-\ImpLoss[t]+\nabla\Psi_t(\Prop[t]))-\Psi^*_{t}(\nabla\Psi_t(\Prop[t])) \notag\\
    &\hspace{3cm}= \sum_{i=1}^K \AnyProp[t](\AnyImpLoss[t]-x) + \frac{2}{\eta_t}\left(\AnyProp[t]^{-\frac{1}{2}} +\frac{\eta_t}{2}(\AnyImpLoss[t]-x) \right)^{-1}-\frac{2}{\eta_t}\AnyProp[t]^\frac{1}{2}\notag\\
    &\hspace{3cm}= \sum_{i=1}^K \frac{2}{\eta_t}\AnyProp[t]^\frac{1}{2}\left( \frac{\eta_t}{2}(\AnyImpLoss[t]-x)\AnyProp[t]^\frac{1}{2}+\left(1 +\frac{\eta_t}{2}(\AnyImpLoss[t]-x)\AnyProp[t]^\frac{1}{2} \right)^{-1} -1\right)\notag\\
    &\hspace{3cm} = \sum_{i=1}^K \frac{\eta_t}{2}\AnyProp[t]^\frac{3}{2}(\AnyImpLoss[t]-x)^2\left(1 +\frac{\eta_t}{2}(\AnyImpLoss[t]-x)\AnyProp[t]^\frac{1}{2} \right)^{-1}.\label{eq:step1}
\end{align}
From $\min_i \eta_t(\AnyImpLoss[t]-x)\AnyProp[t]^\frac{1}{2} \geq -1$ it follows that 
$\forall i:\,(1 +\frac{\eta_t}{2}(\AnyImpLoss[t]-x)\AnyProp[t]^\frac{1}{2})^{-1}\leq 2$, so
\begin{align}
    \left(1 +\frac{\eta_t}{2}(\AnyImpLoss[t]-x)\AnyProp[t]^\frac{1}{2} \right)^{-1} &= 1 - \frac{\eta_t}{2}(\AnyImpLoss[t]-x)\AnyProp[t]^\frac{1}{2}
    \left(1 +\frac{\eta_t}{2}(\AnyImpLoss[t]-x)\AnyProp[t]^\frac{1}{2}\right)^{-1}\notag\\
    &\leq 1 +\eta_t|x-\AnyImpLoss[t]|_+\AnyProp[t]^\frac{1}{2}\,.\label{eq:step2}
\end{align}
Combining equations~\eqref{eq:step1} and \eqref{eq:step2} completes the proof.
\end{proof}

Now we have all the tools to prove the main {\it stability} lemma.
\begin{proof}{\bf of Lemma~\ref{lem:stability}}
\rev{We begin by proving the first and the last part of the lemma followed by the second and third. 
\paragraph{First part of the lemma.} 
First, we bound the \emph{stability} by $1$.
By convexity of $\Phi_t$, we have
\begin{align*}
    \MyLoss[t] +\Phi_{t}(-\CumLoss[t])-\Phi_{t}(-\CumLoss[t-1])\leq \MyLoss[t]-\langle\nabla \Phi_{t}(-\CumLoss[t]),\ImpLoss[t]\rangle\leq 1\,,
\end{align*}
where the second inequality uses the non-negativity of the IW estimators.}

Recall that $\Prop[t] = \nabla\Phi_{t}(-\CumLoss[t-1])$ and $\MyLoss[t] = \left\langle 
\Prop[t],\ImpLoss[t]\right\rangle$.
Furthermore, $\Phi_t(L+x\mathbf{1}_K) = \Phi_t(L) + x$, where $\mathbf{1}_K$ is a vector of $K$ ones, since we take the argmax over probability distributions.
Finally, from equation~\eqref{eq:lagr mult} follows the existence of a constant $c_t$, such that $\nabla \Psi_t(w_t) = -\hat L_{t-1}+c_t\mathbf{1}_K$.
Hence, for any $x\in\mathbb{R}$
\begin{align}
&\E\left[\MyLoss[t] +\Phi_{t}(-\CumLoss[t])-\Phi_{t}(-\CumLoss[t-1])\right]\nonumber\\ 
&\hspace{2cm}= \E\left[\left\langle\Prop[t],\ImpLoss[t]\right\rangle +\Phi_{t}(-\CumLoss[t])-\Phi_{t}(-\CumLoss[t-1])\right]\nonumber\\
&\hspace{2cm}= \E\left[\left\langle\Prop[t],\ImpLoss[t]\right\rangle +\Phi_{t}(\nabla \Psi_t(w_t)-\ImpLoss[t])-\Phi_{t}(\nabla \Psi_t(w_t))\right]\nonumber\\
&\hspace{2cm}= \E\left[\left\langle\Prop[t],\ImpLoss[t]-x\mathbf{1}_K\right\rangle +\Phi_{t}(\nabla \Psi_t(w_t)-\ImpLoss[t]+x\mathbf{1}_K)-\Phi_{t}(\nabla\Psi_t(w_t))\right]\nonumber\\
&\hspace{2cm}\leq\E\left[\left\langle\Prop[t],\ImpLoss[t]-x\mathbf{1}_K\right\rangle +\Psi^*_{t}(\nabla \Psi_t(w_t)-\ImpLoss[t]+x\mathbf{1}_K)-\Psi^*_{t}(\nabla\Psi_t(w_t))\right]\label{eq:9-1}\\
&\hspace{2cm}= \E\left[D_{\Psi_t^*}(\nabla\Psi_t(w_t)-\ImpLoss[t]+x\mathbf{1}_K,\nabla\Psi_t(w_t))\right]\nonumber\\
&\hspace{2cm}\leq \E\left[\max_{z \in \operatorname{conv}(\nabla\Psi_t(w_t),\nabla\Psi_t(w_t)-\ImpLoss[t]+x\mathbf{1}_K)}\frac{1}{2}||\ImpLoss[t]-x\mathbf{1}_K||^2_{\nabla^2\Psi^*_t(z)}\right] \label{eq:9-2}\\
&\hspace{2cm}=\E\left[\max_{w \in \operatorname{conv}(\Prop[t],\nabla\Psi_t^*(\nabla\Psi_t(w_t)-\ImpLoss[t]+x\mathbf{1}_K))}\frac{1}{2}||\ImpLoss[t]-x\mathbf{1}_K||^2_{\nabla^2\Psi_t(w)^{-1}}\right]\label{eq:9-3}\\
&\hspace{2cm}\leq
\E\left[\sum_{\AnyArm=1}^K\max_{w_i \in [\AnyProp[t],\nabla\Psi_t^*(\nabla\Psi_t(w_t)-\ImpLoss[t]+x\mathbf{1}_K)_i]}\frac{\LearningRate[t]\xi_i}{2}(\AnyImpLoss[t]-x)^2w_i^{2-\alpha}\right],\nonumber
\end{align}
where in equation \eqref{eq:9-1} we have $\Phi_t(x)\leq \Psi_t^*(x)$, because $\Phi_t$ is a constrained version of $\Psi_t^*$, and $\Phi_t(\nabla \Psi_t(w_t)) = \Psi^*_t(\nabla \Psi_t(w_t))$, because $\argmax_{w\in\R^K}\ip{w,\nabla \Psi_t(w_t)}-\Psi(w)=w_t$ is in the probability simplex and the constraint is inactive. 
Inequality \eqref{eq:9-2} follows by equation \eqref{eq:taylor}, and \eqref{eq:9-3} by equation \eqref{eq:Dual-Hessian}. 

In order to prove the first part of the Lemma, we set $x=0$ and observe that $\nabla\Psi^*_t(\nabla\Psi_t(w_t)-\ImpLoss[t])_i \leq \nabla\Psi^*_t(\nabla\Psi_t(w_t))_i = \AnyProp[t]$ because of non-negativity of the losses and the fact that $\nabla\Psi^*_t$ is monotonically increasing, see \eqref{eq:nabla-Psi}. 
(The observation implies that the highest value of $w_i \in [\AnyProp[t],\nabla\Psi_t^*(\nabla\Psi_t(w_t)-\ImpLoss[t])_i]$ is $\AnyProp[t]$.) 
Since the importance weighted losses are $0$ for the arms that were not played, we have
\begin{multline*}
\E\left[\sum_{\AnyArm=1}^K\max_{w_i \in [\AnyProp[t],\nabla\Psi_t^*(\nabla\Psi_t(w_t)-\ImpLoss[t])_i]}\frac{\LearningRate[t]\xi_i}{2}\AnyImpLoss[t]^2w_i^{2-\alpha}\right] 
= \E\left[\sum_{\AnyArm=1}^K \frac{\LearningRate[t]\xi_i}{2}\AnyImpLoss[t]^2\AnyProp[t]^{2-\alpha}\right]\\
=\frac{\LearningRate[t]\xi_i}{2}\E\left[\sum_{\AnyArm=1}^K \frac{\ell_{t,i}^2}{\AnyProp[t]^2}\AnyProp[t]^{2-\alpha}\ind[t][\AnyArm]\right]
=\frac{\LearningRate[t]\xi_i}{2}\E\left[\sum_{\AnyArm=1}^K \frac{\ell_{t,i}^2}{\AnyProp[t]^2}\AnyProp[t]^{3-\alpha}\right]
\leq \sum_{\AnyArm=1}^K\frac{\LearningRate[t]\xi_i}{2}\E\left[\AnyProp[t]\right]^{1-\alpha},
\end{multline*}
where we use that $\E[\ind[t][i] | \ell_1,\dots,\ell_{t-1}, \MyArm[1],\dots,\MyArm[t-1]] = \AnyProp[t]$.  
The last inequality follows by Jensen's inequality.

\paragraph{Fourth part of the lemma.} We set $x=\ind[t][j]\ell_{t,j}$.
In the calculation below, for the events $\MyArm[t]\in \{1,\dots,K\}\setminus j$, we have $x=0$ and use the same derivation as in the previous case.
When $\MyArm[t] = j$, for $i\neq j$ we have $\hat \ell_{t,i} - x = -x \geq -1$ and for $j$ we have $\hat \ell_{t,j} - x \geq 0$.
For $i\neq j$ we use Lemma~\ref{lem:wrange} to bound 
$\lr{\nabla\Psi_t^*(\nabla\Psi_t(w_t)-\ImpLoss[t]+x\mathbf{1}_K)_i}^{2-\alpha} \leq 2 w_{t,i}^{2-\alpha}$ and for $j$ we use $\nabla\Psi^*_t(\nabla\Psi_t(w_t)-\ImpLoss[t])_j \leq \nabla\Psi^*_t(\nabla\Psi_t(w_t))_j = w_{t,j}$.
\begin{align*}
&\E\left[\sum_{\AnyArm=1}^K\max_{\tilde{w}_i \in [\AnyProp[t],\nabla\Psi_t^*(\nabla\Psi_t(w_t)-\ImpLoss[t]+x\mathbf{1}_K)_i]}\frac{\LearningRate[t]\xi_i}{2}(\AnyImpLoss[t]-x)^2\tilde{w}_i^{2-\alpha}\right]\\
&\hspace{0.8cm}\leq \sum_{\AnyArm\neq j}\frac{\LearningRate[t]\xi_i}{2}\E\left[\AnyProp[t]\right]^{1-\alpha} + \E\left[\ind[t][j]\left(\frac{\eta_{t}\xi_j}{2}\left(\frac{\ell_{t,j}}{w_{t,j}}-\ell_{t,j}\right)^2w_{t,j}^{2-\alpha} + \sum_{\AnyArm\neq j}\frac{\LearningRate[t]\xi_i}{2}\ell_{t,j}^2 2\AnyProp[t]^{2-\alpha} \right)\right]\\
&\hspace{0.8cm}\leq\sum_{\AnyArm\neq j}\frac{\LearningRate[t]\xi_i}{2}\E\left[\AnyProp[t]\right]^{1-\alpha} +\E\left[ \frac{\eta_{t}\xi_j}{2}(1-w_{t,j})^2w_{t,j}^{1-\alpha}+\sum_{i\neq j}\LearningRate[t]\xi_i\AnyProp[t]^{2-\alpha}w_{t,j}\right]\\
&\hspace{0.8cm}\leq\sum_{\AnyArm\neq j}\lr{\frac{\LearningRate[t]\xi_i}{2}\E\left[\AnyProp[t]\right]^{1-\alpha} +\frac{\eta_{t}(\xi_j + 2 \xi_i)}{2}\E\left[\AnyProp[t]\right]},
\end{align*}
where in the last step for the middle term we use $(1-w_{t,j})^2 w_{t,j}^{1-\alpha} \leq 1 - w_{t,j} = \sum_{i\neq j} w_{t,i}$ and for the last term $w_{t,i}^{2-\alpha} \leq 1$.

\rev{
\paragraph{Second part of the lemma.} 
We set $x=\MyLoss[t]$ and first verify that Lemma~\ref{lem:reduced variance stability} can be applied.
We have for any $i$:
\begin{align*}
    & \eta_t(\AnyImpLoss[t]-x)w_{t,i}^\frac{1}{2}\geq -\eta_t w_{t,i}^\frac{1}{2}\geq -\eta_t\geq -1\,,
\end{align*}
where the last inequality is by the assumption of the lemma. Since $\E[\MyLoss[t]]=\E[\langle\Prop[t],\ell_t\rangle \rangle]=\E[\langle\Prop[t],\ImpLoss[t]\rangle \rangle]$, by applying Lemma~\ref{lem:reduced variance stability} we have
\begin{align}
    \E\left[\MyLoss[t]+\Phi_t(-\hat L_t) - \Phi_t(-\hat L_{t-1})\right]
    &\leq \E\left[\sum_{i=1}^K \frac{\eta_t}{2}\AnyProp[t]^\frac{3}{2}(\AnyImpLoss[t]-\MyLoss[t])^2 + \frac{\eta_t^2}{2} \AnyProp[t]^2\left|\MyLoss[t]-\AnyImpLoss[t]\right|_+^3\right]\notag\\
    &\leq \frac{\eta_t}{2}\E\left[ \MyProp[t]^{-\frac{1}{2}}(1-\MyProp[t])^2+\sum_{i\neq \MyArm[t]} \AnyProp[t]^\frac{3}{2} \right]+\frac{\eta_t^2}{2}\label{eq:iw 1}\\
    &= \frac{\eta_t}{2}\E\left[ \MyProp[t]^{-\frac{1}{2}}(1-\MyProp[t])^2+\lr{\sum_{i=1}^K \AnyProp[t]^\frac{3}{2}} - \MyProp[t]^\frac{3}{2} \right]+\frac{\eta_t^2}{2}\notag\\
    &= \frac{\eta_t}{2}\E\left[\sum_{i=1}^K \AnyProp[t]^\frac{1}{2}(1-\AnyProp[t])^2 + (1-\AnyProp[t])\AnyProp[t]^\frac{3}{2} \right]+\frac{\eta_t^2}{2}\label{eq:iw 2}\\
    &= \frac{\eta_t}{2}\E\left[\sum_{i=1}^K \AnyProp[t]^\frac{1}{2}(1-\AnyProp[t]) \right]+\frac{\eta_t^2}{2}\notag\\    
    &\leq \frac{\eta_t}{2}\sum_{i=1}^K\E[ \AnyProp[t]]^\frac{1}{2}(1-\E[\AnyProp[t]])+\frac{\eta_t^2}{2}\label{eq:iw 3}\,,
\end{align}
where equation~\eqref{eq:iw 1} uses non-negativeness of IW estimators and boundedness of the losses in $[0,1]$, by which $\left|\MyLoss[t]-\AnyImpLoss[t]\right|_+ \leq 1$, and explicit form of $\hat \ell_{t,i}$ for the first term in the summation;
 equation~\eqref{eq:iw 2} uses the conditional probability of $I_t=i$, which is $\AnyProp[t]$; and equation~\eqref{eq:iw 3} follows by concavity of the function $f(z) = z^{\frac{1}{2}}(1-z)$ and Jensen's inequality.
 
\paragraph{Third part of the lemma.} 
We set $x=\MyLoss[t]$ and first verify that Lemma~\ref{lem:reduced variance stability} can be applied. Recall that $\mathbb{B}_t(i) := \frac{1}{2}\ind[][w_{t,i} \geq \eta_t^2]$. 
For $i\neq I_t$ we have $\AnyImpLoss[t] = \mathbb{B}_t(i)$ and 
\begin{align*}
    \eta_t(\AnyImpLoss[t]-x)w_{t,i}^\frac{1}{2}= \eta_t(\mathbb{B}_t(i)-\MyLoss[t])w_{t,i}^\frac{1}{2}\geq -\eta_t \geq -1\,,
\end{align*}
while for $I_t$ we have $\MyImpLoss[t] = \frac{\MyLoss[t]-\mathbb{B}_t(I_t)}{\MyProp[t]}+\mathbb{B}_t(I_t)$ and 
\begin{align*}
    \eta_t(\MyImpLoss[t]-\MyLoss[t])w_{t,i}^\frac{1}{2}= \eta_t(\MyLoss[t]-\mathbb{B}_t(I_t))(\frac{1}{\MyProp[t]}-1)\MyProp[t]^\frac{1}{2}\geq -\eta_t\mathbb{B}_t(I_t)\MyProp[t]^{-\frac{1}{2}}\geq -1\,.
\end{align*}
Since $\E[\MyLoss[t]]=\E[\langle\Prop[t],\ell_t\rangle \rangle]=\E[\langle\Prop[t],\ImpLoss[t]\rangle \rangle]$, applying Lemma~\ref{lem:reduced variance stability} we have
\begin{align}
    \E\left[\MyLoss[t]+\Phi_t(-\hat L_t) - \Phi_t(-\hat L_{t-1})\right]
    &\leq \E\left[\sum_{i=1}^K \frac{\eta_t}{2}\AnyProp[t]^\frac{3}{2}(\AnyImpLoss[t]-\MyLoss[t])^2 + \frac{\eta_t^2}{2} \AnyProp[t]^2\left|\MyLoss[t]-\AnyImpLoss[t]\right|_+^3\right]\notag\\
    &\leq \E\left[\sum_{i=1}^K \frac{\eta_t}{2}\AnyProp[t]^\frac{3}{2}(\AnyImpLoss[t]-\MyLoss[t])^2 + \frac{\eta_t^2}{2} \AnyProp[t]^2\left|\AnyImpLoss[t]-\MyLoss[t]\right|^3\right].\label{eq:with estimators}
\end{align}
For any $\ell\in[0,1]$ and any $i$, we have $|\ell - \mathbb{B}_t(i)| \leq |1 - \mathbb{B}_t(i)|$. For $I_t$, we have $|\MyImpLoss[t]-\MyLoss[t]|=|\MyLoss[t]-\mathbb{B}_t(\MyArm[t])|\frac{1-\MyProp[t]}{\MyProp[t]}\leq |1-\mathbb{B}_t(\MyArm[t])|\frac{1-\MyProp[t]}{\MyProp[t]}$,
while for $i\neq I_t$ we have $|\AnyImpLoss[t]-\MyLoss[t]| \leq |1-\mathbb{B}_t(i)|$.
Let $\bar{\mathbb{B}}_t(i)=\frac{1}{2}\ind[][w_{t,i} < \eta_t^2]=\frac{1}{2}-\mathbb{B}_t(i)$, then
$|1-\mathbb{B}_t(i)|=|\frac{1}{2}+\bar{\mathbb{B}}_t(i)|$. We have
%\begin{align*}
    $|\frac{1}{2}+\bar{\mathbb{B}}_t(i)|^2 = \frac{1}{4}+\frac{3}{2}\bar{\mathbb{B}}_t(i)$ %\,,\\
    and 
    $|\frac{1}{2}+\bar{\mathbb{B}}_t(i)|^3 \leq 1$. %.\,.
%\end{align*}
Thus, for $I_t$ we have
\begin{align*}
    (\hat \ell_{t,I_t} - \ell_{t,I_t})^2 &\leq \lr{\frac{1}{4} + \frac{3}{2}\bar{ \mathbb{B}}_t(I_t)}\frac{\lr{1-w_{t,I_t}}^2}{{w_{t,I_t}}^2}\,,\\
    |\hat \ell_{t,I_t} - \ell_{t,I_t}|^3 &\leq \frac{\lr{1-w_{t,I_t}}^3}{{w_{t,I_t}}^3}
\end{align*}
and for $i\neq I_t$ we have
\begin{align*}
    (\hat \ell_{t,i} - \ell_{t,I_t})^2 &\leq \lr{\frac{1}{4} + \frac{3}{2}\bar{\mathbb{B}}_t(I_t)}\,,\\
    |\hat \ell_{t,I_t} - \ell_{t,I_t}|^3 &\leq 1.
\end{align*}
Plugging this into equation~\eqref{eq:with estimators} leads to
\begin{align*}
    &\E\left[\sum_{i=1}^K \frac{\eta_t}{2}\AnyProp[t]^\frac{3}{2}(\AnyImpLoss[t]-\MyLoss[t])^2 + \frac{\eta_t^2}{2} \AnyProp[t]^2\left|\AnyImpLoss[t]-\MyLoss[t]\right|^3\right]\\
    &\hspace{3cm}\leq\E\Bigg[\left(\frac{1}{4}+\frac{3}{2}\bar{\mathbb{B}}_t(\MyArm[t])\right)\frac{\eta_t}{2}\MyProp[t]^{-\frac{1}{2}}(1-\MyProp[t])^2 + \frac{\eta_t^2}{2} \MyProp[t]^{-1}(1-\MyProp[t])^3\\
    &\hspace{3cm}\qquad +\sum_{i\neq\MyArm[t]}\left(\frac{1}{4}+\frac{3}{2}\bar{\mathbb{B}}_t(i)\right)\frac{\eta_t}{2}\AnyProp[t]^{\frac{3}{2}} + \frac{\eta_t^2}{2} \AnyProp[t]^2\Bigg]\\
    &\hspace{3cm}= \E\Bigg[\sum_{i=1}^K\Bigg(\left(\frac{1}{4}+\frac{3}{2}\bar{\mathbb{B}}_t(i)\right)\frac{\eta_t}{2}\AnyProp[t]^{\frac{1}{2}}(1-\AnyProp[t])^2 + \frac{\eta_t^2}{2} (1-\AnyProp[t])^3\\
    &\hspace{3cm}\qquad+(1-\AnyProp[t])\left(\frac{1}{4}+\frac{3}{2}\bar{\mathbb{B}}_t(i)\right)\frac{\eta_t}{2}\AnyProp[t]^{\frac{3}{2}} + (1-\AnyProp[t])\frac{\eta_t^2}{2} \AnyProp[t]^2\Bigg)\Bigg]\\
    &\hspace{3cm}\leq \E\Bigg[\sum_{i=1}^K\left(\frac{\eta_t}{8}\AnyProp[t]^\frac{1}{2}(1-\AnyProp[t]) +\frac{3}{2}\bar{\mathbb{B}}_t(i)\frac{\eta_t}{2}\AnyProp[t]^\frac{1}{2}+\frac{\eta_t^2}{2}\right)\Bigg]\\
    &\hspace{3cm}\leq \frac{7\eta_t^2}{8}K +\E\left[\sum_{i=1}^K\frac{\eta_t}{8}\AnyProp[t]^\frac{1}{2}(1-\AnyProp[t])\right]\\
    &\hspace{3cm}\leq \frac{7\eta_t^2}{8}K +\sum_{i=1}^K\frac{\eta_t}{8}\E[\AnyProp[t]]^\frac{1}{2}(1-\E[\AnyProp[t]])\,,
\end{align*}
where in the penultimate step we use $\bar{\mathbb B}_t(i){w_{t,i}}^{\frac{1}{2}} = \frac{1}{2}\ind[][w_{t,i}<\eta_t^2] {w_{t,i}}^{\frac{1}{2}} \leq \frac{1}{2}\eta_t$ and the last step follows by concavity of $f(z) = z^{\frac{1}{2}}(1-z)$ and Jensen's inequality.}
\end{proof}

\subsection{Controlling the {\it penalty} Term}
We begin with a standard lemma to simplify the {\it penalty} term. 
\begin{lemma}
\label{lem:tpot-basic}
For any $\alpha\in[0,1]$, any positive learning rate, and any fixed $v,u \in \Simplex$, the \emph{penalty} term satisfies 
\begin{align*}
&\E\left[\sum_{t=1}^T \lr{\Phi_t(-\CumLoss[t-1])-\Phi_t(-\CumLoss[t]) -\OptLoss[t]}\right]\\ 
 &%\qquad\qquad
 \leq \E\bigg[\frac{\Psi(v)-\Psi(\Prop[1])}{\LearningRate[1]}
%&\hspace{3cm}
+\sum_{t=2}^T\left(\LearningRate[t]^{-1}-\LearningRate[t-1]^{-1}\right)\left(\Psi(v)-\Psi(\Prop[t])\right)+\frac{\Psi(u)-\Psi(v)}{\LearningRate[T]} \bigg]%\\
%&\hspace{12cm}
+\left\langle u-\mathbf{e}_{\BestArm},L_T\right\rangle.
\end{align*}
\end{lemma}
        
\begin{proof} 
First, note that all the terms involving $\Psi(v)$ in the lemma sum up to $0$. 
Then, recall that $\Prop[t]$ is defined as $\arg\max_{w\in\Simplex}\Big\{\left\langle w, -\CumLoss[t-1]\right\rangle - \frac{\Psi(w)}{\eta_t}\Big\}$.
Therefore, 
\begin{align*}
\Phi_t(-\CumLoss[t-1]) &= -\left\langle\Prop[t],\CumLoss[t-1]\right\rangle - \frac{\Psi(\Prop[t])}{\eta_t}.
\end{align*}
Furthermore, by definition of the potential function, for any $\tilde{w}\in\Simplex$ it holds that:
\begin{align*}
-\Phi_t(-\CumLoss[t]) = - \max_{w\in\Simplex}\Big\{\left\langle w, -\CumLoss[t]\right\rangle - \frac{\Psi(w)}{\eta_t}\Big\}
\leq \left\langle \tilde{w}, \CumLoss[t]\right\rangle + \frac{\Psi(\tilde{w})}{\eta_t}.
\end{align*}
Setting $\tilde{w}$ to $\Prop[t+1]$ for $t<T$ and to $u$ for $t=T$, and using $\CumLoss[0]=\mathbf{0}_K$, where $\mathbf{0}_K$ is a vector of $K$ zeros, the sum of potential differences can be bounded as follows:
\begin{align*}
\sum_{t=1}^T & \lr{\Phi_t(-\CumLoss[t-1])-\Phi_t(-\CumLoss[t]) }\\
&\quad\leq \sum_{t=1}^T\Big(-\left\langle\Prop[t],\CumLoss[t-1]\right\rangle -\frac{\Psi(\Prop[t])}{\LearningRate[t]}\Big)
+\sum_{t=1}^{T-1}\Big(\left\langle \Prop[t+1], \CumLoss[t]\right\rangle + \frac{\Psi(\Prop[t+1])}{\LearningRate[t]}\Big)+\left\langle u, \CumLoss[T]\right\rangle +\frac{\Psi(u)}{\LearningRate[T]}\\
&\quad= -\frac{\Psi(\Prop[1])}{\LearningRate[1]}-\sum_{t=2}^T\left(\LearningRate[t]^{-1}-\LearningRate[t-1]^{-1}\right)\Psi(\Prop[t])+\frac{\Psi(u)}{\LearningRate[T]} + \left\langle u,\CumLoss[T]\right\rangle\,.
\end{align*}
The proof is finalized by taking the expectation and subtracting the optimal loss. 
Due to unbiasedness of the loss estimators, for a fixed $u$ we have $\E[\langle u, \CumLoss[T]\rangle] = \left\langle u, L_T\right\rangle$.
\end{proof}

\begin{proof}{\bf of Lemma~\ref{lem:penalty}}
The proof of both parts of the lemma is based on Lemma~\ref{lem:tpot-basic}.
\paragraph{Part 1:}
We set $v = \Prop[1]$.
Since $\Prop[1] = \argmax_{w\in\Simplex}-\Psi_1(w) =\argmax_{w\in\Simplex}-\Psi(w)$, we have $\Psi(\Prop[1])-\Psi(\Prop[t]) \leq 0$ for any $t$.
Since the learning rate is non-increasing, the terms $\left(\LearningRate[t]^{-1}-\LearningRate[t-1]^{-1}\right)$ are all positive, so
\begin{align*}
&\E\left[\sum_{t=1}^T \lr{\Phi_t(-\CumLoss[t-1])-\Phi_t(-\CumLoss[t]) -\OptLoss[t]}\right] \\
&\leq \E\bigg[\frac{\Psi(\Prop[1])-\Psi(\Prop[1])}{\LearningRate[1]}
+\sum_{t=2}^T\left(\LearningRate[t]^{-1}-\LearningRate[t-1]^{-1}\right)\left(\Psi(\Prop[1])-\Psi(\Prop[t])\right)+\frac{\Psi(u)-\Psi(\Prop[1])}{\LearningRate[T]}  \bigg]\\ &\hspace{12cm}+\left\langle u-\mathbf{e}_{\BestArm},L_T\right\rangle\\
&\leq \E\left[\frac{\Psi(u)-\Psi(\Prop[1])}{\LearningRate[T]}\right]+\left\langle u-\mathbf{e}_{\BestArm},L_T\right\rangle.
\end{align*}
Following the trick of \citet{agarwal2016corralling}, we set $u_\BestArm = 1-T^{-1}$ and $u_i = \frac{T^{-1}}{K-1}$ for $i\neq \BestArm$. 
The losses are bounded in $[0,T]$, so this choice of $u$ implies $\langle u- \mathbf{e}_{\BestArm},L_T\rangle \leq 1$.
Since we assume that the regularizer is symmetric, the explicit form of $\Prop[1]$ is $\AnyProp[1] = K^{-1}$ and
\begin{align*}
&\E\left[\sum_{t=1}^T \lr{\Phi_t(-\CumLoss[t-1])-\Phi_t(-\CumLoss[t]) -\OptLoss[t]}\right]\\ 
&\hspace{1cm}\leq\frac{K^{1-\alpha}}{\alpha(1-\alpha)\LearningRate[T]} -\frac{(K-1)^{1-\alpha}T^{-\alpha} +(1-T^{-1})^\alpha}{\alpha(1-\alpha)\LearningRate[T]} +1.
\end{align*}
It remains to bound $K^{1-\alpha} - (K-1)^{1-\alpha}T^{-\alpha} -(1-T^{-1})^\alpha$. Since $x^\alpha$ and $x^{1-\alpha}$ are concave functions, \rev{by Taylor's expansion around $X-1$} we have $X^{1-\alpha} \leq (X-1)^{1-\alpha}+(1-\alpha)(X-1)^{-\alpha}$ and $X^{\alpha} \leq (X-1)^{\alpha}+\alpha(X-1)^{\alpha-1}$ for any $X>1$, thus
\begin{alignat*}{2}
K^{1-\alpha}+T^\alpha &\leq (K-1)^{1-\alpha}+(1-\alpha)(K-1)^{-\alpha} + (T-1)^\alpha + \alpha(T-1)^{\alpha-1}\\
&\leq (K-1)^{1-\alpha} + (T-1)^\alpha + 1,% &%\\
%-(K-1)^{1-\alpha} + (T-1)^\alpha &\leq K^{1-\alpha}+T^\alpha + 1
\end{alignat*}
where the last line uses $(T-1)^{\alpha-1},(K-1)^{-\alpha} \leq 1$.
Therefore,
\begin{align*}
K^{1-\alpha} - (K-1)^{1-\alpha}T^{-\alpha} -(1-T^{-1})^\alpha 
&=K^{1-\alpha} + T^{-\alpha}(-(K-1)^{1-\alpha} -(T-1)^\alpha)\\
&\leq  K^{1-\alpha} + T^{-\alpha}(-K^{1-\alpha}-T^\alpha+1)\\ 
 &= (K^{1-\alpha}-1)(1-T^{-\alpha}).
\end{align*}
\paragraph{Part 2:}
Set
\begin{align*}
&\tilde{w} = \argmax_{w\in \Simplex}-\Psi\left((1-T^{-x})\mathbf{e}_{\BestArm}+T^{-x}w\right),\\ 
&v = u = (1-T^{-x})\mathbf{e}_{\BestArm}+T^{-x}\tilde{w},\\
&v_t = (1-T^{-x})\mathbf{e}_{\BestArm}+T^{-x}\Prop[t].
\end{align*}
By definition, $\Psi(v) \leq \Psi(v_t)$ for all $t$. So
\begin{align*}
\Psi(v) - \Psi(\Prop[t]) \leq \Psi(v_t)- \Psi(\Prop[t])\leq \sum_{\AnyArm\neq\BestArm}\frac{(\AnyProp[t]^\alpha-\alpha\AnyProp[t])(1-T^{-\alpha x})}{\alpha(1-\alpha)\xi_i}.
\end{align*}
In the last inequality, we have used the fact that the contribution of the optimal arm $\BestArm$ is non-positive, since $v_{t,\BestArm} \geq \OptProp[t]$ and $w^\alpha-\alpha w$ are monotonically increasing in $w$ over $[0,1]$. 
The choice of $u$ ensures that $\langle u-\mathbf{e}_{\BestArm},L_T\rangle \leq T^{1-x}$.
Starting again with Lemma~\ref{lem:tpot-basic}, we have:
\begin{align*}
&\E\left[\sum_{t=1}^T \lr{\Phi_t(-\CumLoss[t-1])-\Phi_t(-\CumLoss[t]) -\OptLoss[t]}\right] \\
 &\leq \E\bigg[\frac{\Psi(v)-\Psi(\Prop[1])}{\LearningRate[1]}
+\sum_{t=2}^T\left(\LearningRate[t]^{-1}-\LearningRate[t-1]^{-1}\right)\left(\Psi(v)-\Psi(\Prop[t])\right)+\frac{\Psi(u)-\Psi(v)}{\LearningRate[T]} \bigg] \\
&\hspace{12cm}+\left\langle u-\mathbf{e}_{\BestArm},L_T\right\rangle\\
&\leq \frac{1-T^{-\alpha x}}{\alpha}\sum_{\AnyArm\neq\BestArm}\E\bigg[\frac{\AnyProp[1]^\alpha-\alpha\AnyProp[1]}{(1-\alpha)\LearningRate[1]\xi_i}
+\sum_{t=2}^T\left(\LearningRate[t]^{-1}-\LearningRate[t-1]^{-1}\right)\frac{\AnyProp[t]^{\alpha}-\alpha\AnyProp[t]}{(1-\alpha)\xi_i}\bigg] +T^{1-x}\\
&\leq \frac{1-T^{-\alpha x}}{\alpha}\sum_{\AnyArm\neq\BestArm}\lr{\frac{\E[\AnyProp[1]]^\alpha-\alpha\E[\AnyProp[1]]}{(1-\alpha)\LearningRate[1]\xi_i}
+\sum_{t=2}^T\left(\LearningRate[t]^{-1}-\LearningRate[t-1]^{-1}\right)\frac{\E[\AnyProp[t]]^{\alpha}-\alpha\E[\AnyProp[t]]}{(1-\alpha)\xi_i}} +T^{1-x}.
\end{align*} 
\end{proof}

\section{Proof of Theorem~\ref{th:all}}
\label{app:thm-all}

We follow the same strategy as outlined in Section~\ref{sec:proofs}. 
In order to cover the limit cases $\alpha\in\{0,1\}$, the proof is significantly more technical than the proof of Theorem~\ref{th:easy}.

\begin{proof}{\bf of Theorem~\ref{th:all}}
Recall that the learning rate is $\LearningRate[t] = \TimeDependentLearningRate$, where $\overline{t}=\max\{e,t\}$, and regularization parameters are $\xi_i=\TimeIndependentLearningRate$ \rev{for $i\neq i^*$ and $\xi_{i^*} = \Deltamin^{1-2\alpha}$}.

\paragraph{Bounding the \emph{stability} term} We start by bounding the {\it stability} term.
For $t\leq T_0$  we use the first part of Lemma~\ref{lem:stability} and otherwise the second with $j=\BestArmSto$.
The value of $T_0$ is chosen so that $\LearningRate[T_0]\xi_i \leq \frac{1}{4}$.
\begin{alignat*}{2}
stability &= \E\left[\sum_{t=1}^T\MyLoss[t]+ \Phi_t(- \CumLoss[t]) - \Phi_t(-\CumLoss[t-1]) \right]\\
&\leq \underbrace{\sum_{t=1}^{T}\sum_{\AnyArm\neq\BestArmSto}\frac{\LearningRate[t]\xi_i\E[\AnyProp[t]]^{1-\alpha}}{2}}_{concave} + \underbrace{\sum_{t=1}^{T_0}\frac{\LearningRate[t]\xi_\BestArmSto\E[\BestPropSto[t]]^{1-\alpha}}{2}}_{constant} + \underbrace{\sum_{t=T_0+1}^T\sum_{\AnyArm\neq\BestArmSto}\frac{\LearningRate[t](\xi_i + 2\xi_\BestArmSto)}{2}\E[\AnyProp[t]]}_{linear}.
\end{alignat*}
\paragraph{Bounding the concave part}
Since $w^{1-\alpha}$ is a concave function of $w$, it can be upper bounded by the first order Taylor's approximation. For any $w^*$:
\begin{align*}
\AnyProp[t]^{1-\alpha}&\leq {w^*}^{1-\alpha} + (1-\alpha) {w^*}^{-\alpha}(\AnyProp[t]-w^*)\\
&=\alpha{w^*}^{1-\alpha}+(1-\alpha) {w^*}^{-\alpha}\AnyProp[t]. 
\end{align*}
Taking $w^* = \frac{16}{\AnyGap^2t}$ (with $\eta_{t} = \TimeDependentLearningRate$, $\xi_i=\TimeIndependentLearningRate$):
\begin{alignat}{2}
\sum_{t=1}^T\frac{\LearningRate[t]\xi_i}{2} \E[\AnyProp[t]]^{1-\alpha}
&\leq \sum_{t=1}^T \frac{\AnyExplicitLearningRate}{2} \bigg(\alpha\left(\frac{16}{\AnyGap^2t}\right)^{1-\alpha}+ (1-\alpha) \left(\frac{16}{\AnyGap^2t}\right)^{-\alpha}\E[\AnyProp[t]]\bigg)\nonumber\\
&= \sum_{t=1}^T\frac{1-\overline{t}^{-1+\alpha}}{1-\alpha} \left(\frac{2\alpha}{\AnyGap t} + \frac{1-\alpha}{8}\AnyGap \E[\AnyProp[t]]\right)\nonumber\\
&\leq \frac{1-T^{-1+\alpha}}{1-\alpha} \frac{2(\log(T)+1)}{ \AnyGap}+  \sum_{t=1}^T  \frac{\AnyGap \E[\AnyProp[t]]}{8}.\label{line:stability1}
\end{alignat}
Finally, we bound the leading factor of the log term with Lemma~\ref{lem:loglimit}:
\begin{align*}
&\frac{1-T^{-1+\alpha}}{1-\alpha}\leq \min\{\frac{1}{1-\alpha},\log(T)\}.
\end{align*}

\paragraph{Bounding the linear part}
We first show that all $t>T_0=\frac{16}{\OptGap^2}\log^2(\frac{16}{\OptGap^2})$ satisfy $\LearningRate[t]\xi_i \leq \frac{\AnyGap}{4}$.
\begin{align*}
\LearningRate[t]\xi_i = \AnyExplicitLearningRate
&< \frac{\AnyGap}{4}\left(\frac{16}{\OptGap^2T_0}\right)^\alpha \frac{1-\overline{T_0}^{-1+\alpha}}{1-\alpha}\\
&\leq \frac{\AnyGap}{4} \frac{1-(\frac{16}{\OptGap^2}\log^2(\frac{16}{\OptGap^2}))^{-1+\alpha}}{(1-\alpha)(\log(\frac{16}{\OptGap^2}))^{2\alpha}}.
\end{align*}
It remains to show that $ \frac{1-(\frac{16}{\OptGap^2}\log^2(\frac{16}{\OptGap^2}))^{-1+\alpha}}{(1-\alpha)(\log(\frac{16}{\OptGap^2}))^{2\alpha}}\leq 1$. 
By Lemma~\ref{lem:ugly-ineq} we have
\begin{align*}
\frac{1-(\frac{16}{\OptGap^2}\log^2(\frac{16}{\OptGap^2}))^{-1+\alpha}}{1-\alpha}&\leq \left(\log\lr{\frac{16}{\OptGap^2}\log^2(\frac{16}{\OptGap^2})}\right)^\alpha\\
&\leq \left(2\log(\frac{16}{\OptGap^2})\right)^\alpha
\leq \left(\log(\frac{16}{\OptGap^2})\right)^{2\alpha}\,,
\end{align*}
which concludes the proof.
Therefore,
\begin{align}
\sum_{\AnyArm\neq\BestArmSto}\sum_{t=T_0+1}^T\frac{\LearningRate[t](\xi_i+2\xi_{i^*})}{2}\E[\AnyProp[t]] \leq \sum_{\AnyArm\neq\BestArmSto}\sum_{t=1}^T\frac{\AnyGap +2\OptGap }{8}\E[\AnyProp[t]]
\leq \sum_{\AnyArm\neq\BestArmSto}\sum_{t=1}^T\frac{3\AnyGap\E[\AnyProp[t]]}{8}. \label{line:stability2}
\end{align}

\paragraph{Bounding the constant part}
Recall $T_0=\frac{16}{\OptGap^2}\log^2\left(\frac{16}{\OptGap^2}\right) \geq 16$. We can use the estimation $\sum_{t=1}^{T_0}t^{-\alpha} \leq 1+\int_{1}^{T_0}t^{-\alpha}\,dt = 1+\frac{T_0^{1-\alpha} - 1}{1-\alpha}\leq 2\frac{T_0^{1-\alpha} - 1}{1-\alpha}$ :
\begin{align}
\sum_{t=1}^{T_0} \frac{\LearningRate[t]\xi_{i^*}}{2} &\leq \sum_{t=1}^{T_0} \frac{\OptGap^{1-2\alpha}16^\alpha(1-T_0^{-1+\alpha})}{8(1-\alpha)t^\alpha}= \frac{\OptGap^{1-2\alpha}16^\alpha(1-T_0^{-1+\alpha})}{8(1-\alpha)}\sum_{t=1}^{T_0} \frac{1}{t^\alpha}\nonumber\\
&\leq \frac{\OptGap^{1-2\alpha}16^\alpha(1-T_0^{-1+\alpha})(T_0^{1-\alpha}-1)}{4(1-\alpha)^2}\nonumber\\
&=  \frac{16^\alpha T_0^{1-\alpha}(1-T_0^{-1+\alpha})^2}{4\OptGap^{2\alpha-1}(1-\alpha)^2}=  \frac{4 \log^{2-2\alpha}(\frac{16}{\OptGap^2})}{\OptGap}\left(\frac{1-T_0^{-1+\alpha}}{1-\alpha}\right)^2\nonumber\\
&\leq \frac{4 \log^2(\frac{16}{\OptGap^2})\log^2(T_0)}{\OptGap}\leq\frac{4 \log^4(T_0)}{\OptGap},\label{line:stability3}
\end{align}
where the last line uses Lemma~\ref{lem:loglimit}. 
Combining \eqref{line:stability1}, \eqref{line:stability2}, and \eqref{line:stability3} we obtain:
\begin{align}
stability \leq \sum_{\AnyArm\neq\BestArmSto} \lr{ 
\min\lrc{\frac{1}{1-\alpha},\log(T)}\frac{2(\log(T)+1)}{\AnyGap} +\sum_{t=1}^T\frac{\AnyGap\E[\AnyProp[t]]}{2}} + \frac{4\log^4(T_0)}{\OptGap }. \label{line:stability}
\end{align}

\paragraph{Bounding the \emph{penalty} term} For the {\it penalty} term, we start with the second part of Lemma~\ref{lem:penalty} with $x=1$.
We have
\begin{align*}
&penalty = \E\left[\sum_{t=1}^T - \Phi_t(- \CumLoss[t]) + \Phi_t(-\CumLoss[t-1]) -\OptLoss[t]\right]\\
&\qquad\leq  \frac{1-T^{-\alpha }}{\alpha}\sum_{\AnyArm\neq\BestArmSto}\lr{\frac{\E[\AnyProp[1]]^\alpha-\alpha\E[\AnyProp[1]]}{\LearningRate[1]\xi_i(1-\alpha)}+\sum_{t=2}^T\lr{\frac{1}{\LearningRate[t]}-\frac{1}{\LearningRate[t-1]}}\frac{\E[\AnyProp[t]]^\alpha-\alpha\E[\AnyProp[t]]}{(1-\alpha)\xi_i}} + 1\\
&\qquad\leq \sum_{\AnyArm\neq\BestArmSto}\lr{\underbrace{\frac{\E[\AnyProp[1]]^\alpha-\alpha\E[\AnyProp[1]]}{\LearningRate[1]\xi_i(1-\alpha)}}_{constant}\log(T)+\underbrace{\sum_{t=2}^T\lr{\frac{1}{\LearningRate[t]}-\frac{1}{\LearningRate[t-1]}}\frac{\E[\AnyProp[t]]^\alpha-\alpha\E[\AnyProp[t]]}{(1-\alpha)\alpha\xi_i}}_{concave}}+1,
\end{align*}
where in the first term we use $\frac{1-T^{-\alpha }}{\alpha} \leq \log(T)$ and in the second $\frac{1-T^{-\alpha }}{\alpha} \leq \frac{1}{\alpha}$, both bounds following from Lemma~\ref{lem:loglimit}.

\paragraph{Bounding the concave term}
Since $w^{\alpha}$ is a concave function of $w$ it can be upper bounded by the first order Taylor's approximation:
\begin{align*}
\AnyProp[t]^{\alpha}&\leq {w^*}^{\alpha} + \alpha {w^*}^{\alpha-1}(\AnyProp[t]-w^*)=(1-\alpha){w^*}^{\alpha}+\alpha {w^*}^{\alpha-1}\AnyProp[t]. 
\end{align*}
Taking $w^* = \frac{16}{\AnyGap^2t}$ (with $\eta_t = \TimeDependentLearningRate$ and $\xi_i = \Delta_i^{1-2\alpha}$):
\begin{alignat}{2}
&\sum_{t=1}^{T-1}\left(\frac{1}{\LearningRate[t+1]}-\frac{1}{\LearningRate[t]}\right)\frac{\E[\AnyProp[t+1]]^\alpha-\alpha\E[\AnyProp[t+1]]}{(1-\alpha)\alpha\xi_i}\\
&\qquad\leq\sum_{t=1}^{T-1}\left(\frac{(t+1)^\alpha}{1-\overline{(t+1)}^{-1+\alpha}}-\frac{t^\alpha}{1-\overline{t}^{-1+\alpha}}\right)\notag\\
\span\span\cdot\frac{4\AnyGap^{2\alpha-1}}{\alpha16^\alpha}\bigg(    (1-\alpha)\left(\frac{16}{\AnyGap^2t}\right)^\alpha+\alpha\left(\frac{16}{\AnyGap^2t}\right)^{\alpha-1}\E[\AnyProp[t+1]]-\alpha\E[\AnyProp[t+1]] \bigg)\notag\\
&\qquad\leq\sum_{t=1}^{T-1}\left(\frac{(t+1)^\alpha - t^\alpha}{1-\overline{t}^{-1+\alpha}}\right) \bigg(
 \frac{4(1-\alpha)}{\alpha \AnyGap t^\alpha}+ \frac{\AnyGap\E[\AnyProp[t+1]]}{4 t^{\alpha-1}}\left(1-\left(\frac{16}{\AnyGap^2 t}\right)^{1-\alpha}\right)\bigg)\notag\\ 
%\end{alignat*}
&\qquad\text{(by Taylor's approximation $(t+1)^\alpha \leq t^\alpha + \alpha t^{\alpha-1}$ and also use $\frac{16}{\AnyGap^2}>1$)}\notag\\
%\begin{alignat}{2}
&\qquad\leq\sum_{t=1}^{T-1}\left(\frac{\alpha t^{\alpha-1}}{1-\overline{t}^{-1+\alpha}}\right)\bigg(
 \frac{4(1-\alpha)}{\alpha \AnyGap t^\alpha}+ \frac{\AnyGap\E[\AnyProp[t+1]]}{4 t^{\alpha-1}}\left(1-\left(\frac{1}{t}\right)^{1-\alpha}\right)\bigg) \nonumber\\
&\qquad\leq\sum_{t=1}^{T-1}\bigg(
 \frac{1-\alpha}{1-e^{-1+\alpha}}\frac{4}{\AnyGap t}+\frac{\AnyGap\E[\AnyProp[t+1]]}{4}\frac{1-t^{-1+\alpha}}{1-\overline{t}^{-1+\alpha}}\bigg) \nonumber\\
&\qquad\leq \frac{1-\alpha}{1-e^{-1+\alpha}}\frac{4(\log(T)+1)}{\AnyGap}+\sum_{t=1}^T\frac{\AnyGap\E[\AnyProp[t]]}{4}\nonumber
\\
&\qquad\leq \frac{8(\log(T)+1)}{\AnyGap}+\sum_{t=1}^T\frac{\AnyGap\E[\AnyProp[t]]}{4}.\label{line:penalty1}
\end{alignat}
The last step follows by the leading factor $\frac{1-\alpha}{1-e^{-1+\alpha}}$ being bounded by 2. 

\paragraph{Bounding the constant term}
Since $\frac{w^\alpha-\alpha w}{1-\alpha}$ is monotonically decreasing in $\alpha$ and $\xi_i$ is monotonically increasing in $\alpha$, we have 
\begin{align}
\frac{\E[\AnyProp[1]]^\alpha-\alpha\E[\AnyProp[1]]}{(1-\alpha)\LearningRate[1]\xi_i}\leq \frac{4}{\AnyGap(1-e^{-1})}\leq \frac{8}{\AnyGap}. \label{line:penalty2}
\end{align}
Combining \eqref{line:penalty1} and \eqref{line:penalty2} we obtain
\begin{align}
penalty \leq \sum_{\AnyArm\neq\BestArmSto}\lr{\frac{16(\log(T)+1)}{\AnyGap}+\sum_{t=1}^T  \frac{\AnyGap \E[\AnyProp[t]]}{4}}+1.\label{line:penalty}
\end{align}

\paragraph{Finishing the proof} Finally, we combine \eqref{line:stability}, \eqref{line:penalty}, and rearrange the terms to get
\begin{alignat*}{2}
\overline{Reg}_T &\leq \rev{\sum_{t=1}^T\sum_{i\neq i^*}\frac{3\Delta_i\E[w_{t,i}]}{4}+\sum_{\AnyArm\neq\BestArmSto}\Bigg(\frac{\left(2 \min\{\frac{1}{1-\alpha},\log(T)\}+16\right)(\log(T)+1)}{\AnyGap}\Bigg)} \\
&\rev{\hspace{10.2cm}+\frac{4 \log^4(T_0)}{\OptGap}+ 1.}\\
&= \frac{3\overline{Reg}_T}{4}+\sum_{\AnyArm\neq\BestArmSto}\Bigg(\frac{\left(2 \min\{\frac{1}{1-\alpha},\log(T)\}+16\right)(\log(T)+1)}{\AnyGap}\Bigg) +\frac{4 \log^4(T_0)}{\OptGap}+ 1.
\end{alignat*}
Rearranging and multiplying by $4$ finishes the proof. 
\end{proof}
% According to JMLR formatting guidelines the references should be at the end of the paper http://www.jmlr.org/format/format.html Also no \newpages between body/appendicies/references.
%\newpage
\bibliography{bibliography,mybib}

\end{document}